\documentclass[10pt,journal,compsoc]{IEEEtran}
\usepackage[utf8]{inputenc} 
\usepackage[T1]{fontenc}    
\usepackage{hyperref}       
\usepackage{url}            
\usepackage{booktabs}       
\usepackage{amsmath}
\usepackage{amsthm}
\usepackage{amssymb}
\usepackage{amsfonts}       
\usepackage{nicefrac}       
\usepackage{microtype}      
\usepackage{cleveref}
\usepackage{bm}
\usepackage{multirow}
\usepackage{balance}
\usepackage{graphics}
\usepackage{graphicx}
\usepackage{float}
\usepackage{subfigure}
\usepackage{epstopdf}
\usepackage{proof}
\usepackage{multicol}
\usepackage{bbm}
\usepackage{caption}
\usepackage{pifont}
\usepackage{titlesec}
\usepackage{titletoc}
\usepackage{lipsum}
\usepackage{atbegshi}
\usepackage{etoolbox}
\usepackage{ragged2e}
\usepackage{newclude}
\usepackage{colortbl}
\usepackage{marvosym}

\usepackage{xcolor}         
\definecolor{ballblue}{rgb}{0.13, 0.67, 0.8}
\definecolor{lightseagreen}{rgb}{0.13, 0.7, 0.67}
\definecolor{darkturq}{HTML}{00CED1}
\definecolor{org}{HTML}{F8A145}
\definecolor{blu}{HTML}{63ACE5}
\definecolor{c1}{HTML}{8ECFC9}
\definecolor{c2}{HTML}{FFBE7A}
\definecolor{c3}{HTML}{FA7F6F}
\definecolor{softred}{HTML}{FE8A71}
\definecolor{softblue}{HTML}{63ACE5}

\hypersetup
{
    colorlinks=true,
    linkcolor=blue,
    citecolor=darkturq 	 	
}

\usepackage{algorithmic}
\usepackage{algorithm}

\newcommand{\cmark}{\ding{51}}

\usepackage{thm-restate}

\newtheorem{thm}{Theorem}

\newtheorem{lem}{Lemma}
\newtheorem{rem}{Remark}
\newtheorem{asm}{Assumption}

\newcommand{\OP}[1]{{\color{org}(\textit{OP#1})}}

\newcommand{\Eqref}[1]{Eq.~(\ref{#1})}
\newcommand{\Tbref}[1]{Tab.~\ref{#1}}
\newcommand{\Fgref}[1]{Fig.~\ref{#1}}

\newcommand{\Thmref}[1]{Thm.~\ref{#1}}
\newcommand{\Lemref}[1]{Lem.~\ref{#1}}

\newcommand{\Asmpref}[1]{Asmp.~\ref{#1}}

\renewcommand{\cref}[1]{Sec.~\ref{#1}}

\newcommand{\Remref}[1]{Rem.~\ref{#1}}

\def \xbari {\bar{\bm{x}}_i}

\def \wij {w_{ij}}
\def \vij {v_{ij}}
\def \whatij {\hat{w}_{ij}}

\def \sij {\zi^\top\zj}
\def \sjl {\zj^\top\zl}

\def \xi {\bm{x}_i}
\def \xj {\bm{x}_j}
\def \zi {\bm{z}_i}
\def \zj {\bm{z}_j}
\def \zl {\bm{z}_l}

\def \phat {\hat{\bm{p}}}
\def \qhat {\hat{\bm{q}}}
\def \one {\mathbbm{1}}

\def \Expt {\mathbb{E}}

\def \Prob {\mathbb{P}}

\def \eg {\textit{e.g.}}
\def \ie {\textit{i.e.}}
\def \etal {\textit{et al.}}
\def \wrt {w.r.t.}

\ifCLASSOPTIONcompsoc
  \usepackage[nocompress]{cite}
\else
  \usepackage{cite}
\fi
\ifCLASSINFOpdf
\else
\fi
\begin{document}

\makeatletter
\newcommand{\discardpages}[1]{
  \xdef\discard@pages{#1}
  \AtBeginShipout{
    \renewcommand*{\do}[1]{
      \ifnum\value{page}=##1\relax%
        \AtBeginShipoutDiscard
        \gdef\do####1{}
      \fi%
    }%
    \expandafter\docsvlist\expandafter{\discard@pages}
  }%
}
\newif\ifkeeppage
\newcommand{\keeppages}[1]{
  \xdef\keep@pages{#1}
  \AtBeginShipout{
    \keeppagefalse%
    \renewcommand*{\do}[1]{
      \ifnum\value{page}=##1\relax%
        \keeppagetrue
        \gdef\do####1{}
      \fi%
    }%
    \expandafter\docsvlist\expandafter{\keep@pages}
    \ifkeeppage\else\AtBeginShipoutDiscard\fi
  }%
}
\makeatother

%

\title{Semantic Concentration for Self-Supervised Dense Representations Learning}
%
%
%
%

\author{Peisong Wen, 
        Qianqian Xu,~\IEEEmembership{Senior Member,~IEEE,}
        Siran Dai, \\
        Runmin Cong,~\IEEEmembership{Senior Member,~IEEE,}
        and~Qingming Huang,~\IEEEmembership{Fellow,~IEEE}
\IEEEcompsocitemizethanks{
\IEEEcompsocthanksitem Peisong Wen is with the School of Computer Science and Technology, University of Chinese Academy of Sciences, Beijing 101408, China (email: \texttt{wenpeisong@ucas.ac.cn}).\protect\\
\IEEEcompsocthanksitem Qianqian Xu is with the State Key Laboratory of AI Safety, Institute of Computing Technology, Chinese Academy of Sciences, Beijing 100190, China, and also with Peng Cheng Laboratory, Shenzhen 518055, China (email: \texttt{xuqianqian@ict.ac.cn}).\protect\\
\IEEEcompsocthanksitem Siran Dai with State Key Laboratory of Information Security (SKLOIS),
Institute of Information Engineering, Chinese Academy of Sciences, Beijing 100093, China, and also with the School of Cyber Security, University of Chinese Academy of Sciences, Beijing 100049, China (email: \texttt{daisiran@iie.ac.cn}).\protect\\
\IEEEcompsocthanksitem Runmin Cong is with the School of Control Science and Engineering and the Key Laboratory of Machine Intelligence and System Control, Ministry of Education, Shandong University, Jinan 250061, Shandong, China (email: \texttt{rmcong@sdu.edu.cn}).\protect\\
\IEEEcompsocthanksitem Qingming Huang is with the School of Computer Science and Technology, University of Chinese Academy of Sciences, Beijing 101408, China, also with the Key Laboratory of Big Data Mining and Knowledge Management (BDKM), University of Chinese Academy of Sciences, Beijing 101408, China, and also with the Key Laboratory of Intelligent Information Processing, Institute of Computing Technology, Chinese Academy of Sciences, Beijing 100190, China (e-mail: \texttt{qmhuang@ucas.ac.cn}).\protect
}
}
\markboth{Journal of \LaTeX\ Class Files,~Vol.~14, No.~8, August~2015}%
{Shell \MakeLowercase{\textit{et al.}}: Bare Demo of IEEEtran.cls for Computer Society Journals}
%



\IEEEtitleabstractindextext{%
\justifying

\begin{abstract}
  Recent advances in image-level self-supervised learning (SSL) have made significant progress, yet learning dense representations for patches remains challenging. Mainstream methods encounter an over-dispersion phenomenon that patches from the same instance/category scatter, harming downstream performance on dense tasks. This work reveals that image-level SSL avoids over-dispersion by involving implicit semantic concentration. Specifically, the non-strict spatial alignment ensures intra-instance consistency, while shared patterns, \textit{i.e.}, similar parts of within-class instances in the input space, ensure inter-image consistency. Unfortunately, these approaches are infeasible for dense SSL due to their spatial sensitivity and complicated scene-centric data. These observations motivate us to explore explicit semantic concentration for dense SSL. First, to break the strict spatial alignment, we propose to distill the patch correspondences. Facing noisy and imbalanced pseudo labels, we propose a noise-tolerant ranking loss. The core idea is extending the Average Precision (AP) loss to continuous targets, such that its decision-agnostic and adaptive focusing properties prevent the student model from being misled. Second, to discriminate the shared patterns from complicated scenes, we propose the object-aware filter to map the output space to an object-based space. Specifically, patches are represented by learnable prototypes of objects via cross-attention. Last but not least, empirical studies across various tasks soundly support the effectiveness of our method. Code is available in \url{https://github.com/KID-7391/CoTAP}.
\end{abstract}

}

\maketitle
\IEEEdisplaynontitleabstractindextext
\IEEEpeerreviewmaketitle

\IEEEraisesectionheading{\section{Introduction}\label{sec:introduction}}
\IEEEPARstart{L}{earning} semantic visual representations without human annotation is a long-standing goal in the computer vision community. Recent swift progress in Self-Supervised Learning (SSL) is gradually realizing it. Ideally, visual representations should satisfy \textit{diversity} between distinct instances, maintain cross-view \textit{alignment} of the same instance, and ensure \textit{concentration} of samples within categories. For diversity, contrastive SSL methods explicitly increase the gap between negative sample pairs \cite{chen2020improved}. On the other hand, non-contrastive methods prevent model collapse via asymmetric techniques such as additional projector \cite{byol}, centering and sharpening \cite{dino}, online clustering \cite{swav}, and uniformity regularization of embedding space \cite{balestrierocontrastive}. By getting rid of the costly negative pair construction, non-contrastive SSL enjoys significant advantages in both efficiency and effectiveness, which has led to its popularity in recent advancements. Regarding alignment, mainstream methods learn the invariance of different views augmented from the same image \cite{dino,swav,wu2018unsupervised,byol}, including multiscale random cropping. These SSL methods surpass supervised ones on state-of-the-art network architectures such as Vision Transformer (ViT) \cite{dosovitskiyimage}. 

Despite their notable success in image-level tasks \cite{tomasev2022pushing}, the performance of these methods on dense tasks such as semantic segmentation and object detection still requires enhancement. The primary deficiency lies in the inadequate alignment of fine-grained representation. As an early trial, Pinheiro \etal~\cite{o2020unsupervised} propose to align representations of matched patches, significantly boosting performance in dense vision tasks. In a similar vein, Ziegler \etal~\cite{leopart} develop the self-distillation framework to learn object parts rather than whole images or instances. This approach outperforms models pretrained in a supervised manner across various image segmentation tasks.

However, unlike image-level SSL methods, existing dense SSL methods might fail to capture the global semantic consistency between different objects across images. This limitation can lead to over-segmentation in dense tasks \cite{li2022dynamic}. As seen in \Fgref{fig:over-disp-1}, the strict spatial alignment leads to performance degradation as training progresses.
We argue that \textbf{without a semantic concentration term, the diversity loss will disperse \textit{potential positive pairs, \ie, patches from different instances/images belonging to the same category}}. Consequently, as illustrated in \Fgref{fig:over-disp-2}, parts belonging to the same category or even an individual instance are misclassified into different categories. Such a phenomenon is called \textit{over-dispersion} in this work.

\begin{figure}[t]
  \centering
  \subfigure[Performance degradation in the later period caused by the lack of semantic concentration (SC).]{
    \includegraphics[scale=0.53]{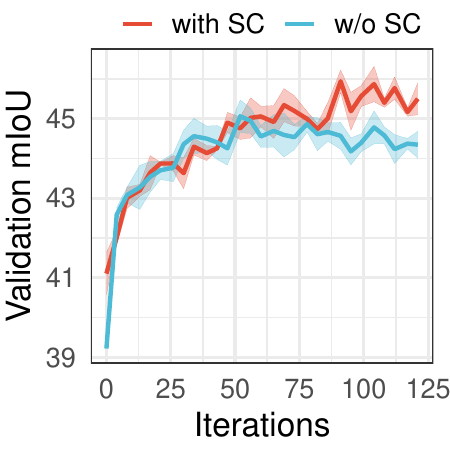}
    \hspace{2mm}
    \label{fig:over-disp-1}
    \includegraphics[scale=0.53]{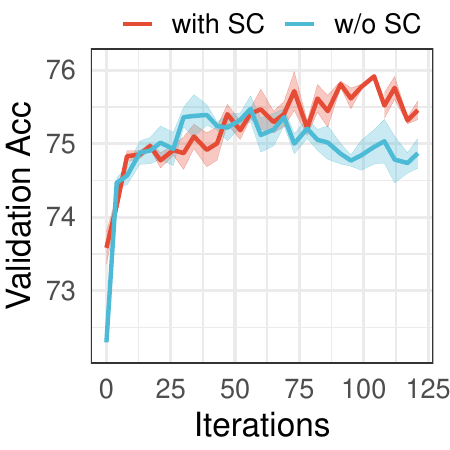}
  }
  \subfigure[Mis-segmentation caused by over-dispersed representations. Patches from the same instance (left) or category (right) are separated.]{
    \hspace{-2mm}
    \label{fig:over-disp-2}
    \includegraphics[scale=0.26]{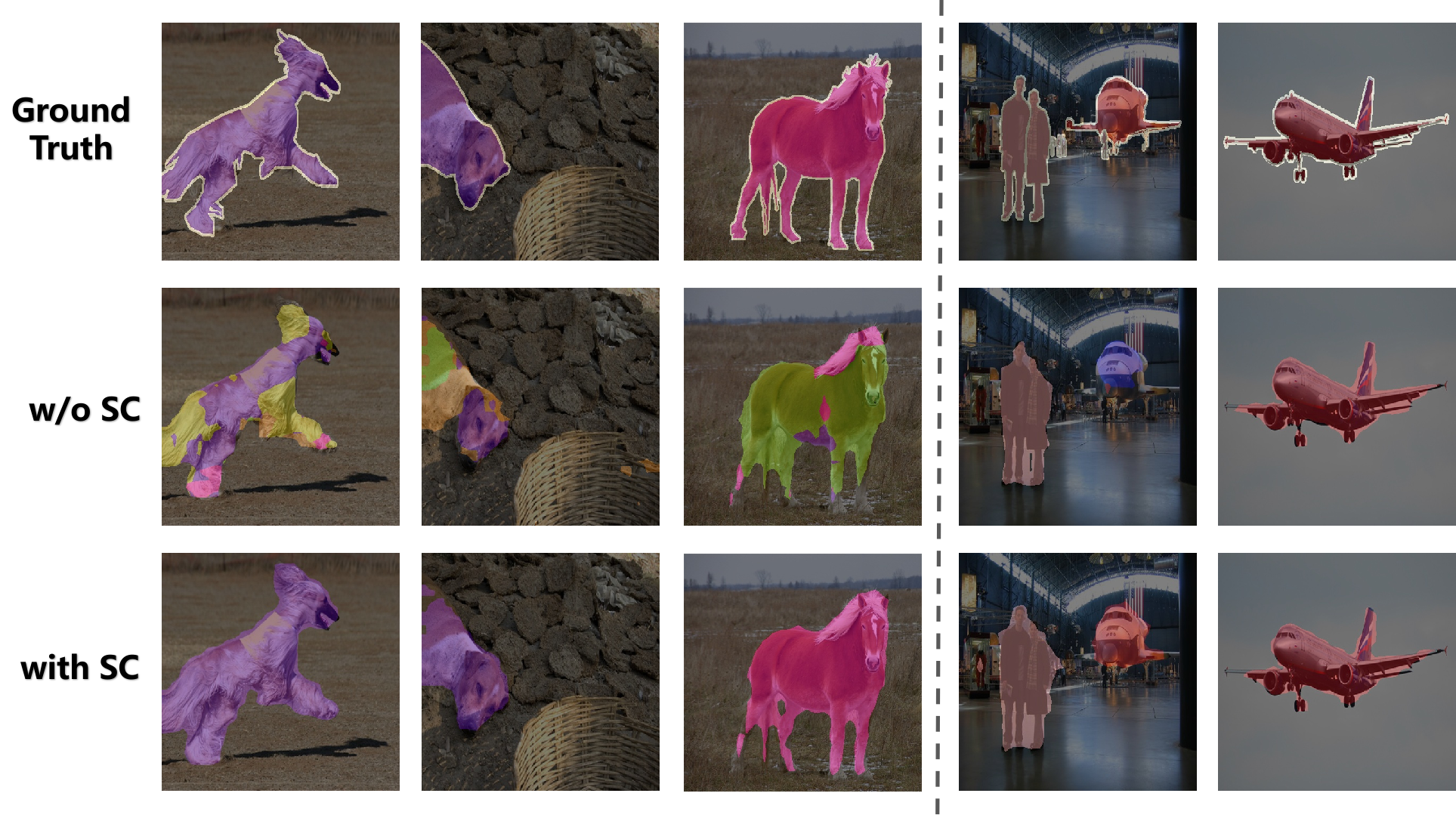}
  }
  \caption{Strict spatial alignment applied in dense SSL leads to poor semantic concentration. }
  \label{fig:demo_over_dispersion}
\end{figure}

Although some image-level SSL methods also use instance discrimination loss, they escape from the over-segmentation phenomenon. To inherit the success to dense SSL, we explore the semantic concentration mechanism of image-level SSL. By analyzing how the optimal representations influence downstream performance, we identify two crucial components: \textbf{1)} Non-strict spatial alignment via random cropping, which keeps the intra-instance consistency; \textbf{2)} Shared patterns via proper augmentations, \ie, objects from the same category exhibit similar local features in the input space, which preserves the inter-image consistency. Unfortunately, these strategies are infeasible for dense SSL. Specifically, fine-grained representation is more susceptible to spatial shifts, making semantic consistency of random crops unguaranteed, particularly in scene-centric images. Second, since scene-centric images are occupied by backgrounds, the commonly used self-attention might be suboptimal to discriminate the shared patterns from complicated scenes.

Motivated by these observations, we aim to improve the semantic concentration in dense SSL on top of the self-distillation framework \cite{dino}:

First, to complement the non-strict spatial alignment that patch discrimination losses lack, we propose to distill the patch correspondences between image pairs.
This approach mitigates over-dispersion by learning the inter-image consistency without involving complicated multi-class clustering. However, a technical challenge is that the soft pseudo labels are imbalanced, noisy, and continuous. To address this, we leverage the Average Precision (AP) loss, a noise-tolerant ranking loss with adaptive weights, to avoid the decision bias introduced by imbalanced distributions and noises. To be compatible with the continuous pseudo labels output by the target model, we extend the original AP loss that only applies to binarized labels to the \textit{\textbf{Co}ntinuous-\textbf{T}arget \textbf{A}verage \textbf{P}recision (\textbf{CoTAP})} loss.

Second, aiming to capture the shared patterns, we propose the \textit{\textbf{O}bject-\textbf{A}ware \textbf{F}ilter (\textbf{OAF})} to map the output features to an object-based space. First, we propose to learn object prototypes from the simple objects of object-centric images. Afterward, inspired by the cross-attention \cite{chen2021crossvit} which is effective in capturing cross-instance information, we treat the patch features as queries to aggregate object prototypes in the last attention block. In this way, the patches are encoded in a subspace spanned by the object prototypes, which is more effective in capturing the shared patterns by eliminating the scene interference.

The main contributions are summarized as follows:
\begin{itemize}
  \item Analytically, we reveal the practical flaw of existing dense SSL frameworks regarding semantic concentration. By analyzing the implicit concentration mechanism of image-level SSL methods, we demonstrate the necessity of developing semantic concentration techniques for dense SSL.
  \item Methodologically, we propose a self-distillation framework for semantic-concentrated dense SSL. To avoid over-dispersion, we propose to explicitly concentrate the potential positive pairs with our noise-tolerant ranking loss. To highlight the shared pattern, we propose an object-aware filter to capture the fine-grained features.
  \item Empirically, we conduct experiments on both image-level and dense downstream tasks to soundly support our conclusions. Moreover, the evaluation results on a variety of benchmarks validate the effectiveness of our proposed framework.
\end{itemize}

\section{Prior Art}\label{sec:related_work}
\subsection{Image-level Self-supervised Learning}

Self-supervised learning (SSL) aims to break through the bottleneck of limited human-annotated data by learning generalizable representations from unlabeled data. In the computer vision community, the mainstream methods encode representations at the image level. Some early work designed various proxy tasks to force the model to learn meaningful representations, \eg, inpainting \cite{pathak2016context}, jigsaw puzzles solving \cite{noroozi2016unsupervised}, and rotation prediction \cite{gidarisunsupervised}. Among them, contrastive learning \cite{chen2020improved, wu2018unsupervised}, which considers augmentations of the same image as positive pairs and others as negative, achieves remarkable success. However, these methods require a large batch size \cite{chen2020simple} or storing negative examples \cite{he2020momentum}. More recently, non-contrastive SSL methods are proposed to overcome this shortage \cite{swav, dino, grill2020bootstrap}. The core idea is only forcing representations of augmentations from the same image to be similar, while introducing asymmetric techniques to avoid trivial solutions, such as centering and sharpening \cite{dino}, additional projector \cite{byol}, and whitening loss \cite{weng2022investigation}. Motivated by the denoising autoencoder \cite{vincent2008extracting}, He \etal~\cite{he2022masked} break away from direct constraints in embedding space and propose to reconstruct masked images in the input space. At the same time, motivated by representative methods on social-aware image understanding \cite{tang2016tri,li2018deep}, another technique routine explores clustering over unlabeled data to capture cross-image information explicitly \cite{zhuang2019local, caron2019unsupervised, huang2019unsupervised}. Although these methods are surpassed by non-contrastive methods, they inspire the clustering methods on dense prediction \cite{cho2021picie, wenself}.

\subsection{Self-supervised Learning for Dense Tasks}

Despite the success of SSL on image-level downstream tasks, the improvement on dense prediction tasks such as semantic segmentation is less significant \cite{he2019rethinking,su2023flsl,zhang2020causal}. This gap raises a wave of research into self-supervised segmentation.
Similar to the image-level problem, researchers propose to learn pixel-wise representations by forcing local features to remain constant after different transformations \cite{o2020unsupervised, xie2021propagate}. However, these methods commonly rely on object-centric images. 
To overcome this deficiency, Wen \etal~\cite{wenself} propose to perform semantic grouping with learnable prototypes. To avoid overfitting to the low-level visual cues caused by clustering, Cho \etal~\cite{cho2021picie} further constrain the geometric and photometric invariance. In this way, both object-centric data and scene-centric data can be segmented without any annotations or preprocessing. Li \etal~\cite{li2022univip} further explore the correlation of scene and instances.




Recently, researchers have found that with a proper model architecture, effective dense representations can be learned even if only image-level similarity is constrained \cite{dino}. Motivated by this fact, Melas \etal~\cite{melas2022deep} propose to formulate the semantic segmentation and localization tasks as graph partitioning, and decompose an image into several objects with spectral methods. Hamilton \etal~\cite{hamiltonunsupervised} propose to add a tiny model on top of the ViT pretrained with DINO \cite{dino}, and then distill the pixel-pixel similarity into high-quality discrete semantic labels. Despite the performance improvement on fully unsupervised segmentation, the above methods fall on transfer learning since the main encoder is not updated. Consequently, Ziegler \etal~\cite{leopart} extend the image-level DINO into patch-level to learn object parts, resulting in outstanding dense representations. Inspired by the masked autoencoder, Zhou \etal~\cite{zhou2021ibot} and Liu \etal~\cite{liu2025future} propose to reconstruct the masked patches in the embedding space, which significantly improves the DINO framework on dense tasks. These techniques demonstrate substantial enhancements, particularly in dense tasks like cellular image representation \cite{dai2025exploring}.

However, in this work, we argue that the dense representations trained by such non-contrastive methods suffer from over-dispersion, leading to performance degradation on dense classification tasks. Su \etal~\cite{su2023flsl} notice a similar phenomenon and attributed it to the fact that dense tasks rely on local semantics rather than global instance-level semantics. Accordingly, they proposed bi-level clustering, which leverages the equivalence between self-attention and Mean-Shift clustering updates. By aligning dense features with cluster centers, this approach learns representations that simultaneously capture both local and global semantics. Similarly, Zhou \etal~\cite{zhou2022mugs} also propose a clustering-based method for dense SSL. This work takes a further step to reveal the critical role of non-strict alignment and shared patterns in SSL, and shows that their absence leads to excessive feature dispersion. To address this issue, we propose distilling pseudo-class labels with a noise-robust soft-label loss, thereby achieving fine-grained feature-to-feature alignment.

\subsection{Theoretical Analysis for Self-supervised Learning}
To build a unified analytical framework for both contrastive and non-contrastive SSL, recent work reveals that some non-contrastive SSL methods (\eg, BYOL \cite{byol}, SimSiam \cite{simsiam}) enjoy similar properties to contrastive methods under certain conditions. For example, both recover spectral embedding optimization problems \cite{balestrierocontrastive}, and learn features that enable uniform clustering \cite{assran2022hidden}. Garrido \etal~\cite{garrido2023on} further show the duality between contrastive and non-contrastive methods and conduct empirical studies on the performance gap. By exploring the dynamics of non-contrastive methods, Tian \etal~\cite{tian2021understanding} point out that the projector used in BYOL and SimSiam leads to aligned eigenspace w.r.t. the correlation matrices of the two branches. On top of this, Richemond \etal~\cite{richemond2023edge} show that the optimal predictor should be close to an orthogonal projection, and exponential moving average and stop gradient are efficient orthonormalization mechanisms. However, the above methods are limited to certain methods, thus infeasible for the latest methods like DINO. To fill this gap, Zhuo \etal~\cite{zhuotowards} propose a unified theory named Rank Differential Mechanism (RDM), where mainstream asymmetric techniques used in non-contrastive SSL potentially enlarge the effective rank of the online branch. The above analyses explain the properties of learned representations, but how the learned encoder generalizes to unknown data is still unclean. To fill this gap, Huang \etal~\cite{huang2022towards} point out that rich augmentations are critical to the clustering and generalization of SSL methods.

On top of the existing research, this work further investigates the mechanism differences between image-level and dense SSL in terms of semantic concentration. Specifically, the theoretical analysis reveals that the non-strict spatial alignment and shared patterns constitute the implicit semantic concentration of image-level SSL, while dense SSL lacks semantic concentration. This guides us to propose semantic concentration techniques for dense SSL. 


\section{Implicit Semantic Concentration Mechanism in Image-level SSL}

\begin{table*}
  \caption{Downstream accuracy under different types of augmentations: (a) random cropping; (b) random Gaussian blur; (c) color dropping; (d) color distortion.}
  \label{tab:effect_aug}
  \centering
  \begin{tabular}{cccc|cccc|cccc}
  \toprule
  \multicolumn{4}{c|}{Transformations} & \multicolumn{4}{c|}{CIFAR-10} &\multicolumn{4}{c}{CIFAR-100} \\
  (a) & (b) & (c) & (d) & SimCLR  & Barlow Twins  & MoCo  & SimSiam & SimCLR  & Barlow Twins  & MoCo  & SimSiam \\ 
  \midrule
  \cmark  &  \cmark  &  \cmark  &  \cmark  & \cellcolor[rgb]{0.569 0.817 0.794}88.5 & \cellcolor[rgb]{0.569 0.817 0.794}85.4 & \cellcolor[rgb]{0.561 0.813 0.790}89.7 & \cellcolor[rgb]{0.566 0.816 0.792}89.3 & \cellcolor[rgb]{1.000 0.759 0.508}55.4 & \cellcolor[rgb]{1.000 0.760 0.509}55.2 & \cellcolor[rgb]{1.000 0.754 0.497}62.5 & \cellcolor[rgb]{1.000 0.759 0.507}60.3 ~\\
  ~  &  \cmark  &  \cmark  &  \cmark & \cellcolor[rgb]{0.889 0.953 0.947}53.1 & \cellcolor[rgb]{0.925 0.968 0.964}39.6 & \cellcolor[rgb]{0.892 0.954 0.948}53.6 & \cellcolor[rgb]{0.921 0.967 0.962}36.8 & \cellcolor[rgb]{1.000 0.984 0.966}19.5 & \cellcolor[rgb]{1.000 0.957 0.913}17.7 & \cellcolor[rgb]{1.000 0.949 0.896}26.4 & \cellcolor[rgb]{1.000 0.990 0.980}8.2 ~\\
  \cmark  &  \cmark  &  \cmark  &  ~  & \cellcolor[rgb]{0.614 0.836 0.815}83.5 & \cellcolor[rgb]{0.595 0.828 0.806}82.0 & \cellcolor[rgb]{0.587 0.825 0.803}86.8 & \cellcolor[rgb]{0.592 0.827 0.805}85.4 & \cellcolor[rgb]{1.000 0.824 0.639}45.1 & \cellcolor[rgb]{1.000 0.785 0.560}50.4 & \cellcolor[rgb]{1.000 0.784 0.558}57.0 & \cellcolor[rgb]{1.000 0.799 0.588}51.4 ~\\
  ~  &  \cmark  &  \cmark  &  ~  & \cellcolor[rgb]{1.000 1.000 1.000}40.8 & \cellcolor[rgb]{0.983 0.993 0.992}32.1 & \cellcolor[rgb]{1.000 1.000 1.000}41.9 & \cellcolor[rgb]{1.000 1.000 1.000}25.2 & \cellcolor[rgb]{1.000 1.000 1.000}16.9 & \cellcolor[rgb]{1.000 1.000 1.000}9.6 & \cellcolor[rgb]{1.000 1.000 1.000}17.0 & \cellcolor[rgb]{1.000 1.000 1.000}6.0 ~\\
  \cmark  &  \cmark  &  ~  &  ~  & \cellcolor[rgb]{0.798 0.914 0.903}63.2 & \cellcolor[rgb]{0.705 0.875 0.859}67.8 & \cellcolor[rgb]{0.695 0.870 0.854}75.1 & \cellcolor[rgb]{0.742 0.890 0.877}63.3 & \cellcolor[rgb]{1.000 0.931 0.858}28.0 & \cellcolor[rgb]{1.000 0.871 0.736}34.1 & \cellcolor[rgb]{1.000 0.875 0.744}40.2 & \cellcolor[rgb]{1.000 0.910 0.816}26.3 ~\\
  ~  &  \cmark  &  ~  &  ~  & \cellcolor[rgb]{0.960 0.983 0.981}45.2 & \cellcolor[rgb]{1.000 1.000 1.000}29.9 & \cellcolor[rgb]{0.942 0.975 0.972}48.2 & \cellcolor[rgb]{0.975 0.989 0.988}28.9 & \cellcolor[rgb]{1.000 0.959 0.917}23.4 & \cellcolor[rgb]{1.000 0.998 0.996}10.0 & \cellcolor[rgb]{1.000 0.973 0.945}22.0 & \cellcolor[rgb]{1.000 0.995 0.990}7.2 ~\\
  \bottomrule
  \end{tabular}%
\end{table*}

\subsection{Intuitive Comprehension}
\label{subsec:motivations}

While image-level SSL methods primarily focus on aligning augmentations of the same image, they avoid over-dispersion and achieve category-aware clustering. To migrate the successful practice of image-level SSL, in this subsection, we take a deep dive into its underlying mechanism and explain the difference between dense SSL and image-level SSL.

According the over-dispersion phenomenon, the target of SSL could be decomposed into: \textbf{1)} Patches of an individual instance share similar representations; \textbf{2)} Instances from the same category should be consistently represented. Subsequently, we employ the illustrative example in \Fgref{fig:implicit} to show how image-level SSL methods accomplish the two objectives.

Intuitively, by aligning the randomly cropped images, parts of the same object will be concentrated. A crucial prior is that different parts of the same instance frequently co-occurr within the same image, leading to target \textbf{1)}.
As for target \textbf{2)}, we argue that the clustering effect stems from the \textit{shared patterns} across intra-class instances. Specifically, with proper augmentations, local parts of different instances could appear similar in the input space. For instance, as shown in \Fgref{fig:implicit}, the faces of two dogs differ in color and posture, yet grayscale conversion and rotation can reveal the underlying shared pattern. With a Lipschitz continuous encoder, the corresponding part representations are still similar, effectively linking different instances.

In a nutshell, the core components that facilitate the implicit semantic concentration of image-level SSL are:
\begin{itemize}
  \item \textbf{Non-strict spatial alignment} achieved by aligning random crops, which leads to consistent representations for parts of an individual instance.
  \item \textbf{Shared patterns.} With proper augmentations, objects from the same category share some similar parts in the input space, resulting in consistent representations across intra-class instances.
\end{itemize}

To support our hypothesis, following Huang \etal \cite{huang2022towards}, we assess the impact of augmentations including random cropping for various image-level SSL methods \cite{chen2020simple,zbontar2021barlow,he2020momentum,simsiam}. The results are shown in \Tbref{tab:effect_aug} (see Appendix {\color{blue}B.1} for implementation details). Accordingly, we have two observations:
\textbf{1)} The exclusion of random cropping results in a notable decline in performance across all groups, especially for CIFAR-100 with more categories. This highlights the necessity of non-strict spatial alignment in achieving semantic concentration.
\textbf{2)} In most of the settings, the downstream performance is correlated to the richness of augmentations. Adequate augmentations allow for the unveiling of shared patterns among instances within the same category, despite differences in the original appearance.

Building on the above analysis, we argue that, compared with image-level SSL, \textbf{the over-dispersion in dense SSL arises from the absence of non-strict spatial alignment and the weakening of shared patterns.} First, due to the lack of fine-grained correspondence annotations, dense SSL typically constrains representation alignment only at the exact spatial location, and fails to enforce intra-instance representation consistency through random cropping. Second, to capture multi-object relationships in complex scenes, some dense SSL methods \cite{leopart,hamiltonunsupervised} adopt scene-centric datasets like COCO instead of instance-centric datasets such as ImageNet-1k. In this case, the influence of background and distractors reduces the frequency of shared patterns. Among these two factors, missing non-strict spatial alignment plays the leading role.


\begin{figure}[t]
  \centering
  \includegraphics[scale=0.6]{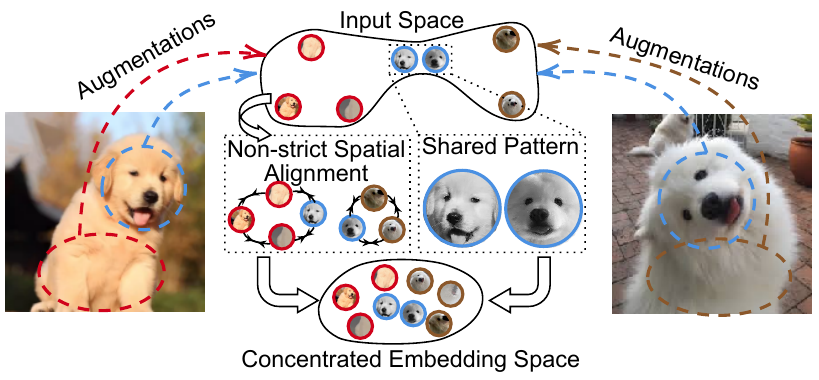}
  \caption{Implicit semantic concentration of image-level SSL.}
  \label{fig:implicit}
\end{figure}

\subsection{Problem Formulation}
\label{subsec:setting}

In addition to the intuitive explanations, the remainder of this section focuses on quantitatively discussing how implicit semantic concentration affects downstream task performance. The subsequent theoretical discussion serves merely as an additional justification of the above hypotheses and is not essential for understanding the method or experiments.
It is worth noting that the following theoretical results hold for both image-level and dense SSL. The main difference lies in the values of certain parameters, as shown in \Remref{rem:ideal}, \Remref{rem:spa-align}, and \Remref{rem:shared-pattern}.

\textbf{Notations.}
For the sake of presentation, we simplify the (dense) SSL as a classification problem with $K$ categories. For the dense problem, each patch is treated as a sample. Specifically, consider a dataset $\bar{\mathcal{D}} = \{\xbari\}_{i=1}^{\bar{N}}$, where $\xbari \in \mathcal{X}$ represents the natural images sampled from the input space $\mathcal{X}$. Typically, each natural image is augmented for $m$ times, yielding an augmented dataset $\mathcal{D} = \{\xi\}_{i=1}^{N}$, where $N = m\times \bar{N}$. Let $y: \mathcal{X} \mapsto \{1,\cdots,K\}$ be a ground-truth labelling function such that $y(\bm{x})$ correctly identifies the category of instance $\bm{x}$. To construct the adjacency matrix, denote $\wij = w(\xi, \xj) = \mathbbm{1}[y(\xi) = y(\xj)]$. Notice that neither $y$ nor $w$ is accessible during the pre-training phase. Therefore, we estimate $\wij$ with $\whatij$, \eg, $\whatij = \mathbbm{1}[\xi \textit{\text{ and }} \xj \textit{\text{ are augmented frome the same image}}]$. For all $k\in [K]$, let $\mathcal{D}^{(k)} = \{j \in [N]~|~y(\xj) = k\}$ be the index set of $k$-th category, and $N_k = |\mathcal{D}^{(k)}|$. For all $i \in [N]$, the in-class samples are divided into two subsets: $\mathcal{T}_{1}^{(i)} = \{j \in [N]~|~\wij = 1 ~\&~ \whatij = 0\}$, encompassing samples sharing the same label as $i$ but not augmented from the same image, and $\mathcal{T}_{2}^{(i)} = \{j \in [N] ~|~ \whatij = 1\}$, including samples augmented from the same image as $i$. 

Our target is to learn an embedding function $f \mapsto \mathbb{R}^d$, such that $f(\bm{x})$ provides a $d$-dimensional representation of the image/patch $\bm{x}$ on downstream tasks. We assume that all representations are normalized: $\|f(\bm{x})\| = r$. Following Huang \etal \cite{huang2022towards}, in this work we consider a simple $k$ Nearest Neighbor ($k$-NN) classifier: 
\begin{equation}
\label{eq:defi_knn_cls}
  G_f (\bm{x}) = \arg \min_{k\in [K]} \|f(\bm{x}) - \bm{\mu}_k\|,
\end{equation}
where $\bm{\mu}_k = \mathbb{E}_{\bm{x}\in \mathcal{D}, y(\bm{x}) = k}[f(\bm{x})]$. On top of this, the downstream performance is evaluated with the error rate:
\begin{equation*}
  \mathcal{E}(f) = \mathbb{P}_{\bm{x}\in \mathcal{D}}[G_f(\bm{x}) \neq y(\bm{x})].
\end{equation*}

Given the varied objective functions utilized by SSL methods, we consider a linear combination of the alignment term and the diversity term to simplify the analysis:
\begin{equation*}
  \mathcal{L}_{SSL} = \mathcal{L}_{align} + \alpha \cdot \mathcal{L}_{div},
\end{equation*}
where $\alpha$ is a hyperparameter to control the trade-off between alignment and diversity. The above formulation covers most of the mainstream SSL methods. For example, contrastive SSL methods \cite{chen2020improved,he2020momentum} use the InfoNCE loss (with only one negative sample):
\begin{equation*}
  \begin{aligned}
    &\mathcal{L}_{INCE} \\
    =& - \sum_{i,j,k=1}^N \whatij (1 - \hat{w}_{ik}) \log\left(\frac{e^{f(\bm{x}_i)^\top f(\bm{x}_j)}}{e^{f(\bm{x}_i)^\top f(\bm{x}_j)} + e^{f(\bm{x}_i)^\top f(\bm{x}_k)}}\right) \\
    =& \underbrace{- (N - m)\sum_{i,j,k=1}^N \whatij f(\bm{x}_i)^\top f(\bm{x}_j)}_{\mathcal{L}_{INCE-align}} \\
    & + \underbrace{\sum_{i,j=1}^N \whatij (1 - \hat{w}_{ik}) \log\left(e^{f(\bm{x}_i)^\top f(\bm{x}_j)} + e^{f(\bm{x}_i)^\top f(\bm{x}_k)}\right)}_{\mathcal{L}_{INCE-div}}.
  \end{aligned}
\end{equation*}
Some non-contrastive SSL methods such as Barlow Twins \cite{zbontar2021barlow} directly constrain the cross-correlation of the features:
\begin{equation*}
  \begin{aligned}
    \mathcal{L}_{BT} = \underbrace{\sum_{i=1}^d (1 - C_{ii})^2}_{\mathcal{L}_{BT-align}} + \alpha \underbrace{\sum_{i=1}^d\sum_{j\neq i} C_{ij}^2}_{\mathcal{L}_{BT-div}},
  \end{aligned}
\end{equation*}
where $C_{ij}$ is the correlation:
\begin{equation*}
  C_{ij} = \frac{1}{\bar{N} M} \sum_{k,l=1}^N \hat{w}_{kl} f_i(\bm{x}_k)f_j(\bm{x}_l).
\end{equation*}
The output is normalized such that $\frac{1}{N}\sum_{k} f_i(\bm{x}_k)^2 = 1$ for each dimension $i$. Some other methods \cite{dino,leopart} leverage an implicit diversity term, which is hard to formulate explicitly. To sum up, $\mathcal{L}_{align}$ forces $f(\bm{x}_i) \approx f(\bm{x}_j)$ if $\whatij = 1$, while $\mathcal{L}_{div}$ achieves its minimum(s) when the representations are uniformly distributed. For the simplification of analysis, in this work, we focus on the following objective:
\begin{equation*}
  \begin{aligned}
    \OP{0}~~\min_{f}~~\mathcal{L}_{SSL} &= \mathcal{L}_{align} + \alpha \mathcal{L}_{div}, \\
    \mathcal{L}_{align} &= \frac{1}{2N^2} \sum_{i,j=1}^N \whatij \|f(\bm{x}_i) - f(\bm{x}_j)\|^2, \\
    \mathcal{L}_{div} &= \frac{1}{2} \Big\|\bm{C}_f - \frac{r^2}{d}\mathbf{I}_{d}\Big\|_F^2,
  \end{aligned}
\end{equation*}
where $\bm{C}_f = \frac{1}{N} \sum_{i=1}^{N} f(\bm{x}_i)f(\bm{x}_i)^\top$ is the feature correlation matrix and $\mathbf{I}_d$ is a d-dimension identity matrix. Notice that the outputs are normalized, \ie, $\|f(\bm{x})\| = r$, thus $\mathcal{L}_{align}$ is equivalent to $\mathcal{L}_{INCE-align}$ and $\mathcal{L}_{BT-align}$ to a certain extent:
\begin{equation*}
  \begin{aligned}
    &~~~~~~~~~~~~N^2\mathcal{L}_{align} = \mathcal{L}_{INCE-align} / (N-m) + N^2r^2, \\
    &N / d \cdot \mathcal{L}_{align} + 1 - r^2/d \leq \mathcal{L}_{BT-align} \leq N \mathcal{L}_{align} + d - r^2.
  \end{aligned}
\end{equation*}
Although the diversity terms in SSL methods are various, recent research \cite{zhuo2023towards} reveals a unified theoretical understanding that most non-contrastive SSL methods reduce the effective rank \cite{roy2007effective} of $\bm{C}_f$, which achieves its minimum when $C_f = \frac{r^2}{d}\bm{I}_d$. Similarly, contrastive SSL methods prefer uniformly distributed representations, which can be effectively measured by $\mathcal{L}_{div}$.

\subsection{Theoretical Analysis}
Current objective \OP{0} has limitations in explaining the semantic concentration of existing image-level SSL methods. Intuitively, if both the capacity of $f$ and the number of samples are sufficiently large, then $f(\bm{x}_i) = f(\bm{x}_j)$ if $\whatij = 1$ and $f(\bm{x}_i)^\top f(\bm{x}_j) \approx 0$ otherwise (see Rem. \ref{rem:case_fail} for formal discussions). This solution violates the practical observation \cite{he2020momentum, dino, hamiltonunsupervised} that $f(\bm{x}_i)$ is closed to $f(\bm{x}_j)$ if $\bm{x}_i$ and $\bm{x}_j$ are semantically similar even if $\whatij = 0$. Therefore, to reflect the shared pattern mechanism, the encoder $f$ should be Lipschitz-continuous, \ie, insensitive to a small shift in the input space:

\begin{asm}[\textbf{Lipschitz Continous Encoder}]
\label{asm:lip_cont}
  Given any two samples $\bm{x},\bm{x}'\in \mathcal{X}$, we have $\|f(\bm{x}) - f(\bm{x}')\| \leq \phi_f \|\bm{x} - \bm{x}'\|$, where $\phi_f$ is a finite constant.
\end{asm}
The above assumption brings constrains on similar sample pairs: $\|f(\bm{x}_i) - f(\bm{x}_j)\| \leq \phi_f\|\xi - \xj\|$, or equivalently:
\begin{equation*}
  f(\bm{x}_i)^\top f(\bm{x}_j) \geq r^2 - \frac{\phi_f^2\|\xi - \xj\|^2}{2}.
\end{equation*}
For the sake of presentation, denote the representation matrix as $\bm{Z} \in \mathbb{R}^{d\times N}$, where the $i$-th column of $\bm{Z}$ is $\zi = f(\bm{x}_i)$, then \OP{0} under the above constraints can be reformulated as:
\begin{equation*}
  \begin{aligned}
    &~~~~\OP{1}~~\min_{\bm{Z}}~~\mathcal{L}_{SSL} = \mathcal{L}_{align} + \alpha \mathcal{L}_{div}, \\
    &\mathcal{L}_{align} = \frac{1}{2N^2} \sum_{i,j=1}^N \whatij \|\zi - \zj\|^2, \\
    &~~\mathcal{L}_{div} = \frac{1}{2} \Big\|\bm{C}_f - \frac{r^2}{d}\mathbf{I}_{d}\Big\|_F^2 = \frac{1}{2} \Big\|\frac{1}{N}\sum_{i=1}^N \zi\zi^\top - \frac{r^2}{d}\mathbf{I}_{d}\Big\|_F^2, \\
    &s.t. ~~~~\zi^\top\zj \geq r^2 - \frac{\phi_f^2\|\xi - \xj\|^2}{2}, ~~\forall i,j \in [N].
  \end{aligned}
\end{equation*}
In the rest of this subsection, we explore how the optimal solutions of \OP{1} perform in the downstream task. Proofs are attached in  Appendix {\color{blue}A}. Intuitively, with a sufficiently large embedding space, $\mathcal{L}_{align}$ shrinks the cross-view distance, while $\mathcal{L}_{div}$ prefers irrelevant representations that push away the inter-class samples. As for the intra-class distance, it is controlled by the constraints. This motivates us to consider the intra-class similarities in the input space:
\begin{asm}
\label{asm:in-of-class-dist}
  There exist two constants $-r^2 \leq d_T \leq r^2$, $0 \leq q_T \leq 1$ such that $\mathbb{P}_{j\in \mathcal{T}_2^{(i)}, l\in \mathcal{T}_1^{(j)}}[\|\xj - \bm{x}_l\|^2 \geq \frac{2(r^2 - d_T)}{\phi_f^2}] \leq q_T$ for any sample $\xi \in \mathcal{D}$.
\end{asm}
\begin{rem}
  \Asmpref{asm:in-of-class-dist} always holds for any data distribution if $d_T \rightarrow -r^2$ or $q_T \rightarrow 1$, and $q_T$ is monotonically increasing to one as $d_T$ increases. From the perspective of semantic concentration, data with well-defined patterns should enjoy a high $q_T$ even though $d_T$ is close to $r^2$. 
\end{rem}

To begin with, we provide a sufficient condition that an augmented sample $\bm{x}_i \in \mathcal{D}$ can be correctly classified by the $k$-NN classifier (defined in \Eqref{eq:defi_knn_cls}):

\begin{restatable}{lem}{ClsCorrect}
\label{lem:condition_cls_correct}
  For any sample $\xi \in \mathcal{D}$, it can be correctly classified, \ie, $G_f(\xi) = y(\xi)$ if 
  \begin{equation*}
    \begin{aligned}
      &\underbrace{\mathbb{E}_{j\in \mathcal{T}_2^{(i)}, l\in \mathcal{D}^{(y(\xi))}}[\bm{z}_j^\top \bm{z}_l]}_{\text{intra-class similarity}}
      - \underbrace{\mathbb{E}_{j\in \mathcal{T}_2^{(i)}, l\in \mathcal{D}^{(k)}}[\bm{z}_j^\top \bm{z}_l]}_{\text{cross-view similarity}} \\
      &\quad\quad\quad\quad\quad\quad\geq 4r \underbrace{\big\|\zi - \mathbb{E}_{j\in \mathcal{T}_2^{(i)}}[\zj]\big\|}_{\text{inter distance}}, ~~\forall k\in [K] \setminus \{y(\xi)\}.
    \end{aligned}
  \end{equation*}
\end{restatable}

\begin{rem}
  \Lemref{lem:condition_cls_correct} highlights that accuracy is affected by three factors: cross-view (alignment), inter-class similarity (uniformity), and intra-class similarity (semantic concentration). Huang \etal \cite{huang2022towards} points out that rich data augmentations lead to sharper concentration, leading to a better generalization. However, we argue that it might simultaneously impact the other factors. Therefore, this work focuses on the trade-off between these factors and delves into the mechanism of implicit semantic concentration.
\end{rem}

Next, we explore how to achieve the above condition through SSL. The following result shows the trade-off between similarities of different sample pairs:
\begin{restatable}{lem}{SumDelta}
\label{lem:sum_delta}
  For any $i \in [N]$, $k \in [K]\setminus \{y(\xi)\}$, consider the following terms reflecting the representation similarities:
  \begin{equation*}
    \begin{aligned}
      \Delta^{(i)}_1 &= \sum_{j\in \mathcal{T}_1^{(i)}}\big(r^4 - (\sij)^2\big), \\
      \Delta^{(i)}_2 &= \sum_{j\in \mathcal{T}_2^{(i)}}\Big(\big(r^2 - \frac{1}{4\alpha}\big)^2 - \big(\sij - \frac{1}{4\alpha}\big)^2\Big), \\
      \Delta^{(i,k)} &= \sum_{j\in \mathcal{D}^{(k)}}\big(r^4 - (\sij)^2\big).
    \end{aligned}
  \end{equation*}
  The following equality holds when $\bm{Z}$ achieves the optimum of \OP{1}:
  \begin{equation*}
    \Delta^{(i)}_1 + \Delta^{(i)}_2 + \sum_{k\in [K]\setminus \{y(\xi)\}}\Delta^{(i,k)} \geq \big(1 - \frac{1}{d}\big)N r^4.
  \end{equation*}
\end{restatable}
\begin{rem}
  It can be verified that when $\alpha \leq 1 / (4r^2)$, both $\Delta^{(i)}_1$ and $\Delta^{(i,k)}$ are non-negative and monotonically decreasing \wrt~ $|\sij|$, while $\Delta^{(i)}_2$ is non-positive and monotonically increasing. This lemma indicates that the cross-view similarity is positively correlated with the other two terms. Hence, without \Asmpref{asm:in-of-class-dist}, the alignment operation will sacrifice the semantic concentration.
\end{rem}
\begin{rem}
\label{rem:case_fail}
  The equation holds for \OP{0}. Consider $\phi_f \rightarrow \infty$ and $d \rightarrow \infty$, we get $\Delta^{(i)}_1 = \Delta^{(i,k)} = r^2$, $\Delta^{(i)}_2 = 0$, or equivalently $\sij = \whatij r^2$. In this case, the encoder fails to capture the semantic concentration. 
\end{rem}

\begin{thm}[Informal]
\label{thm:main}
  Let \Asmpref{asm:in-of-class-dist} hold and assume that all categories have $n = \bar{N} / K$ samples. Consider a proper feature dimension such that $K / 2 \leq d \leq K$. Denote $p_T = \sup_{i\in [N]}\Prob_{j,l \in \mathcal{T}_2^{(i)}}\left[\sjl \leq \delta_{T} \right]$. By minimizing $\mathcal{L}_{SSL}$ and 
  setting $\delta_T \geq d_T \geq 10r^2/13$, there exists positive scales $C_1 \leq (3d_T + \delta_T - r^2) / 2, C_2 > 0, C_3 \geq 2r^2$, such that for all $q_T, p_T$ with $C_1 > C_2 \sqrt{p_T} + C_3 \sqrt{q_T}$, the error rate is upper bounded:
  \begin{equation*}
    \mathcal{E}(f) = O\left(\frac{r^2 - \delta_T + r^2 \sqrt{p_T}}{(C_1 - C_2 \sqrt{p_T} - C_3 \sqrt{q_T})^2}\right).
  \end{equation*}
\end{thm}

\begin{rem}[Ideal encoder]
\label{rem:ideal}
  Given fixed $d_T, \delta_T$, the error rate $\mathcal{E}(f)$ decreases to $O\big((r^2 - \delta_T) / C_1^2\big)$ as $q_T,p_T \rightarrow 0$, \ie, indicating that intra-class samples are almost certainly well-aligned. Consider a special case that $d_T, \delta_T \rightarrow r^2$ and $q_T, p_T \rightarrow 0$, we achieve an optimal classifier with $\mathcal{E}(f) \rightarrow 0$.
\end{rem}

\begin{rem}[Role of non-strict spatial alignment]
\label{rem:spa-align}
  In contrast, consider a bad case where augmented views are poorly aligned, \ie, $p_T \rightarrow 1$ for a small $\delta_T < r^2$. In this case, the upper bound is larger than $O\big((2r^2 - \delta_T) / (C_1 - C_2)^2\big)$, leading to a weak classifier.
\end{rem}

\begin{rem}[Role of shared pattern]
\label{rem:shared-pattern}
  Another failure case is that the intra-class samples are dispersed in the input space, \ie, $q_T \rightarrow 1$ with a small $d_T$. In this situation, \OP{1} degenerates into an unconstrained problem, and $C_1 - C_2 \sqrt{p_T} - C_3 \sqrt{q_T} \leq 3r^2/2 - 2 \sqrt{q_T} r^2 $ tends to zero or even a negative number, making the error rate unbounded.
\end{rem}

\begin{rem}
  The condition $K / 2 \leq d \leq K$ ensures the embedding space has a proper capacity, allowing representations of irrelevant classes to be approximately orthogonal with a sufficient capacity. On the other hand, an over-complicated embedding space might lead to the curse of dimensionality.
\end{rem}

In summary, the mechanism behind implicit semantic concentration is two-fold:
\begin{itemize}
  \item[\textbf{a)}] The shared patterns depend on the part similarities in the input space, described by $(d_T, q_T)$. With a large $d_T$ and a small $q_T$, the encoder groups instances into category-aware clusters via the correlation between local pattern similarity and category, thereby \textbf{instance-discriminative losses can produce semantic-aware embeddings across images}.
  \item[\textbf{b)}] Non-strict spatial alignment depends on $(\delta_T, p_T)$ determined by the augmentations of difference views. With a large $\delta_T$ and a small $p_T$, \textbf{it avoids over-dispersion by connecting parts of an object cropped from the same images, even if their appearances exhibit wide disparities.}
\end{itemize}

\begin{figure*}[t]
  \centering
  \vspace{-2mm}
  \includegraphics[scale=1]{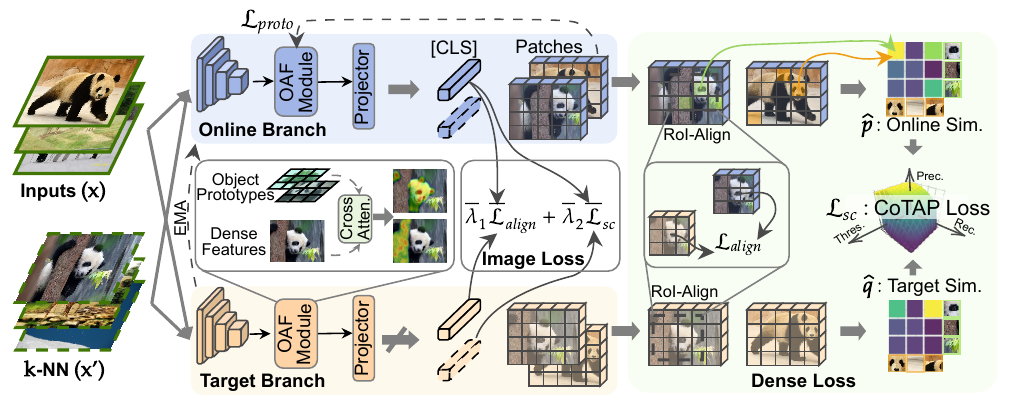}
  \vspace{-2mm}
  \caption{An overview of the proposed framework. The model consists of the target branch and the online branch. The image-level representations of the input image and its $k$-NN are aligned with $\bar{\mathcal{L}}_{align}$ and $\bar{\mathcal{L}}_{sc}$. Patch representations in the same position are aligned with $\mathcal{L}_{align}$, and the correspondences are distilled with the proposed CoTAP loss.}
  \label{fig:overview}
\end{figure*}

\section{Methodology}
\label{sec:method}

\subsection{Motivations}
As analyzed in the previous subsection, non-strict spatial alignment and shared patterns build the foundation of an SSL algorithm. However, this paper argues that \textbf{these key elements may have been overlooked in previous studies on dense SSL, leading to a weak semantic concentration.}

For the sake of comparison, we discuss some practical operations connected to our hypothesis. As a standard setup, state-of-the-art methods \cite{byol,dino,simsiam,swav,li2023sere} on image-level SSL align views randomly cropped from various positions and scales. These random crops from object-centric datasets such as ImageNet \cite{russakovsky2015imagenet} preserve semantic consistency with a high probability. This practice verifies the necessity of \textit{non-strict spatial alignment}. Moreover, with proper transformations such as color distortion and color dropping, the input space is mapped such that the parts of semantically similar objects are clustered \cite{huang2022towards}.

On the contrary, dense SSL methods apply a different strategy. First, mainstream methods \cite{leopart,o2020unsupervised} utilize \textit{strict spatial alignment}. Specifically, two views are aligned in a patch-to-patch manner where only patches from the same position are matched. It contradicts the non-strict spatial alignment principle, leading to over-dispersion. Second, since dense tasks require discriminative features from complicated backgrounds, scene-centric images are preferred. However, due to the low proportion of foreground objects, the distance between object parts is distorted by scenes, reducing the similarity in the input space.

\subsection{Overview}

Given the above challenges, investigating the mechanics of semantic concentration within the context of dense SSL emerges as a promising direction for developing effective dense representations. In this section, we present a dense SSL framework with explicit semantic concentration components in a self-distillation manner. 

The structure of our proposed framework is illustrated in \Fgref{fig:overview}. The model contains two branches: the online branch and the target branch. First, we introduce conventional alignment losses for both image-level and dense alignment. Then we leverage a straightforward yet effective strategy as explicit semantic concentration at the image level. Furthermore, we propose two techniques aimed at enhancing semantic concentration for dense representations.
The key idea consists of two aspects: \textbf{1)} In \cref{subsec:method:slap}, to compensate the absence of the non-strict spatial alignment, we propose to align patches across images by distilling the patch correspondences via a noise-tolerant ranking loss; \textbf{2)} In \cref{subsec:method:obj_filter}, we propose to magnify the shared pattern effect by filtering the object-centric features through the employment of learnable object prototypes.

\textbf{Basic notations.}
We consider an unlabeled dataset $\mathcal{D} = \{\xi \in \mathcal{X}\}_{i=1}^N$ where the input space is set to natural image space $\mathcal{X} \subset \mathbb{R}^{H_I \times W_I \times 3}$. Given a natural image $\bm{x}$, we create multiple views via random augmentations. For the sake of presentation, we consider two views $v_1(\bm{x})$ and $v_2(\bm{x})$. The dense encoder contains two branches $f_t, f_o: \mathcal{X}\mapsto \mathbb{R}^{H \times W \times D}$, where $f_t, f_o$ are called target branch and online branch respectively, $H, W$ are the height and width of the feature map respectively, $D$ is the number of feature dimensions. Besides, the model also contains two branches to encoder the image-level representations: $\bar{f}_t, \bar{f}_o: \mathcal{X}\mapsto \mathbb{R}^{D}$. The model architecture is implemented as ViT \cite{dosovitskiyimage}, where the dense and image-level representations are available from the embeddings of patch tokens and the \texttt{[CLS]} token. Therefore, $f_t$ and $\bar{f}_t$ share parameters $\bm{\theta}_t$, and similarly denote $\bm{\theta}_o$. Denote $f_t(\cdot)_i, f_o(\cdot)_i$ as the $i$-th spatial token, where $i \in [HW]$. Here the feature maps are $L_2$ normalized, \ie, $\|f_t(\cdot)_i\| = \|f_o(\cdot)_i\| = 1$. Given two input images $\bm{u}, \bm{v}$, the patch correspondence maps $\bm{S}_t(\bm{u}, \bm{v}),\bm{S}_t(\bm{u}, \bm{v}) \in [0,1]^{HW \times HW}$ are composed of the scaled cosine similarity of all patch pairs $1 \leq i,j \leq HW$:
\begin{equation}
  \begin{aligned}
    \bm{S}_t(\bm{u}, \bm{v})_{ij} = f_t(\bm{u})_i^\top f_t(\bm{v})_j / 2 + 0.5,\\
    \bm{S}_o(\bm{u}, \bm{v})_{ij} = f_o(\bm{u})_i^\top f_o(\bm{v})_j / 2 + 0.5.    
  \end{aligned}
\end{equation}

\textbf{Cross-view alignment.}
Following \cite{dino,leopart,asano2019self}, we normalize the online outputs and utilize the Sinkhorn-Knopp algorithm (denoted as SK) \cite{cuturi2013sinkhorn} to sharpen the distribution of target representations:
\begin{equation}
  \bar{p}(\bm{x}) = \sigma(\bar{f_o}(\bm{x})), ~~\bar{q}(\bm{x}) = \text{SK}(\bar{f_t}(\bm{x})),
\end{equation}
where $\sigma(\cdot)$ refers to the softmax operation along the feature dimension. 
As for dense representations, following \cite{leopart} we extract the overlap area of two views $v_1, v_2$ with RoI-Align (denoted as RoI-A) \cite{he2017mask}:
\begin{equation}
  \begin{aligned}
    \hat{f}_t(v_1(\bm{x})), \hat{f}_o(v_2(\bm{x})) &= \text{RoI-A}(f_t(v_1(\bm{x})), f_o(v_2(\bm{x}))), \\
    p(\bm{x})_i = \sigma(\hat{f}_o(&\bm{x})_i), ~~q(\bm{x})_i = \text{SK}(\hat{f}_o(\bm{x})_i).
  \end{aligned}
\end{equation}
Then we minimize the distribution shift between two branches:
\begin{equation}
\label{eq:}
  \begin{aligned}
    \mathcal{L}_{align}(\bm{x}) &= \frac{1}{HW}\sum_{i=1}^{HW} \ell_{ce}\Big(p(v_1(\bm{x}))_i,~ [q(v_2(\bm{x}))_i]_{sg}\Big), \\
    \bar{\mathcal{L}}_{align}(\bm{x}) &= \ell_{ce}\Big(\bar{p}(v_1(\bm{x})), ~[\bar{q}(v_2(\bm{x}))]_{sg}\Big),
  \end{aligned}
\end{equation}
where $\ell_{ce} (\bm{a}, \bm{b}) = - \sum_{i=1}^D \bm{b}_i \log (\bm{a}_i)$ is the cross-entropy loss, and $[\cdot]_{sg}$ refers to the stop-gradient operation. 

\textbf{Embedding diversity via asymmetric update.} Following the non-contrastive SSL route \cite{dino}, our framework involves no explicit diversity term. Instead, it blocks the parameter gradients of the target branch to avoid collapse solutions and updates the target branch by ensembling the history models of the online branch. Specifically, denote the parameters of target and online branches at step $k$ as $\bm{\theta}_t^{(k)}$ and $\bm{\theta}_o^{(k)}$ respectively, then the target branch is updated with an exponential moving average strategy:
\begin{equation}
  \bm{\theta}_t^{(k+1)} = \beta \bm{\theta}_t^{(k)} + (1 - \beta) \bm{\theta}_o^{(k)},
\end{equation}
where $\beta$ is a tunable hyperparameter. In practice, $\beta$ is gradually increased to $1$ as the training progresses. Similar to the mean teacher \cite{tarvainen2017mean} used in semi-supervised learning, the target branch plays the role of model ensembling and provides guidance for the online branch. The asymmetric update rules of two branches implicitly increase the embedding diversity \cite{dino,zhou2021ibot}.

\textbf{Training pipeline.} Due to the high computational cost to train a ViT-based model from scratch, we deploy the proposed method by fine-tuning existing SSL models. Concretely, at the first stage the model is trained with alignment losses such as $\bar{\mathcal{L}}_{align}$ (DINO \cite{dino}) and $\mathcal{L}_{align}$ (Leopart \cite{leopart}), which provides a good initialization. Afterward, at the second stage, we fine-tune the model for a few epochs by adding the semantic concentration techniques as presented below.



\subsection{Correspondence Distillation via Ranking Loss}
\label{subsec:method:slap}

In this subsection, we present the correspondence distillation to introduce explicit semantic concentration. To verify our hypothesis, we leverage a simple strategy to extend the implicit semantic concentration into an explicit one for image-level representations.

\textbf{Explicit semantic concentration for image-level SSL.} The core idea is to align representations of semantically similar objects. For clean object-centric data, most images contain a single object, thus the inter-image similarity could provide effective guidance. Specifically, given an anchor image $\bm{x}$, we identify its $k$-nearest neighbors to act as potential positive samples and apply the alignment loss:
\begin{equation}
  \bar{\mathcal{L}}_{sc}(\bm{x}) = \ell_{ce}\big(\bar{p}(v_1(\bm{x})), [\bar{q}(v_2(\bm{x}'))]_{sg}\big),
\end{equation}
where $\bm{x}'$ is randomly sampled from $k$-NN of $\bm{x}$. 

To avoid the high computational burden associated with online searching, we leverage a pretrained model to preprocess the $k$-NNs. We extract the deep features of the training set and then search the $k$-NNs based on the \texttt{[CLS]} tokens. In each iteration, we randomly sample a positive image $\bm{x}'$ from the top-$k$ similar images for an input image $\bm{x}$. All these operations are executed during the preprocessing phase, thus the additional time consumption is ignorable. 

\textbf{Patch correspondence distillation.} Despite the significant improvement of simple concentration for image-level representations (see \cref{subsec:ablations}), transferring the simple strategy to patches does not yield the same success, often leading to performance degradation. The main reason lies in that the cross-image similarity cannot reflect whether patches are matched, especially for scene-centric images with diverse objects and scenes. In addition, it is challenging to online sample or process $k$-NN patches.

In this case, without fine-grained annotations, assigning a ``hard'' matched sample to $\bm{x}$ might introduce noises disturbing the model. Instead, we turn to a ``soft'' assignment guided by the correspondences predicted by the target branch, \ie, $\bm{S}_t$ is considered as supervision to refine $\bm{S}_o$. Specifically, given a pair of images $(\bm{u}, \bm{v})$, we extract the soft pseudo labels $\bm{S}_t(v_1(\bm{u}), v_1(\bm{v}))$ indicating whether two patches belong to the same category.

At first glance, the online branch can be updated by minimizing the binary classification loss:
\begin{equation}
  \begin{aligned}
    \mathcal{L}_{sc}(\bm{u}, \bm{v}) = \frac{1}{N_{pair}}\sum_{k=1}^{N_{pair}} & \ell_{bce} \big(\phat_k, \qhat_k\big),
  \end{aligned}
\end{equation}
where $N_{pair} = (HW)^2$, $\ell_{bce}$ refers to the binary cross-entropy loss, $\hat{\bm{p}} = \bm{S}_o(v_1(\bm{u}), v_1(\bm{v}))$, $\hat{\bm{q}} = \bm{S}_t(v_2(\bm{u}), v_2(\bm{v}))$, and $\phat_k,\qhat_k$ refers to the $k$-th patch pair. However, the pseudo labels suffer from two weaknesses: \textbf{1)} The distribution is highly skewed, \ie, the mismatched patch pairs occupy the majority, potentially leading to decision bias; \textbf{2)} Limited by the capability of the target branch, the presence of noise is inevitable. Consequently, the binary cross-entropy loss cannot effectively distill the correspondence information from the target branch.

\textbf{Ranking-based loss for noise-tolerant learning.} In searching for an alternative to instance-wise binary classification loss, we turn to ranking losses. Literature \cite{raghavan1989critical,yang2022auc} suggests that \textit{by learning the relative ranking attributes between samples, ranking losses are insensitive to the skewed distribution}. Besides, as shown in \Fgref{fig:noise_rate}, the noise rate concerning positive samples decreases within the higher scoring intervals, particularly after fine-tuning the model with our approach. Therefore, \textit{by assigning higher weights to sample pairs with top scores, the negative effect of noises can be alleviated}. Therefore, among various ranking losses, the Average Precision (AP) loss is more suitable for the noise distribution. Concretely, given a threshold $t$ that discretizes to $\{0, 1\}$, the AP loss is formulated as \cite{wen2022exploring,wen2024algorithm}:
\begin{equation}
\label{eq:def_auprc}
  \begin{aligned}
    \ell_{AP}(\bm{u}, \bm{v}; t)
    =& \frac{1}{\sum_{i=1}^{N_{pair}}\one[\qhat_i \geq t]} \sum_{\qhat_i \geq t} \frac{R^-(\phat_i)}{R^+(\phat_i) + R^-(\phat_i)} \\
    =& \frac{1}{\sum_{i=1}^{N_{pair}}\one[\qhat_i \geq t]} \sum_{\qhat_i \geq t} g\left(R^-(\phat_i) / R^+(\phat_i)\right),
  \end{aligned}
\end{equation}
where $g(x) = x / (1 + x)$ is a monotonically increasing function, $\mathbbm{1}[\cdot]$ is the indicator function. $R^+,R^-$ denote the rankings of $\phat_i$ among the positive and negative sets respectively, \ie, $R^+(\phat_i) = \sum_{\qhat_j \geq t} \one[\phat_i \leq \phat_j]$, $R^-(\phat_i) = \sum_{\qhat_j < t} \one[\phat_i \leq \phat_j]$.
Consider the basic unit $R^-(\phat_i) / R^+(\phat_i)$, we aim to minimize the loss $R^-(\phat_i)$, while $R^+(\phat_i)$ servers as ranking-based adaptive weights. From this perspective, the loss of the top-$k$ positive pairs is assigned a weight of $1 / k$, which helps mitigate the influence of noise.

However, the original AP loss requires a threshold $t$ to discrete the continuous target $\qhat_i$. In the training process, the optimal threshold shifts due to the score distribution drift, thus a static threshold cannot capture the dynamic conditions. Motived by threshold-metric curve optimization used in multi-label classification \cite{wang2022optimizing} and recent advances on robust learning \cite{ye2023sequence,ye2025towards,ye2024robust}, our approach encompasses multiple thresholds to better adapt to these dynamics, as illustrated in \Fgref{fig:overview_cotap}. Specifically, we propose the \textit{\textbf{Co}ntinuous-\textbf{T}arget \textbf{AP} (\textbf{CoTAP})} loss:
\begin{equation}
\label{eq:def_ctap}
  \begin{aligned}
    \hat{\ell}_{CoTAP}(\bm{u}, \bm{v}) = \frac{1}{N_{pair}}\sum_{k=1}^{N_{pair}} \gamma(\qhat_k) \cdot \ell_{AP}(\bm{u}, \bm{v};\qhat_k),
  \end{aligned}
\end{equation}
where $\gamma$ is a weighting function. $\ell_{CoTAP}$ is a weighted average of the AP loss with a threshold $\qhat_k$ since $\ell_{AP}$ changes only if the threshold belongs to $\qhat$.

\begin{figure}[t]
  \centering
  \vspace{-2mm}
  \subfigure[Pretrained model of DINO.]{
    \includegraphics[scale=0.4]{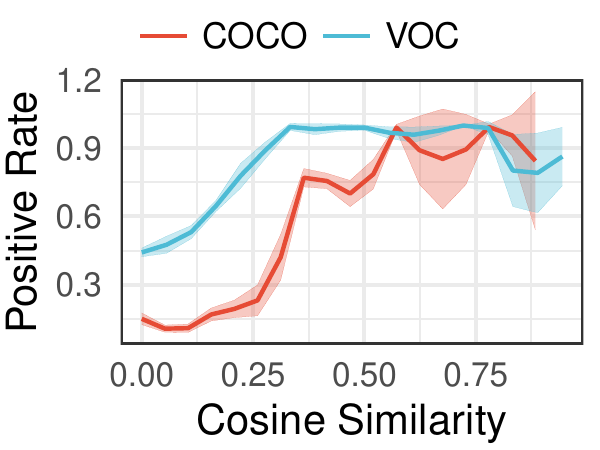}
  }
  \subfigure[Fine-tuned model of ours.]{
    \includegraphics[scale=0.4]{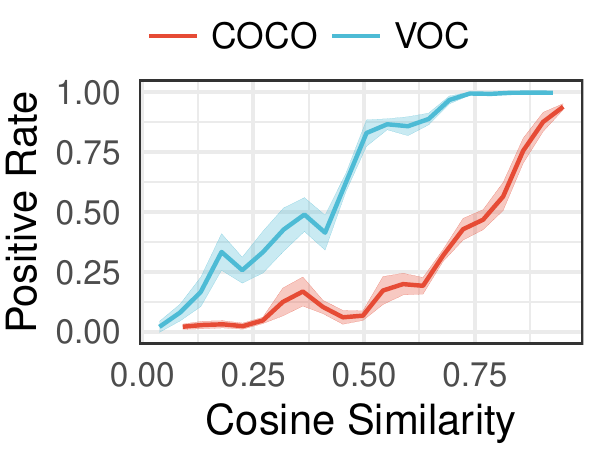}
  }
  \vspace{-2mm}
  \caption{Positive rate \textit{v.s.} similarity.}
  \label{fig:noise_rate}
\end{figure}

It is time-consuming to directly compute the CoTAP loss since it involves three summations. To break through the efficiency bottleneck, we consider an upper bound of $\ell_{CoTAP}$:
\begin{restatable}{prop}{CotapUpBnd}
\label{prop:cotap_up_bnd}
  For any input image pairs $(\bm{u}, \bm{v})$, we have
  \begin{equation*}
    \begin{aligned}
      \hat{\ell}_{CoTAP}(\bm{u}, \bm{v}) \leq \frac{1}{N_{pair}}\sum_{i=1}^{N_{pair}} &\tilde{\gamma}_i \cdot g\Big(
        \psi_i \sum_{\qhat_j < \qhat_i} \one[\phat_i \leq \phat_j]
      \Big), \\
      \psi_i = 1 / \sum_{\qhat_j \geq \qhat_i} \one[\phat_i \leq \phat_j],
      &~~\tilde{\gamma}_i = \sum_{\qhat_k \leq \qhat_i} \gamma(\qhat_k) / r_k,
    \end{aligned}
  \end{equation*}
  where $r_k = \sum_{\qhat_j \geq \qhat_k}1$ is the ranking of $\qhat_k$ among $\qhat$.
\end{restatable}

\begin{figure}[t]
  \centering
  \vspace{-2mm}
  \includegraphics[scale=0.84]{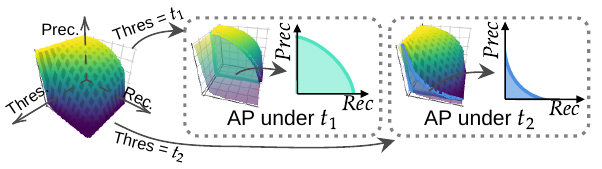}
  \vspace{-2mm}
  \caption{An illustration of the CoTAP loss. The AP losses under all thresholds for target discretization are jointly considered.}
  \label{fig:overview_cotap}
\end{figure}

Here $\psi_i$ is a ranking-based weight, while $\tilde{\gamma}_i$ depends on the choice of $\gamma$ and the ranking of $\qhat$. In fact, for any (non-strict) monotonically increasing $\tilde{\gamma}$, there exists a corresponding $\gamma$. Therefore, we can choose a proper $\tilde{\gamma}$ instead of $\gamma$. Practically, we define $\tilde{\gamma}_i = [\qhat_i - \tau_1]_+$, where $[x]_+ = \max(x,0)$ and $\tau_1$ is a tunable hyperparameter. To make the loss differentiable, following \cite{wen2024algorithm,liu2024not} we utilize the one-side Huber loss $\ell$ as a surrogate loss for $\one[\phat_i \leq \phat_j]$:
\begin{equation}
\label{eq:sur_loss}
  \ell(\phat_i - \phat_j) = \left\{
    \begin{array}{cc}
    \begin{aligned}
        & - 2(\phat_i - \phat_j)/\tau_2 + 1, & \phat_i < \phat_j, \\
        &[1 - (\phat_i \leq \phat_j) / \tau_2]_+^2, & \phat_i \geq \phat_j,
    \end{aligned} 
    \end{array}
  \right.
\end{equation}
where $\tau_2 > 0$ is a tunable hyperparameter. To sum up, the efficient and differential CoTAP loss is formulated as:
\begin{equation}
  \begin{aligned}
    \ell_{CoTAP}(\bm{u}, \bm{v}) \leq \frac{1}{N_{pair}}\sum_{i=1}^{N_{pair}} &\tilde{\gamma}_i \cdot g\Big(
      \psi_i \sum_{\qhat_j < \qhat_i} \ell(\phat_i - \phat_j)]
    \Big), \\
    \psi_i = 1 / \sum_{\qhat_j \geq \qhat_i} \one[\phat_i \leq \phat_j],
    &~~\tilde{\gamma}_i = [\qhat_i - \tau_1]_+,
  \end{aligned}
\end{equation}
where $g(x) = x / (1 + x)$. Intuitively, the above loss achieves its minimal when $\phat_i > \phat_j + \tau_2$ for all $\qhat_i > \qhat_j$, \ie, the correspondence ranking of the online branch is consistent with the target branch, and $\tau_2$ servers as a margin. $\psi_i,\tilde{\gamma}_i$ are adaptive weights to balance the noisy and informative samples. On top of this, we apply the CoTAP loss as the semantic concentration term for patches:
\begin{equation}
  \mathcal{L}_{sc}(\bm{u},\bm{v}) = \ell_{CoTAP}(\bm{u}, \bm{v}),
\end{equation}
and the total loss is a linear combination of the above losses:
\begin{equation}
  \begin{aligned}
    \mathcal{L}_{SSL} = \frac{1}{N} \sum_{\bm{x}\in \mathcal{D}} \Big(&\lambda_1 \mathcal{L}_{align}(\bm{x}) + \bar{\lambda}_1 \bar{\mathcal{L}}_{align}(\bm{x}) \\
    &~~~~~~~~+ \lambda_2 \mathcal{L}_{sc}(\bm{x}, \bm{x}')
    + \bar{\lambda}_2 \bar{\mathcal{L}}_{sc}(\bm{x})\Big),
  \end{aligned}
\end{equation}
where $\{\lambda_{1},\bar{\lambda}_1,\lambda_2,\bar{\lambda}_2\}$ are tunable hyperparameters, $\bm{x}'$ are images randomly sampled from $\mathcal{D}$.

\subsection{Object-aware Filtering via Cross-attention}
\label{subsec:method:obj_filter}
In recent years, the rapid development of attention mechanism \cite{vaswani2017attention,khan2022transformers,chen2021crossvit} significantly enhances the robustness of vision encoders. However, for views randomly augmented from scene-centric images, the self-attention might fail to capture the foreground objects since the dense features are occupied by the background. In this case, it is hard to match the shared local patterns among potential positive samples.

To handle this issue, motivated by the cross-attention module applied in multimodal learning \cite{wei2020multi}, we propose to filter the object-aware features via the cross-attention between image features and representative objects. In this way, each patch is encoded in an object-span space, which reduces the difficulty of capturing shared patterns. Nonetheless, incorporating all possible objects directly would require extensive computational resources. To this end, we propose the \textit{\textbf{O}bject-\textbf{A}ware \textbf{F}ilter (\textbf{OAF})} module to learn object prototypes. An illustration is provided in \Fgref{fig:overview_oaf}.

\textbf{Object prototypes learning.}
Given a set of object-centric images $\bar{\mathcal{D}}\subseteq \mathcal{D}$, we first train a set of prototypes $\bm{U}\in \mathbb{R}^{M \times K_s \times K_s \times D}$ such that its distribution is closed to that of $f_t(\bm{x})$. Specifically, $f_t(\bm{x})$ is downsampled into $K_s \times K_s$, and then flatten $\bm{U}_i$ and $f_t(\bm{x})$ into a $K_s^2 D$ dimensional vector. The result feature matrices are denoted as $\bm{U}^\downarrow \in \mathbb{R}^{M\times(K_s^2D)}$ and $F^{\downarrow} \in \mathbb{R}^{|\bar{\mathcal{D}}|\times(K_s^2D)}$ respectively, where $\|\bm{U}^\downarrow_i\|$ and $\|\bm{F}^\downarrow_i\|$ are normalized to 1. Then we compute the similarities
\begin{equation}
  \bm{\phi} = \sigma(\bm{U}^\downarrow (\bm{F}^\downarrow)^\top / \tau_3),
\end{equation}
where $\sigma$ is the softmax operation along the last dimension, and $\tau_3$ is a tunable hyperparameter to control the temperature. Afterward, the prototypes are updated by minimizing the Shannon entropy:
\begin{equation}
  \mathcal{L}_{proto} = -\frac{1}{|\bar{\mathcal{D}}|M} \sum_{i=1}^M \sum_{j=1}^{|\bar{\mathcal{D}}|} \bm{\phi}_{ij} \log \bm{\phi}_{ij}.
\end{equation}
By minimizing $\mathcal{L}_{proto}$, prototypes get close to their nearest neighbors in $\bm{F}^\downarrow$. Since the images are randomly sampled in each iteration, prototypes are capable of capturing the representative features of foreground objects.

\begin{figure}[t]
  \centering
  \vspace{-2mm}
  \includegraphics[scale=1.1]{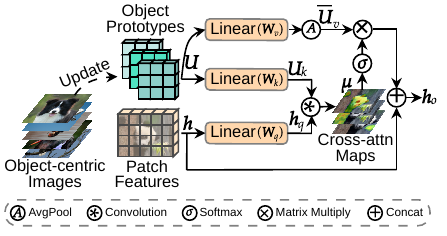}
  \vspace{-2mm}
  \caption{An illustration of the object-aware filter module.}
  \label{fig:overview_oaf}
\end{figure}

\textbf{Prototype filtering.}
Meanwhile, we filter the object-aware information by selecting a linear combination of prototypes to represent the input patches according to the similarities. Given a deep feature map $\bm{h} \in \mathbb{R}^{H\times W \times D}$, we extract the query, key, and value features with a $1\times 1$ convolution:
\begin{equation}
  \begin{aligned}
    \bm{h}_{q} = \bm{W}_{q} * \bm{h}, ~~(\bm{U}_{k})_i = \bm{W}_{k} * \bm{U}_i, ~~(\bm{U}_{v})_i = \bm{W}_{v} * \bm{U}_i,
  \end{aligned}
\end{equation}
where $\bm{W}_{k}, \bm{W}_{v}, \bm{W}_{q} \in \mathbb{R}^{D'\times 1\times 1\times D}$ are the learnable kernels, $*$ refers to the 2D convolution, $D'$ is the output dimension. Afterward, we compute the similarity matrix $\bm{\mu}$ and aggregate the value features:
\begin{equation}
  \bm{\mu} = \texttt{flatten}(\bm{U}_q * \bm{h}_k), ~~ \bar{\bm{U}}_{v} = \texttt{avgpool}(\bm{U}_{v}),
\end{equation}
where $\texttt{flatten}$ means reshaping the $H\times W \times M$ feature map into $(HW)\times M$, and $\texttt{avgpool}$ maps the $M\times K_s\times K_s \times D'$ feature map into $M\times D'$. Finally, we compute the weighted sum and concatenate the result with input features:
\begin{equation}
  \bm{h}_{o} = \bm{h} \oplus \sigma\left(\bm{\mu} / \sqrt{D'}\right) \bar{\bm{U}}_v,
\end{equation}
where $\oplus$ is the concatenation along the feature dimension.

\begin{table*}[t]
  \caption{Quantitative results on COCOStuff-27 (linear segmentation) and ImageNet-1k (linear or $k$-NN classification). The best and the second-best results are highlighted in \textbf{bold} and \underline{underline}, respectively. Best viewed in colors.}
  \centering
  \begin{tabular}{lll|cccccc}
      \toprule
      \multirow{2}{*}{Methods} & \multirow{2}{*}{Arch.} & \multirow{2}{*}{Dataset} & \multicolumn{2}{c}{\textbf{Linear Seg.}} & \multicolumn{2}{c}{\textbf{Linear Cls.}} & \multicolumn{2}{c}{\textbf{$k$-NN Cls.}}\\
      && & mIoU & Acc & Acc@1 & Acc@5 & Acc@1 & Acc@5 \\
      \midrule
      MoCo-V2 \cite{chen2020improved} &  ResNet50  &  IN-1k  & \cellcolor[rgb]{1.000 1.000 1.000}28.2 & \cellcolor[rgb]{1.000 1.000 1.000}61.2 & \cellcolor[rgb]{1.000 1.000 1.000}71.1 &  --  & \cellcolor[rgb]{1.000 1.000 1.000}59.9 & \cellcolor[rgb]{1.000 1.000 1.000}81.1 ~\\
  SwAV \cite{swav} &  ResNet50  &  IN-1k  & \cellcolor[rgb]{1.000 1.000 1.000}39.3 & \cellcolor[rgb]{1.000 1.000 1.000}70.5 & \cellcolor[rgb]{1.000 0.989 0.978}75.3 &  --  & \cellcolor[rgb]{1.000 1.000 1.000}64.6 & \cellcolor[rgb]{1.000 1.000 1.000}83.8 ~\\
  DINO \cite{dino} &  ResNet50  &  IN-1k  & \cellcolor[rgb]{1.000 1.000 1.000}37.1 & \cellcolor[rgb]{1.000 1.000 1.000}67.5 & \cellcolor[rgb]{1.000 0.989 0.978}75.3 & \cellcolor[rgb]{1.000 0.960 0.918}92.6 & \cellcolor[rgb]{1.000 1.000 1.000}64.4 & \cellcolor[rgb]{1.000 1.000 1.000}83.5 ~\\
  SlotCon \cite{wenself} &  ResNet50  &  IN-1k  & \cellcolor[rgb]{1.000 1.000 1.000}39.9 & \cellcolor[rgb]{1.000 1.000 1.000}72.1 & \cellcolor[rgb]{1.000 1.000 1.000}64.8 & \cellcolor[rgb]{1.000 1.000 1.000}87.1 & \cellcolor[rgb]{1.000 1.000 1.000}53.1 & \cellcolor[rgb]{1.000 1.000 1.000}76.4 ~\\
  SlotCon \cite{wenself} &  ResNet50  &  CC  & \cellcolor[rgb]{1.000 1.000 1.000}41.3 & \cellcolor[rgb]{1.000 1.000 1.000}72.7 & \cellcolor[rgb]{1.000 1.000 1.000}48.1 & \cellcolor[rgb]{1.000 1.000 1.000}74.4 & \cellcolor[rgb]{1.000 1.000 1.000}36.0 & \cellcolor[rgb]{1.000 1.000 1.000}58.4 ~\\
  \midrule
  MoCo-V3 \cite{chen2021empirical} &  ViT-S/16  &  IN-1k  & \cellcolor[rgb]{1.000 1.000 1.000}39.0 & \cellcolor[rgb]{1.000 1.000 1.000}72.3 & \cellcolor[rgb]{1.000 1.000 1.000}73.2 & \cellcolor[rgb]{1.000 1.000 1.000}91.2 & \cellcolor[rgb]{0.993 0.830 0.808}74.6 & \cellcolor[rgb]{0.990 0.736 0.703}90.0 ~\\
  SwAV \cite{swav} &  ViT-S/16  &  IN-1k  &  --  &  --  & \cellcolor[rgb]{1.000 1.000 1.000}73.5 &  --  & \cellcolor[rgb]{1.000 1.000 1.000}66.3 &  --  ~\\
  STEGO \cite{hamiltonunsupervised} &  ViT-S/8  &  IN-1k + CC  & \cellcolor[rgb]{1.000 1.000 1.000}36.3 & \cellcolor[rgb]{1.000 1.000 1.000}72.1 &  --  &  --  &  --  &  --  ~\\
  SERE \cite{li2023sere} &  ViT-S/16  &  IN-1k  & \cellcolor[rgb]{0.967 0.986 0.984}44.8 & \cellcolor[rgb]{0.891 0.954 0.948}76.4 & \cellcolor[rgb]{1.000 0.996 0.993}75.1 & \cellcolor[rgb]{1.000 0.987 0.973}92.2 & \cellcolor[rgb]{0.999 0.986 0.984}71.3 & \cellcolor[rgb]{1.000 0.991 0.990}87.1 ~\\
  DINO \cite{dino} &  ViT-S/16  &  IN-1k  & \cellcolor[rgb]{1.000 1.000 1.000}38.8 & \cellcolor[rgb]{1.000 1.000 1.000}71.7 & \cellcolor[rgb]{1.000 0.928 0.853}77.0 & \cellcolor[rgb]{1.000 0.933 0.863}93.0 & \cellcolor[rgb]{0.994 0.834 0.814}74.5 & \cellcolor[rgb]{0.990 0.736 0.703}90.0 ~\\
  Leopart \cite{leopart} &  ViT-S/16  &  IN-1k + CC  & \cellcolor[rgb]{0.867 0.944 0.937}47.2 & \cellcolor[rgb]{0.837 0.931 0.922}77.1 & \cellcolor[rgb]{1.000 1.000 1.000}70.0 & \cellcolor[rgb]{1.000 1.000 1.000}88.9 & \cellcolor[rgb]{1.000 1.000 1.000}55.0 & \cellcolor[rgb]{1.000 1.000 1.000}76.6 ~\\
  iBOT \cite{zhou2021ibot} &  ViT-S/16  &  IN-1k  & \cellcolor[rgb]{0.925 0.968 0.964}45.8 & \cellcolor[rgb]{0.883 0.950 0.944}76.5 & \cellcolor[rgb]{1.000 0.896 0.787}77.9 &  --  & \cellcolor[rgb]{0.992 0.801 0.776}75.2 & \cellcolor[rgb]{0.993 0.815 0.792}89.1 ~\\
  Mugs \cite{zhou2022mugs} &  ViT-S/16  &  IN-1k  & \cellcolor[rgb]{0.855 0.938 0.931}\underline{47.5} & \cellcolor[rgb]{0.821 0.924 0.915}\underline{77.3} & \cellcolor[rgb]{1.000 0.860 0.714}\underline{78.9} & \cellcolor[rgb]{1.000 0.866 0.725}\underline{94.0} & \cellcolor[rgb]{0.991 0.782 0.755}\underline{75.6} & \cellcolor[rgb]{0.988 0.683 0.643}\underline{90.6} ~\\
  \textbf{DINO + Ours} &  ViT-S/16  &  IN-1k + CC  & \cellcolor[rgb]{0.731 0.886 0.871}50.5 & \cellcolor[rgb]{0.697 0.871 0.855}78.9 & \cellcolor[rgb]{1.000 0.853 0.699}79.1 & \cellcolor[rgb]{1.000 0.846 0.684}94.3 & \cellcolor[rgb]{0.987 0.669 0.627}78.0 & \cellcolor[rgb]{0.991 0.771 0.742}89.6 ~\\
  \textbf{Leopart + Ours } &  ViT-S/16  &  IN-1k + CC  & \cellcolor[rgb]{0.735 0.887 0.873}50.4 & \cellcolor[rgb]{0.705 0.875 0.859}78.8 & \cellcolor[rgb]{1.000 0.871 0.736}78.6 & \cellcolor[rgb]{1.000 0.866 0.725}94.0 & \cellcolor[rgb]{0.988 0.702 0.664}77.3 & \cellcolor[rgb]{0.992 0.789 0.762}89.4 ~\\
  \textbf{iBOT + Ours} &  ViT-S/16  &  IN-1k + CC  & \cellcolor[rgb]{0.685 0.866 0.850}\textbf{51.6} & \cellcolor[rgb]{0.673 0.861 0.844}\textbf{79.2} & \cellcolor[rgb]{1.000 0.842 0.677}79.4 & \cellcolor[rgb]{1.000 0.832 0.657}94.5 & \cellcolor[rgb]{0.986 0.650 0.606}78.4 & \cellcolor[rgb]{0.985 0.613 0.564}91.4 ~\\
  \textbf{Mugs + Ours} &  ViT-S/16  &  IN-1k + CC  & \cellcolor[rgb]{0.689 0.868 0.852}51.5 & \cellcolor[rgb]{0.689 0.868 0.851}79.0 & \cellcolor[rgb]{1.000 0.813 0.618}\textbf{80.2} & \cellcolor[rgb]{1.000 0.805 0.602}\textbf{94.9} & \cellcolor[rgb]{0.985 0.626 0.579}\textbf{78.9} & \cellcolor[rgb]{0.983 0.577 0.524}\textbf{91.8} ~\\
  \midrule
  MoCo-V3 \cite{chen2021empirical} &  ViT-B/16  &  IN-1k  &  --  &  --  & \cellcolor[rgb]{1.000 0.939 0.875}76.7 & \cellcolor[rgb]{1.000 0.913 0.822}93.3 &  --  &  --  ~\\
  SERE \cite{li2023sere} &  ViT-B/16  &  IN-1k  & \cellcolor[rgb]{0.781 0.907 0.895}49.3 & \cellcolor[rgb]{0.743 0.891 0.877}78.3 & \cellcolor[rgb]{1.000 0.892 0.780}78.0 & \cellcolor[rgb]{1.000 0.873 0.739}93.9 & \cellcolor[rgb]{0.994 0.844 0.824}74.3 & \cellcolor[rgb]{0.993 0.815 0.792}89.1 ~\\
  Leopart \cite{leopart} &  ViT-B/8  &  IN-1k + CC  & \cellcolor[rgb]{0.752 0.894 0.881}\underline{50.0} & \cellcolor[rgb]{0.712 0.878 0.863}\underline{78.7} &  --  &  --  & \cellcolor[rgb]{1.000 1.000 1.000}65.3 & \cellcolor[rgb]{1.000 1.000 1.000}84.1 ~\\
  DINO \cite{dino} &  ViT-B/16  &  IN-1k  & \cellcolor[rgb]{1.000 1.000 1.000}44.0 & \cellcolor[rgb]{0.977 0.990 0.989}75.3 & \cellcolor[rgb]{1.000 0.885 0.765}78.2 & \cellcolor[rgb]{1.000 0.873 0.739}93.9 & \cellcolor[rgb]{0.991 0.758 0.728}76.1 & \cellcolor[rgb]{0.986 0.639 0.594}91.1 ~\\
  iBOT \cite{zhou2021ibot} &  ViT-B/16  &  IN-1k  & \cellcolor[rgb]{0.768 0.901 0.889}49.6 & \cellcolor[rgb]{0.743 0.891 0.877}78.3 & \cellcolor[rgb]{1.000 0.838 0.669}79.5 &  --  & \cellcolor[rgb]{0.989 0.711 0.675}77.1 & \cellcolor[rgb]{0.984 0.586 0.534}91.7 ~\\
  Mugs \cite{zhou2022mugs} &  ViT-B/16  &  IN-1k  & \cellcolor[rgb]{0.768 0.901 0.889}49.6 & \cellcolor[rgb]{0.790 0.911 0.900}77.7 & \cellcolor[rgb]{1.000 0.799 0.589}\underline{80.6} & \cellcolor[rgb]{1.000 0.812 0.616}94.8 & \cellcolor[rgb]{0.987 0.669 0.627}\underline{78.0} & \cellcolor[rgb]{0.990 0.753 0.723}89.8 ~\\
  DINO \cite{dino} &  ViT-B/8  &  IN-1k  & \cellcolor[rgb]{0.921 0.967 0.962}45.9 & \cellcolor[rgb]{0.883 0.950 0.944}76.5 & \cellcolor[rgb]{1.000 0.817 0.625}80.1 & \cellcolor[rgb]{1.000 0.799 0.588}\underline{95.0} & \cellcolor[rgb]{0.988 0.697 0.659}77.4 & \cellcolor[rgb]{0.983 0.560 0.505}\underline{92.0} ~\\
  \textbf{DINO + Ours} &  ViT-B/16  &  IN-1k + CC  & \cellcolor[rgb]{0.660 0.856 0.838}52.2 & \cellcolor[rgb]{0.658 0.855 0.837}79.4 & \cellcolor[rgb]{1.000 0.813 0.618}80.2 & \cellcolor[rgb]{1.000 0.805 0.602}94.9 & \cellcolor[rgb]{0.985 0.612 0.563}79.2 & \cellcolor[rgb]{0.984 0.595 0.544}91.6 ~\\
  \textbf{iBOT + Ours} &  ViT-B/16  &  IN-1k + CC  & \cellcolor[rgb]{0.590 0.826 0.804}53.9 & \cellcolor[rgb]{0.611 0.835 0.814}80.0 & \cellcolor[rgb]{1.000 0.792 0.574}80.8 & \cellcolor[rgb]{1.000 0.792 0.575}95.1 & \cellcolor[rgb]{0.983 0.569 0.515}80.1 & \cellcolor[rgb]{0.988 0.701 0.663}90.4 ~\\
  \textbf{Mugs + Ours} &  ViT-B/16  &  IN-1k + CC  & \cellcolor[rgb]{0.665 0.858 0.840}52.1 & \cellcolor[rgb]{0.673 0.861 0.844}79.2 & \cellcolor[rgb]{1.000 0.759 0.508}81.7 & \cellcolor[rgb]{1.000 0.765 0.520}95.5 & \cellcolor[rgb]{0.982 0.536 0.478}80.8 & \cellcolor[rgb]{0.982 0.542 0.485}92.2 ~\\
  \textbf{DINO + Ours} &  ViT-B/8  &  IN-1k + CC  & \cellcolor[rgb]{0.557 0.812 0.788}\textbf{54.7} & \cellcolor[rgb]{0.557 0.812 0.788}\textbf{80.7} & \cellcolor[rgb]{1.000 0.745 0.478}\textbf{82.1} & \cellcolor[rgb]{1.000 0.745 0.478}\textbf{95.8} & \cellcolor[rgb]{0.980 0.498 0.435}\textbf{81.6} & \cellcolor[rgb]{0.980 0.498 0.435}\textbf{92.7} ~\\
      \bottomrule
  \end{tabular}
  \label{tab:main}
  \end{table*}
  
  \section{Experiments}
  \label{sec:experiment}
  
  To validate the soundness of the proposed method, we conduct empirical studies to verify the soundness of our analysis and the effectiveness of the proposed framework. In \cref{subsec:details}, we present implementation details of the model pretraining and provide quantitative results on ImageNet and MS COCO in \cref{subsec:main_results}. Due to the limitation of space, details of downstream tasks are attached in  Appendix {\color{blue}B}. Afterward, in \cref{subsec:transfer} we transfer our pretrained model to other datasets of dense downstream tasks including semantic segmentation, object detection, instance segmentation, and video object segmentation. Ablation studies and visualizations are provided in \cref{subsec:ablations} and \cref{subsec:vis}, respectively.
  
  \subsection{Implementation Details}
  \label{subsec:details}
  
  \textbf{Datasets.}
  Following the commonly used setting, all models in this work are pretrained with ImageNet-1k (IN-1k) \cite{russakovsky2015imagenet}. Similar to STEGO \cite{hamiltonunsupervised} and Leopart \cite{leopart}, we fine-tune the pretrained model with our proposed framework on ImageNet-1k and MS COCO (CC) \cite{lin2014microsoft}. Since the scene-centric images in MS COCO violate the basic assumption of image-level losses, we only apply image-level losses to ImageNet-1k.
  
  \textbf{Network architecture.}
  We use ViT-Small with a patch size of 16 by default. We also provide results of ViT-Base with a patch size of 16 or 8. Following \cite{leopart}, we add a 3-layer multilayer perceptron to adjust the feature dimension to 256, where the hidden layers have 2048 dimensions. Accordingly, the prototypes are set to a $64\times 384 \times 3 \times 3$ tensor for ViT-Small and $64\times 768 \times 3 \times 3$ for ViT-Base. Other learnable layers in prototype filtering are set to the same as the self-attention layer used in ViT-Small/ViT-Base \cite{dosovitskiyimage};
  
  \textbf{Optimization strategy.} To avoid the high computational cost of training a ViT model from scratch, we first load the official checkpoints of existing SSL models as initialization. Here the proposed OAF module is initialized from scratch. Afterward, all parameters are fine-tuned with our proposed loss (as summarized in Eq. (14)). At the training phase, each image is randomly cropped into 2 global views ($224\times 224$) and 4 local views ($96\times 96$). Afterward, we apply random standard data augmentations including color jittering, gray scaling, and Gaussian blur. The online branch is optimized with the AdamW \cite{loshchilovdecoupled} optimizer. The learning rate is initialized as $3\times 10^{-5}$ and decays to $10^{-6}$ with a cosine decay schedule. The weight decay is set to $0.024$ and increases to $0.24$ with a cosine schedule. Following \cite{dino}, the target branch is updated by the exponential moving average of the online branch, where the update rate is initialized to $0.9997$ and increases to $1$ with a cosine schedule. The hyperparameters of the loss function are set as follows: $\lambda_1 =$ $\lambda_2 =$ $\bar{\lambda}_1 = $ $\bar{\lambda}_2 = \lambda_3 = 1$, $\tau_1 = -0.2$, $\tau_2 = 0.5$, $\tau_3 = 0.1$. By default, we set the batch size to $256$ and the number of training iterations to $100k$. It takes about $5$ hours to train a ViT-S/16 model on 8 NVIDIA RTX 4090 GPUs.
  
  
  \textbf{Competitors.}
  We compare the proposed methods with 2 types of competitors: \textbf{1) Image-level SSL methods}, including MoCo V2 \cite{chen2020improved}, MoCo V3 \cite{chen2021empirical}, SwAV \cite{swav}, DINO \cite{dino}, and SERE \cite{li2023sere}. These methods aim to learn general representations of images, but some like DINO are also capable of distinguishing dense representations. \textbf{2) Pixel/patch-level SSL methods}, including SlotCon \cite{wenself}, STEGO \cite{hamiltonunsupervised}, CrOC \cite{stegmuller2023croc} Leopart \cite{leopart}, iBOT \cite{zhou2021ibot}, and Mugs \cite{zhou2022mugs}.
  
  
  \begin{table*}
  \caption{Transfer learning results (mIoU) on semantic segmentation with fixed representations.}
  \centering
  \begin{tabular}{ll|cccccccc}
      \toprule
      \multirow{2}{*}{Methods} & \multirow{2}{*}{Dataset} & \multicolumn{2}{c}{\textbf{COCOStuff-27}} & \multicolumn{2}{c}{\textbf{PASCAL VOC}}  & \multicolumn{2}{c}{\textbf{ADE20k}} & \multicolumn{2}{c}{\textbf{Cityscapes}} \\
      & & Linear & FCN & Linear & FCN & Linear & FCN & Linear & FCN \\
      \midrule
      DINO \cite{dino} & IN-1k & \cellcolor[rgb]{1.000 1.000 1.000}38.8 & \cellcolor[rgb]{1.000 1.000 1.000}49.5 & \cellcolor[rgb]{1.000 1.000 1.000}50.7 & \cellcolor[rgb]{1.000 1.000 1.000}66.8 & \cellcolor[rgb]{1.000 1.000 1.000}22.3 & \cellcolor[rgb]{1.000 1.000 1.000}28.9 & \cellcolor[rgb]{1.000 1.000 1.000}42.9 & \cellcolor[rgb]{1.000 1.000 1.000}54.5 ~\\
      iBOT \cite{zhou2021ibot} & IN-1k & \cellcolor[rgb]{0.758 0.897 0.884}45.8 & \cellcolor[rgb]{0.803 0.917 0.906}52.4 & \cellcolor[rgb]{1.000 0.835 0.662}66.0 & \cellcolor[rgb]{1.000 0.871 0.735}72.7 & \cellcolor[rgb]{0.991 0.760 0.730}27.2 & \cellcolor[rgb]{0.989 0.712 0.676}32.0 & \cellcolor[rgb]{1.000 1.000 1.000}42.9 & \cellcolor[rgb]{1.000 1.000 1.000}54.4 ~\\
      Leopart \cite{leopart} & IN-1k + CC & \cellcolor[rgb]{0.723 0.882 0.868}47.2 & \cellcolor[rgb]{0.838 0.931 0.922}51.9 & \cellcolor[rgb]{1.000 0.860 0.713}63.7 & \cellcolor[rgb]{1.000 0.954 0.906}68.9 & \cellcolor[rgb]{0.993 0.819 0.797}26.0 & \cellcolor[rgb]{0.997 0.936 0.928}29.6 & \cellcolor[rgb]{0.699 0.810 0.891}\underline{46.3} & \cellcolor[rgb]{0.896 0.934 0.963}54.8 ~\\
      Mugs \cite{zhou2022mugs} & IN-1k & \cellcolor[rgb]{0.699 0.872 0.856}\underline{47.5} & \cellcolor[rgb]{0.694 0.870 0.854}\underline{54.0} & \cellcolor[rgb]{1.000 0.778 0.545}\underline{71.3} & \cellcolor[rgb]{1.000 0.791 0.573}\underline{76.3} & \cellcolor[rgb]{0.984 0.582 0.530}\underline{30.8} & \cellcolor[rgb]{0.981 0.507 0.446}\underline{34.2} & \cellcolor[rgb]{0.726 0.827 0.901}46.0 & \cellcolor[rgb]{0.510 0.690 0.824}\textbf{\underline{55.9}} ~\\
      SERE \cite{li2023sere} & IN-1k & \cellcolor[rgb]{0.754 0.896 0.883}45.9 & \cellcolor[rgb]{0.872 0.946 0.939}51.4 & \cellcolor[rgb]{1.000 0.808 0.607}68.5 & \cellcolor[rgb]{1.000 0.868 0.731}72.8 & \cellcolor[rgb]{0.989 0.716 0.680}28.1 & \cellcolor[rgb]{0.989 0.722 0.687}31.9 & \cellcolor[rgb]{0.860 0.912 0.950}44.5 & \cellcolor[rgb]{0.931 0.957 0.975}54.7 ~\\
      CrOC \cite{stegmuller2023croc} & IN-1k & \cellcolor[rgb]{0.823 0.925 0.916}43.9 & \cellcolor[rgb]{0.961 0.983 0.981}50.1 & \cellcolor[rgb]{1.000 0.809 0.609}68.4 & \cellcolor[rgb]{1.000 0.882 0.758}72.2 & \cellcolor[rgb]{0.992 0.790 0.763}26.6 & \cellcolor[rgb]{0.995 0.861 0.844}30.4 & \cellcolor[rgb]{0.932 0.957 0.976}43.7 & \cellcolor[rgb]{1.000 1.000 1.000}54.2 ~\\
      \textbf{DINO + Ours} & IN-1k + CC & \cellcolor[rgb]{0.595 0.828 0.806}50.5 & \cellcolor[rgb]{0.632 0.844 0.824}54.9 & \cellcolor[rgb]{1.000 0.775 0.540}71.5 & \cellcolor[rgb]{1.000 0.794 0.577}76.2 & \cellcolor[rgb]{0.985 0.612 0.563}30.2 & \cellcolor[rgb]{0.982 0.535 0.477}33.9 & \cellcolor[rgb]{0.681 0.798 0.885}46.5 & \cellcolor[rgb]{0.615 0.757 0.861}55.6 ~\\
      \textbf{iBOT + Ours} & IN-1k + CC & \cellcolor[rgb]{0.557 0.812 0.788}\textbf{51.6} & \cellcolor[rgb]{0.564 0.815 0.792}55.9 & \cellcolor[rgb]{1.000 0.757 0.503}73.2 & \cellcolor[rgb]{1.000 0.772 0.532}77.2 & \cellcolor[rgb]{0.983 0.562 0.508}31.2 & \cellcolor[rgb]{0.980 0.498 0.435}34.3 & \cellcolor[rgb]{0.654 0.781 0.875}46.8 & \cellcolor[rgb]{0.791 0.868 0.925}55.1 ~\\
      \textbf{Mugs + Ours} & IN-1k + CC & \cellcolor[rgb]{0.560 0.813 0.790}51.5 & \cellcolor[rgb]{0.557 0.812 0.788}\textbf{56.0} & \cellcolor[rgb]{1.000 0.745 0.478}\textbf{74.3} & \cellcolor[rgb]{1.000 0.745 0.478}\textbf{78.4} & \cellcolor[rgb]{0.980 0.498 0.435}\textbf{32.5} & \cellcolor[rgb]{0.980 0.498 0.435}\textbf{34.3} & \cellcolor[rgb]{0.510 0.690 0.824}\textbf{48.4} & \cellcolor[rgb]{0.615 0.757 0.861}55.6 ~\\ 
      \bottomrule
  \end{tabular}
  \label{tab:semantic_seg}
  \end{table*}
  
  \begin{table*}[t]  
    \begin{minipage}[t]{0.55\textwidth}
    \caption{Full fine-tuning on semantic segmentation (mIoU).}
    \centering
      \setlength\tabcolsep{4.6pt}
      \begin{tabular}{l|cccccc}
        \toprule
        \multirow{2}{*}{Methods} & \multicolumn{2}{c}{\textbf{ADE20k}} & \multicolumn{2}{c}{\textbf{Cityscapes}} & \multicolumn{2}{c}{\textbf{COCOStuff-27}} \\
        & UperNet & FPN & UperNet & FPN & UperNet & FPN \\
        \midrule
        DINO \cite{dino}
        & \cellcolor[rgb]{1.000 1.000 1.000}43.8 & \cellcolor[rgb]{0.979 0.991 0.990}43.7 & \cellcolor[rgb]{1.000 0.973 0.945}75.8 & \cellcolor[rgb]{1.000 1.000 1.000}73.9 & \cellcolor[rgb]{1.000 1.000 1.000}58.1 & \cellcolor[rgb]{1.000 1.000 1.000}56.8 ~\\
       iBOT \cite{zhou2021ibot}
        & \cellcolor[rgb]{0.842 0.933 0.925}45.4 & \cellcolor[rgb]{0.873 0.946 0.939}44.7 & \cellcolor[rgb]{1.000 1.000 1.000}75.4 & \cellcolor[rgb]{1.000 0.919 0.834}75.4 & \cellcolor[rgb]{0.995 0.867 0.851}59.0 & \cellcolor[rgb]{0.997 0.916 0.906}57.2 ~\\
       Leopart \cite{leopart} & \cellcolor[rgb]{1.000 1.000 1.000}43.3 & \cellcolor[rgb]{1.000 1.000 1.000}43.5 & \cellcolor[rgb]{1.000 1.000 1.000}75.4 & \cellcolor[rgb]{1.000 0.902 0.800}75.7 & \cellcolor[rgb]{0.998 0.956 0.950}58.4 & \cellcolor[rgb]{0.998 0.958 0.953}57.0 ~\\
       Mugs \cite{zhou2022mugs}
        & \cellcolor[rgb]{0.724 0.883 0.868}\underline{46.6} & \cellcolor[rgb]{0.768 0.901 0.889}\underline{45.7} & \cellcolor[rgb]{1.000 0.826 0.643}\underline{78.0} & \cellcolor[rgb]{1.000 0.810 0.612}\underline{77.4} & \cellcolor[rgb]{0.987 0.660 0.618}\underline{60.4} & \cellcolor[rgb]{0.984 0.603 0.553}\underline{58.7} ~\\
       SERE \cite{li2023sere}
        & \cellcolor[rgb]{0.911 0.962 0.958}44.7 & \cellcolor[rgb]{0.926 0.969 0.965}44.2 & \cellcolor[rgb]{1.000 0.940 0.876}76.3 & \cellcolor[rgb]{1.000 0.881 0.756}76.1 & \cellcolor[rgb]{0.989 0.719 0.684}60.0 & \cellcolor[rgb]{0.987 0.665 0.624}58.4 ~\\
       CrOC \cite{stegmuller2023croc}
        & \cellcolor[rgb]{1.000 1.000 1.000}42.9 & \cellcolor[rgb]{1.000 1.000 1.000}42.4 & \cellcolor[rgb]{1.000 1.000 1.000}74.4 & \cellcolor[rgb]{1.000 0.897 0.789}75.8 & \cellcolor[rgb]{0.998 0.941 0.934}58.5 & \cellcolor[rgb]{0.999 0.979 0.976}56.9 ~\\
       \textbf{DINO + Ours
       } & \cellcolor[rgb]{0.872 0.946 0.939}45.1 & \cellcolor[rgb]{0.873 0.946 0.939}44.7 & \cellcolor[rgb]{1.000 0.873 0.739}77.3 & \cellcolor[rgb]{1.000 0.821 0.634}77.2 & \cellcolor[rgb]{0.994 0.838 0.817}59.2 & \cellcolor[rgb]{1.000 1.000 1.000}56.6 ~\\
       \textbf{iBOT + Ours
       } & \cellcolor[rgb]{0.724 0.883 0.868}46.6 & \cellcolor[rgb]{0.652 0.852 0.834}46.8 & \cellcolor[rgb]{1.000 0.745 0.478}\textbf{79.2} & \cellcolor[rgb]{1.000 0.745 0.478}\textbf{78.6} & \cellcolor[rgb]{0.985 0.616 0.568}60.7 & \cellcolor[rgb]{0.989 0.728 0.694}58.1 ~\\
       \textbf{Leopart + Ours} & \cellcolor[rgb]{0.862 0.941 0.934}45.2 & \cellcolor[rgb]{0.821 0.924 0.914}45.2 & \cellcolor[rgb]{1.000 0.993 0.986}75.5 & \cellcolor[rgb]{1.000 0.788 0.567}77.8 & \cellcolor[rgb]{0.987 0.675 0.635}60.3 & \cellcolor[rgb]{0.993 0.833 0.812}57.6 ~\\
       \textbf{Mugs + Ours} & \cellcolor[rgb]{0.557 0.812 0.788}\textbf{48.3} & \cellcolor[rgb]{0.557 0.812 0.788}\textbf{47.7} & \cellcolor[rgb]{1.000 0.759 0.506}79.0 & \cellcolor[rgb]{1.000 0.761 0.512}78.3 & \cellcolor[rgb]{0.980 0.498 0.435}\textbf{61.5} & \cellcolor[rgb]{0.980 0.498 0.435}\textbf{59.2} ~\\
        \bottomrule
    \end{tabular}
    \label{tab:semantic_seg_ft}
    \end{minipage}
    \hspace{4mm}
      \begin{minipage}[t]{0.45\textwidth}
        \caption{Results on object detection and instance segmentation.}
        \setlength\tabcolsep{3.6pt}
        \centering
          \begin{tabular}{ll|cccccccc}
              \toprule
              \multirow{2}{*}{Methods} & \multirow{2}{*}{Arch.} & \multicolumn{3}{c}{\textbf{Obj. Det.}} & \multicolumn{1}{c}{\textbf{Ins. Seg.}} \\
              & & $\mathrm{AP^b}$ & $\mathrm{AP^b_{50}}$ & $\mathrm{AP^b_{75}}$ & $\mathrm{AP^m}$ \\
              \midrule
              Sup. \cite{liu2021swin} & Swin-T & \cellcolor[rgb]{0.982 0.992 0.991}48.1 & \cellcolor[rgb]{0.986 0.994 0.993}67.1 & \cellcolor[rgb]{0.815 0.922 0.912}52.5 & \cellcolor[rgb]{1.000 0.926 0.848}41.7 ~\\
              MoBY \cite{xie2021self} & Swin-T & \cellcolor[rgb]{0.982 0.992 0.991}48.1 & \cellcolor[rgb]{0.986 0.994 0.993}67.1 & \cellcolor[rgb]{0.865 0.942 0.935}52.1 & \cellcolor[rgb]{1.000 0.947 0.891}41.5 ~\\
              Sup. \cite{liu2021swin} & ViT-S/16 & \cellcolor[rgb]{1.000 1.000 1.000}46.2 & \cellcolor[rgb]{1.000 1.000 1.000}65.9 & \cellcolor[rgb]{1.000 1.000 1.000}49.6 & \cellcolor[rgb]{1.000 1.000 1.000}40.1 ~\\
              iBOT \cite{zhou2021ibot} & ViT-S/16 & \cellcolor[rgb]{0.742 0.890 0.876}49.4 & \cellcolor[rgb]{0.757 0.897 0.884}68.7 & \cellcolor[rgb]{0.717 0.880 0.865}53.3 & \cellcolor[rgb]{1.000 0.830 0.652}42.6 ~\\
              Mugs \cite{zhou2022mugs} & ViT-S/16 & \cellcolor[rgb]{0.668 0.859 0.841}\underline{49.8} & \cellcolor[rgb]{0.728 0.885 0.870}\underline{68.9} & \cellcolor[rgb]{0.692 0.869 0.853}\underline{53.5} & \cellcolor[rgb]{1.000 0.788 0.565}\underline{43.0} ~\\
              \textbf{DINO + Ours} & ViT-S/16 & \cellcolor[rgb]{0.742 0.890 0.876}49.4 & \cellcolor[rgb]{0.771 0.903 0.891}68.6 & \cellcolor[rgb]{0.729 0.885 0.871}53.2 & \cellcolor[rgb]{1.000 0.830 0.652}42.6 ~\\
              \textbf{iBOT + Ours} & ViT-S/16 & \cellcolor[rgb]{0.668 0.859 0.841}49.8 & \cellcolor[rgb]{0.714 0.879 0.863}69.0 & \cellcolor[rgb]{0.643 0.848 0.829}53.9 & \cellcolor[rgb]{1.000 0.798 0.587}42.9 ~\\
              \textbf{Mugs + Ours} & ViT-S/16 & \cellcolor[rgb]{0.557 0.812 0.788}\textbf{50.4} & \cellcolor[rgb]{0.557 0.812 0.788}\textbf{70.1} & \cellcolor[rgb]{0.557 0.812 0.788}\textbf{54.6} & \cellcolor[rgb]{1.000 0.745 0.478}\textbf{43.4} ~\\
              \bottomrule
          \end{tabular}
        \label{tab:obj_det}
      \end{minipage}
    \end{table*}
  
    \begin{table*}[t]
      \begin{minipage}[t]{0.39\textwidth}
        \caption{Quantitative results on video object segmentation.}
        \setlength\tabcolsep{3.6pt}
        \centering
          \begin{tabular}{ll|ccc}
              \toprule
              Methods & Arch. & $\mathcal{J}\&\mathcal{F}$ & $\mathcal{J}_{\mathrm{mean}}$ & $\mathcal{F}_{\mathrm{mean}}$ \\
              \midrule
              DINO \cite{dino} & ViT-S/16 & \cellcolor[rgb]{1.000 1.000 1.000}62.3 & \cellcolor[rgb]{1.000 1.000 1.000}60.7 & \cellcolor[rgb]{1.000 1.000 1.000}63.9 ~\\
              iBOT \cite{zhou2021ibot} & ViT-S/16 & \cellcolor[rgb]{0.996 0.998 0.998}62.4 & \cellcolor[rgb]{0.995 0.998 0.998}60.8 & \cellcolor[rgb]{0.997 0.999 0.998}64.0 ~\\
              Mugs \cite{zhou2022mugs} & ViT-S/16 & \cellcolor[rgb]{0.969 0.987 0.985}63.1 & \cellcolor[rgb]{0.968 0.987 0.985}61.4 & \cellcolor[rgb]{0.966 0.986 0.984}64.9 ~\\
              DINO  \cite{dino} & ViT-B/8 & \cellcolor[rgb]{0.649 0.851 0.832}\underline{71.4} & \cellcolor[rgb]{0.674 0.862 0.844}\underline{67.9} & \cellcolor[rgb]{0.628 0.842 0.822}\underline{74.9} ~\\
              \textbf{DINO + Ours} & ViT-S/16 & \cellcolor[rgb]{0.904 0.959 0.954}64.8 & \cellcolor[rgb]{0.891 0.954 0.948}63.1 & \cellcolor[rgb]{0.915 0.964 0.960}66.4 ~\\
              \textbf{iBOT + Ours} & ViT-S/16 & \cellcolor[rgb]{0.911 0.962 0.958}64.6 & \cellcolor[rgb]{0.901 0.958 0.952}62.9 & \cellcolor[rgb]{0.919 0.966 0.961}66.3 ~\\
              \textbf{Mugs + Ours} & ViT-S/16 & \cellcolor[rgb]{0.915 0.964 0.959}64.5 & \cellcolor[rgb]{0.887 0.952 0.946}63.2 & \cellcolor[rgb]{0.939 0.974 0.971}65.7 ~\\
              \textbf{DINO + Ours} & ViT-B/8 & \cellcolor[rgb]{0.557 0.812 0.788}\textbf{73.8} & \cellcolor[rgb]{0.557 0.812 0.788}\textbf{70.5} & \cellcolor[rgb]{0.557 0.812 0.788}\textbf{77.0} ~\\
              \bottomrule
          \end{tabular}
        \label{tab:video_seg}
      \end{minipage}
      \hspace{4mm}
      \begin{minipage}[t]{0.6\textwidth}
        \caption{Ablation studies over different components of the proposed method on COCOStuff-27 \& ImageNet-1k. The linear classifier is trained with a fast evaluation protocol.}
        \setlength\tabcolsep{4.0pt}
        \centering
        \begin{tabular}{llllll|cccc}
            \toprule
            \multirow{2}{*}{No.} & \multirow{2}{*}{$\bar{\mathcal{L}}_{align}$} & \multirow{2}{*}{$\mathcal{L}_{align}$} & \multirow{2}{*}{$\bar{\mathcal{L}}_{sc}$} & \multirow{2}{*}{$\mathcal{L}_{sc}$} & \multirow{2}{*}{OAF} &  \multicolumn{2}{c}{\textbf{Linear Seg.}} & \textbf{Linear Cls.} & \textbf{$k$-NN Cls.} \\
            &&&&&& mIoU & Acc & Acc@1 & Acc@1 \\
            \midrule
            1 & \cmark &  &  &  &  & \cellcolor[rgb]{1.000 1.000 1.000}40.1 & \cellcolor[rgb]{1.000 1.000 1.000}72.6 & \cellcolor[rgb]{1.000 0.937 0.871}69.9 & \cellcolor[rgb]{0.994 0.839 0.819}74.6 ~\\
            2 &  & \cmark &  &  &  & \cellcolor[rgb]{0.882 0.950 0.944}47.2 & \cellcolor[rgb]{0.832 0.929 0.920}77.1 & \cellcolor[rgb]{1.000 1.000 1.000}55.6 & \cellcolor[rgb]{1.000 1.000 1.000}52.1 ~\\
            3 & \cmark & \cmark &  &  &  & \cellcolor[rgb]{0.941 0.975 0.972}46.6 & \cellcolor[rgb]{0.847 0.935 0.927}77.0 & \cellcolor[rgb]{1.000 0.974 0.946}68.8 & \cellcolor[rgb]{0.997 0.930 0.921}73.7 ~\\
            4 & \cmark &  & \cmark &  &  & \cellcolor[rgb]{1.000 1.000 1.000}40.6 & \cellcolor[rgb]{1.000 1.000 1.000}72.7 & \cellcolor[rgb]{1.000 0.858 0.709}\underline{72.3} & \cellcolor[rgb]{0.982 0.538 0.480}\underline{77.6} ~\\
            5 &  & \cmark &  & \cmark &  & \cellcolor[rgb]{0.606 0.833 0.812}\underline{50.0} & \cellcolor[rgb]{0.587 0.825 0.803}\underline{78.7} & \cellcolor[rgb]{1.000 1.000 1.000}55.3 & \cellcolor[rgb]{1.000 1.000 1.000}52.0 ~\\
            6 & \cmark & \cmark & \cmark & \cmark &  & \cellcolor[rgb]{0.645 0.849 0.831}49.6 & \cellcolor[rgb]{0.618 0.838 0.817}78.5 & \cellcolor[rgb]{1.000 0.884 0.763}71.5 & \cellcolor[rgb]{0.984 0.588 0.537}77.1 ~\\
            7 & \cmark & \cmark & \cmark & \cmark & \cmark & \cellcolor[rgb]{0.557 0.812 0.788}\textbf{50.5} & \cellcolor[rgb]{0.557 0.812 0.788}\textbf{78.9} & \cellcolor[rgb]{1.000 0.745 0.478}\textbf{75.7} & \cellcolor[rgb]{0.980 0.498 0.435}\textbf{78.0} ~\\
            \bottomrule
        \end{tabular}
        \label{tab:ablation}
      \end{minipage}
    \end{table*}

    \begin{figure}[t]
      \centering
      \vspace{-2mm}
      \subfigure[Initialized with DINO]{
        \includegraphics[scale=0.27]{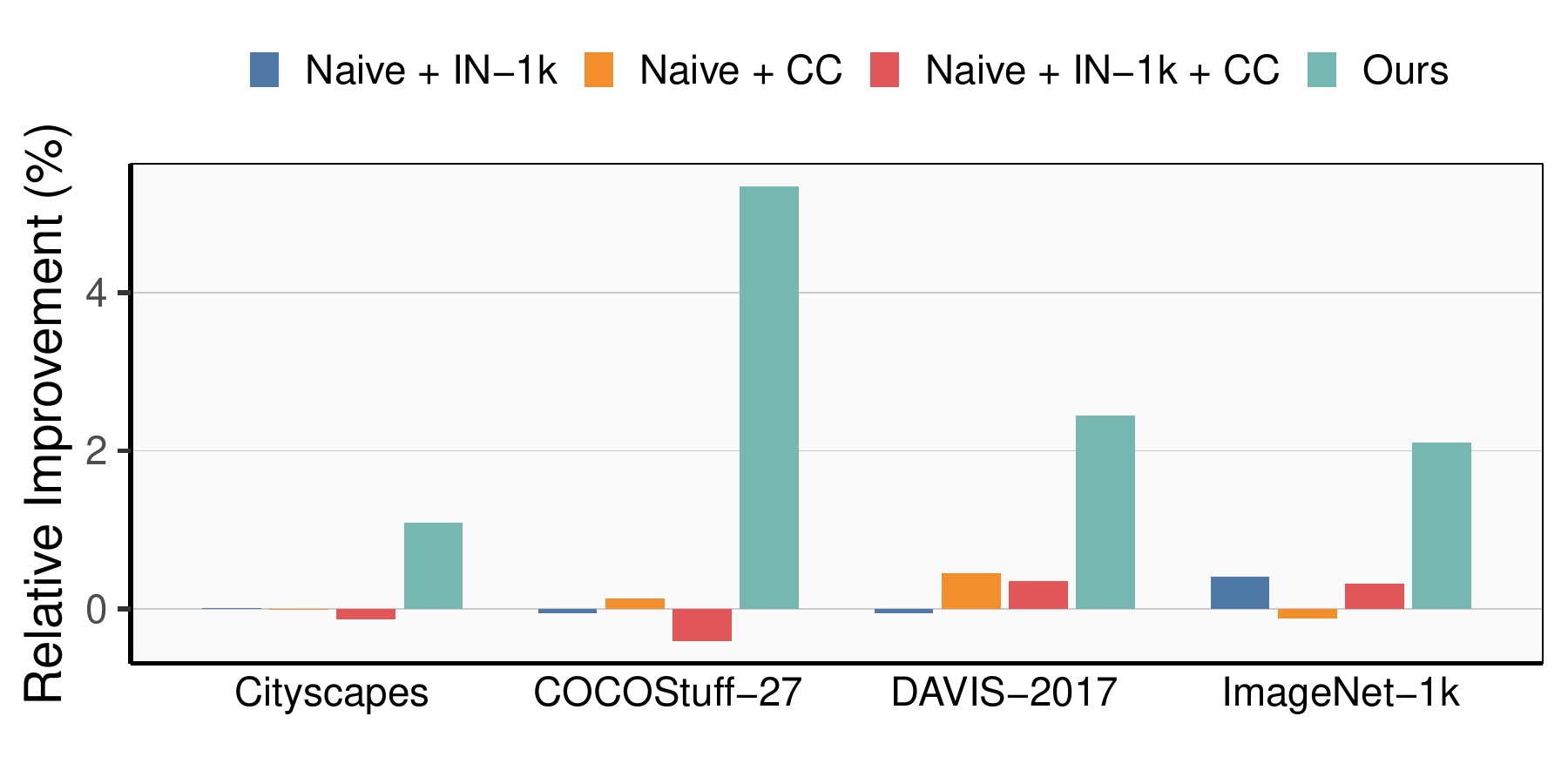}
      }
      \subfigure[Initialized with iBOT]{
        \includegraphics[scale=0.27]{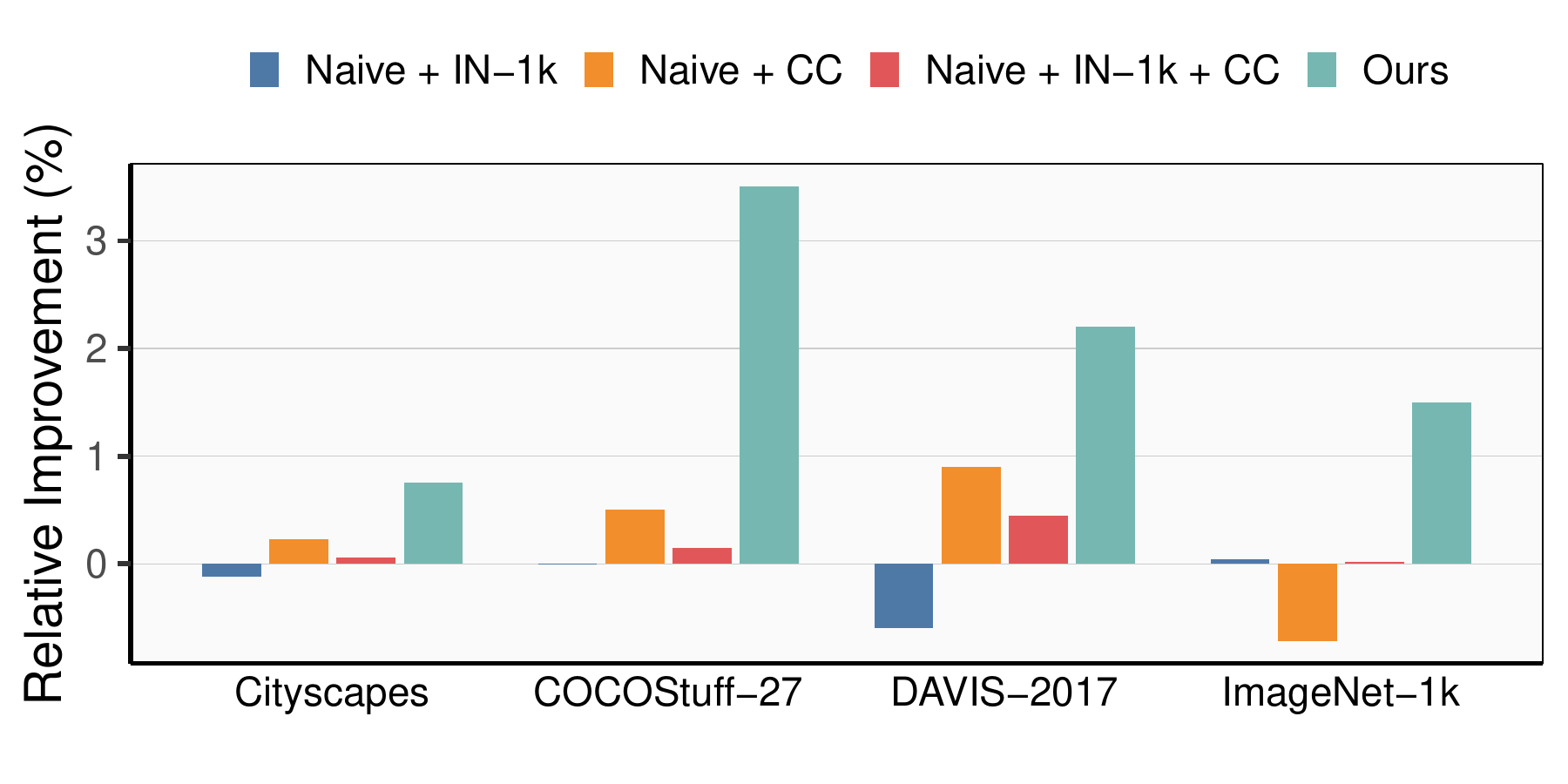}
      }
      \subfigure[Initialized with Mugs]{
        \includegraphics[scale=0.27]{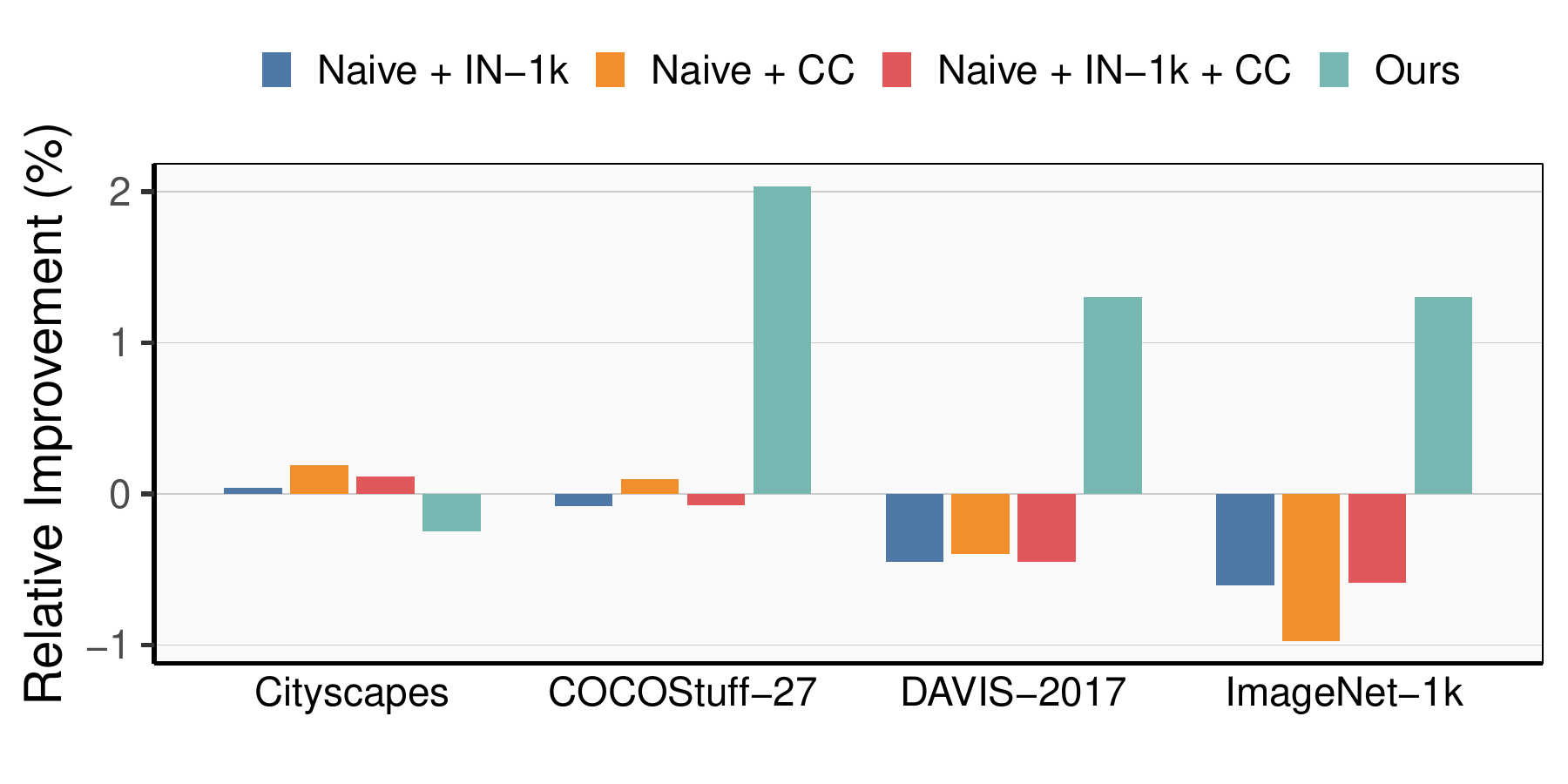}
      }
      \vspace{-2mm}
      \caption{Relative improvement of additional training epochs and the training dataset over different baselines. We report the mIoU improvement for FCN fine-tuning on Cityscapes and COCOStuff-27, Acc@1 for linear probing on ImageNet-1k, and $\mathcal{J}\&\mathcal{F}$ on DAVIS-2017.}
      \label{fig:longer_training}
    \end{figure}

    \begin{figure*}[h]
      \centering
      \includegraphics[scale=0.26]{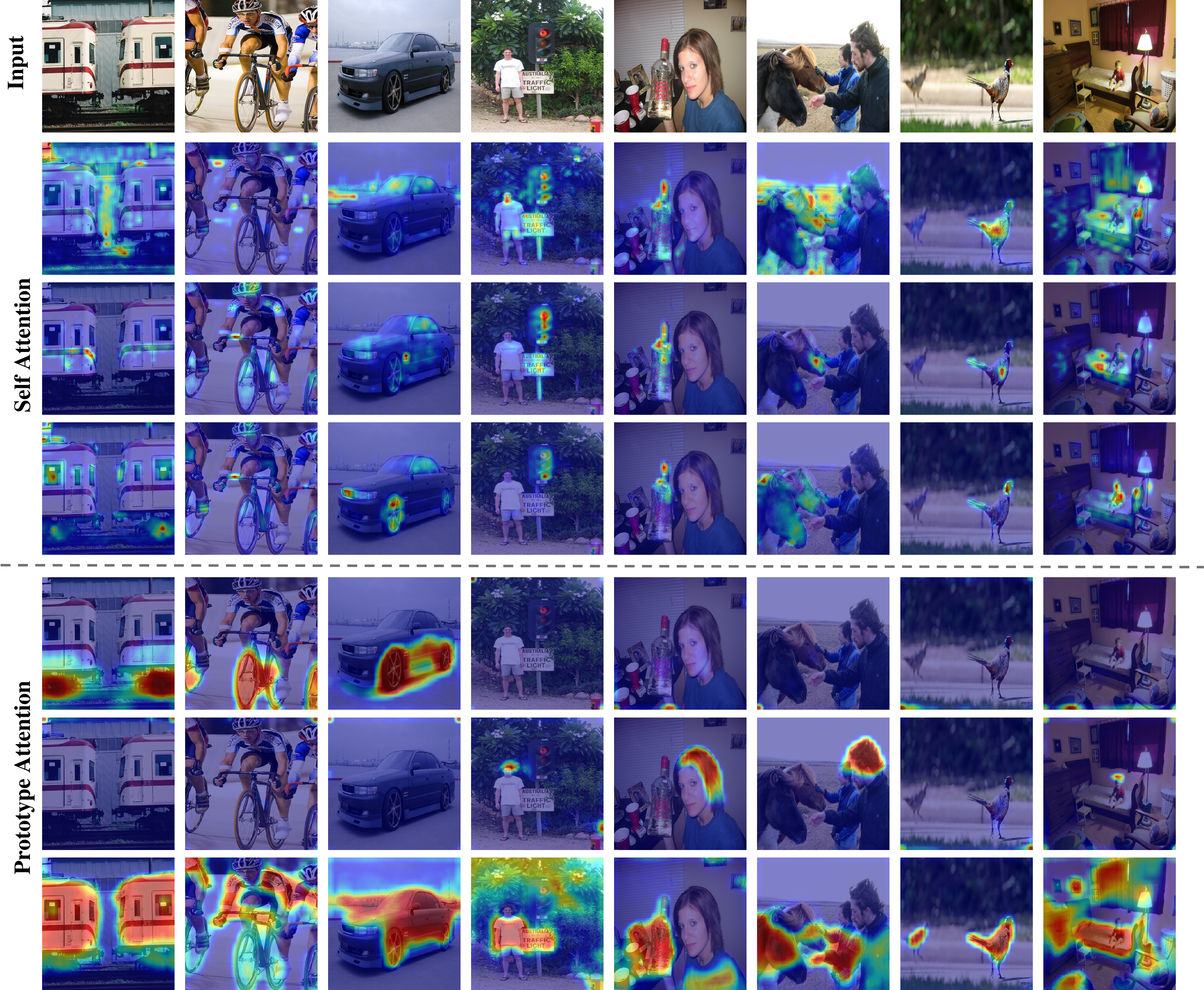}
      \caption{Self-attention maps (\textit{row 2-4}) and cross-attention maps (\textit{row 5-7}) of the OAF module. Best viewed in colors.}
      \label{fig:vis_attn}
    \end{figure*}
  
    \begin{figure*}[h]
      \centering
      \includegraphics[scale=0.28]{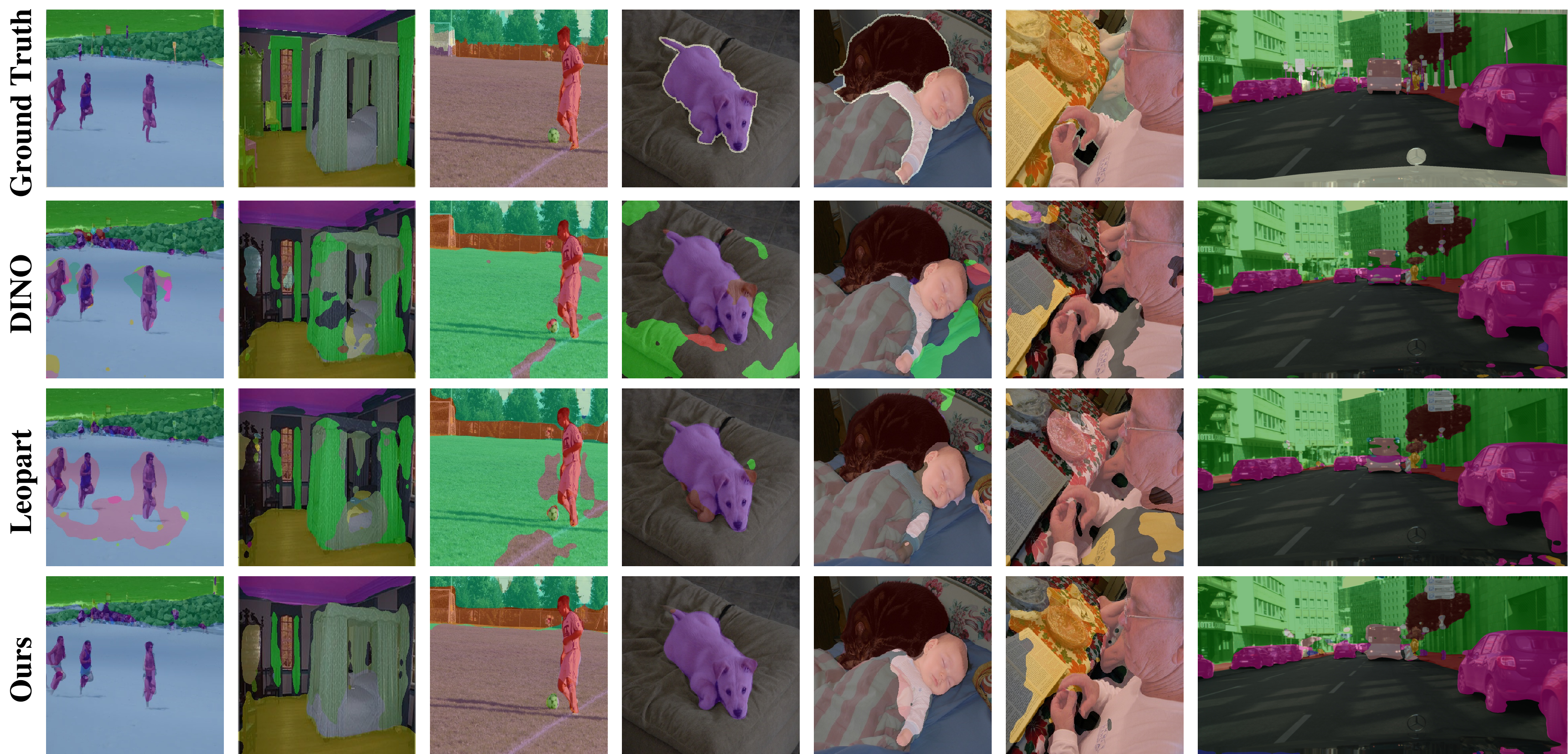}
      \caption{Visualizations of the segmentation results for FCN transfer learning. Best viewed in colors.}
      \label{fig:vis_seg}
    \end{figure*}

    \subsection{Main Results on ImageNet \& COCO}
    \label{subsec:main_results}
  
    \textbf{Evaluation protocol.}
    After the pretraining phase, the online branch is frozen as an embedding extractor. The quality of learned representations is quantitatively measured via two downstream tasks at different granularities: \textbf{1) Semantic segmentation.} A linear model is trained in a supervised manner to map the patch embeddings into the category predictions (27 classes). We utilize the COCOStuff-27 annotations as ground truth labels and report the mean intersection over union (mIoU) and pixel accuracy (Acc) performance. \textbf{2) Image classification.} Following the standard setup in SSL \cite{he2020momentum,swav,zhou2021ibot}, we build a $k$-NN classifier or train a linear classifier to map the $\texttt{[CLS]}$ embedding into the category (1000 classes). Top-1 and Top-5 accuracy are reported.

    \textbf{Optimization strategy.} For the semantic segmentation task, the images are randomly cropped and resized to $336 \times 336$ at the fine-tuning stage and resized to $448\times 448$ at the testing stage. The linear layer is optimized by an Adam optimizer, where the learning rate and the batch size are set to $3\times 10^{-3}$ and $256$, respectively. The model is trained for one epoch. For the image classification task, images are resized and center cropped to $224\times 224$. The linear image classifier is trained on ImageNet-1k for 100 epochs by an AdamW optimizer, where the initial learning rate and the batch size are set to $10^{-5}$ and $4096$, respectively. Details of the $k$-NN classifier are the same as \cite{he2020momentum}. 
  
    \textbf{Quantitative results.} The evaluation results are summarized in \Tbref{tab:main}. Consequently, we have the following observations: \textbf{1)} Our proposed semantic concentration techniques bring consistent improvements in various settings. Specifically, for image classification, our proposed method increases the top-1 accuracy by $1.1\%\sim 4.2\%$. Since previous dense SSL methods ignore the semantic concentration mechanism, the proposed method is more effective in dense tasks, improving the mIoU by $2.5\%\sim10.8\%$. \textbf{2)} Benefitting from the transformer structure, some image-level methods such as DINO and SERE are capable of generating reasonable dense representations. However, compared with dense methods, the performance of dense tasks could be further improved. \textbf{3)} Taking DINO and Leopart as examples, a single training loss might lead to a model preference for specific granularity tasks. Therefore, it is necessary to build a multi-granularity unified learning framework to improve model versatility.
  
  
    \subsection{Transfer Learning on Dense Tasks}
    \label{subsec:transfer}
    \textbf{Evaluation protocol.}
    We verify the generalization of learned representations on three dense tasks:
    \textbf{1) Semantic segmentation with fixed representations.} On top of the frozen encoder, we train a linear model or a fully convolutional network (FCN) to fit four datasets including COCOStuff-27 \cite{caesar2018coco}, PASCAL VOC \cite{everingham2010pascal}, ADE20k \cite{zhou2017scene}, and Cityscapes \cite{cordts2016cityscapes}.
    \textbf{2) Full fine-tuning for semantic segmentation with complex models.} Following \cite{zhou2021ibot,wangpoodle}, the encoder is initialized with SSL pretrained models, and combined with FPN or UPerNet to predict per-pixel classes from multi-level features. All parameters, including the encoder, are jointly fine-tuned.
    \textbf{3) Object detection \& instance segmentation.} 
    The whole network is fine-tuned on MS COCO \cite{lin2014microsoft}. Following the same protocol as \cite{zhou2021ibot}, we predict bounding boxes and instance masks simultaneously on top of Cascade Mask R-CNN \cite{cai2019cascade} for fairness.
    \textbf{4) Video object segmentation (VOS).} Following \cite{dino}, based on the frozen dense presentations, we leverage a $k$-NN strategy to propagate the initial mask and predictions of previous frames. Evaluation results on DAVIS-2017 \cite{pont20172017} are reported.
  
    \textbf{Quantitative results.} The semantic segmentation results are summarized in \Tbref{tab:semantic_seg} and \Tbref{tab:semantic_seg_ft}, the object detection and VOS results are shown in \Tbref{tab:obj_det} and \Tbref{tab:video_seg}, respectively. Consequently, we have the following conclusions: \textbf{1)} The consistent improvement over various tasks and metrics shows the generalization of dense representations generated by our model. Specifically, it improves the metrics by about $1\%\sim 4\%$ for semantic segmentation and $0.4\%\sim 1.6\%$ for objective detection. As for VOS, the region measure $\mathcal{J}_{\mathrm{mean}}$ and boundary measure $\mathcal{F}_{\mathrm{mean}}$ increase by $1.8\%\sim 2.6\%$ and $0.8\%\sim 2.5\%$, respectively.
    \textbf{2)} Fully fine-tuning with more complex models leads to significant improvements in semantic segmentation performance, especially for complex datasets with greater differences from the pretraining data. For instance, it improves from 34\% to 48\% on ADE20k, and from 55\% to 79\% on Cityscapes. After fully fine tuning the backbone, the performance gap between different initialization methods tends to narrow. Nevertheless, our method still maintains a performance advantage, improving results by about 1\%-4\% in most cases.
    \textbf{3)} It is worth noting that VOS requires instance discriminative features while the other two tasks prefer category-aware features. The consistent improvements reveal that the proposed semantic concentration mechanisms can take into account the instance-level information simultaneously.
  
  

  Notably, VOS benefits from instance-discriminative features, whereas semantic segmentation and object detection tend to favor category-aware features. The across-the-board improvements underscore that the semantic concentration mechanisms we propose are adept at accommodating instance-level information, enhancing the model's versatility and effectiveness in diverse tasks.
    \subsection{Ablation Studies}
    \label{subsec:ablations}
    Next, we explore how the components of our framework affect the performance. To reduce the computation cost for linear fune-tining on ImageNet-1k, we apply a fast evaluation protocol: randomly sample $30k$ images and update the model for one epoch. All experiments are conducted based on DINO. The evaluation results are provided in \Tbref{tab:ablation} and detailed discussions are attached as follows:
  

    \textbf{Effect of alignment losses.}
    Comparing \textit{line 1} and \textit{line 2} in \Tbref{tab:ablation}, it can be seen that a single loss significantly reduces performance on another task. Therefore, it is necessary to jointly consider both image-level and dense representation learning. Besides, \textit{line 3} tells that combining two alignment losses cannot form a positive transfer between two types of representations, resulting in a trade-off between tasks.
  
  
    \textbf{Effect of image-level semantic concentration $\bar{\mathcal{L}}_{sc}$.} Comparing \textit{line 1} and \textit{line 4}, we find that simply alignment of the $k$-NN samples improves the top-1 accuracy by $2.4\%\sim 3\%$ on the classification task, which shows the advance of explicit semantic concentration over implicit one. 
  
  
    \textbf{Effect of dense semantic concentration $\mathcal{L}_{sc}$.} Comparing \textit{line 2} and \textit{line 5}, $\mathcal{L}_{sc}$ improves the mIoU by $3.4\%$, from which can be seen that the proposed CoTAP loss successfully distills the noisy correspondences to enhances the semantic concentration. Comparison between \textit{line 3} and \textit{line 6} further supposes the soundness of the above conclusions. However, similar to the experiment in \textit{line 3}, our proposed losses still suffer from the trade-off between different tasks.

    \textbf{Effect of Object-Aware Filter (\textbf{OAF}).} Comparing \textit{line 6} and \textit{line 7}, the proposed OAF block consistently improves the performance in multiple settings, and surpasses the model trained with single-granularity losses (\textit{line 4} and \textit{line 5}). Therefore, to break the trade-off between different tasks, it is a feasible path to enhance the shared pattern. From this perspective, how to generate effective views is worth exploring in future work.
  
  
    \subsection{Effects of Extended Training and Dataset}
    \label{subsec:longer_training}
    Note that our proposed method is trained for extra epochs and an additional dataset COCO. 
    To eliminate the potential fairness issue in experimental comparisons, we control these factors by including baselines with extended training and an extra dataset COCO. Concretely, we keep training the pretrained baselines (DINO, iBOT, and Mugs) with the corresponding loss for the same epochs. We provide three dataset settings for the extra baselines: \textbf{(a)} ImageNet-1k only; \textbf{(b)} COCO only; \textbf{(c)} a 1:1 mixture of ImageNet-1k and COCO, which is exactly the same data configuration as in ours. The resulting models are evaluated on image classification and semantic segmentation. 
  
    The relative performance improvement over the initial model is visualized in \Fgref{fig:longer_training}, from which we have following conclusions: \textbf{1)} On most downstream tasks, extending the training time has little impact on the three baselines. This is because the original training schedule (800 epochs) already brings the models close to convergence on the original loss, making it difficult for a few additional epochs (about 9.3 epochs) to yield significant improvements. In contrast, we introduces the OAF module and two SC losses, effectively altering the representation hypothesis space and reshaping the loss landscape, where the model has not yet converged, so even short additional training can lead to substantial improvements. \textbf{2)} The training dataset affects the trade-off between tasks. Compared with ImageNet-1k, models further trained on COCO perform better on dense tasks, but suffer from a performance drop in image classification. Training on the mixed dataset falls in between. The main reason lies in the weakened prior of the shared pattern. Ours alleviating this trade-off via OAF.
  
    \subsection{Qualitative Results}
    \label{subsec:vis}
    \textbf{Attention maps.} In \Fgref{fig:vis_attn}, we visualize two types of attention maps involved in our method: self-attention maps (\textit{row 2-4}), and prototype attention maps used in the object-aware filter (\textit{row 5-7}). Consequently, it can be observed that although self-attention successfully captures the key points, it might fail to reflect the shared patterns under complicated scenes. For example, in the $7$-th column, the attention map focuses on the large object but ignores the small one. By contrast, the object-prototype attention highlights the shared patterns, such as wheels in the $5$-th row, and hairs in the $6$-th row. Additionally, as shown in the last row, by learning a more general concept like "foreground objects", object-prototype attention is effective in capturing the main body of objects, especially for small objects. See  Appendix {\color{blue}C} for more visualizations.
  

    \textbf{Segmentation results.} To intuitively demonstrate the improvements, we provide some visualization results in \Fgref{fig:vis_seg}. The models are implemented with ViT-S/16 and an FCN block. The predictions of DINO suffer from chaotic features. Although Leopart significantly improves the dense representations, it might segment an object into multiple parts due to the over-dispersion. In contrast, our proposed method avoids this problem and leads to finer details.

  \section{Conclusion \& Future Work}
  \label{sec:conclusion}
  In this paper, we focus on self-supervised learning for dense vision tasks. From practice and theory, we verify that the lack of semantic concentration mechanism in dense SSL leads to over-dispersion in downstream tasks. To address this issue, we propose a self-distillation framework to introduce the non-strict spatial alignment and highlight the shared patterns. Technically, we propose a ranking-based loss called \textit{\textbf{Co}ntinuous-\textbf{T}arget \textbf{AP} (\textbf{CoTAP})} loss to mitigate the effect of noisy and imbalanced pseudo-labels. The key idea is to extend the robust AP loss to continuous targets, making full use of the properties of pseudo-labels such as the noise distribution. Furthermore, we propose the \textit{\textbf{O}bject-\textbf{A}ware \textbf{F}ilter (\textbf{OAF})} module to focus on object-aware patterns. By mapping the image features to the object-prototype-span space, the shared patterns are more noticeable. Experiments on real-world benchmarks validate the advantages of the proposed framework.
  One limitation of this paper is the requirement of object-centric data to extract clean object prototypes. In the future, we will explore the adaptive extraction of object prototypes from uncurated images to fit large-scale data.



%



\ifCLASSOPTIONcaptionsoff
  \newpage
\fi



%
\bibliographystyle{abbrv}
\bibliography{ref}

%

\begin{IEEEbiography}
	[{\includegraphics[width=1in,height=1.25in,clip,keepaspectratio]{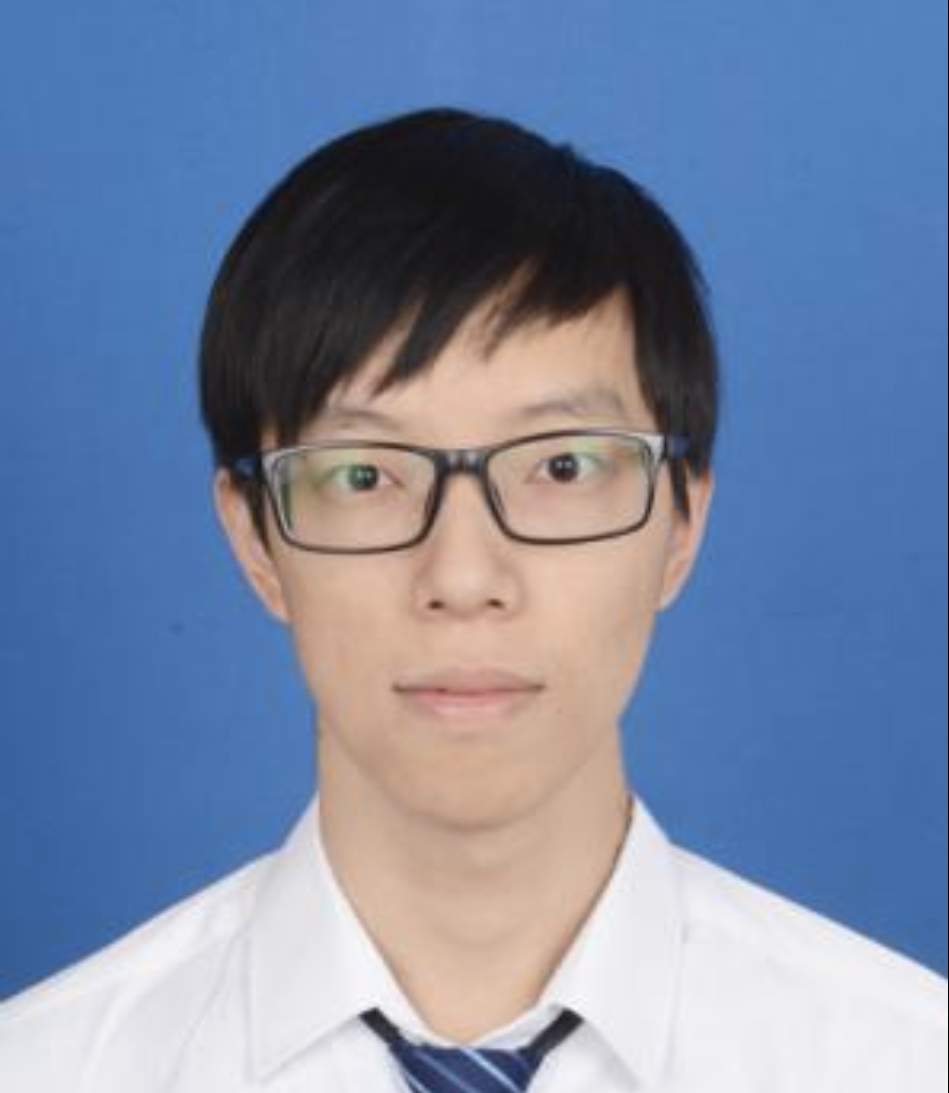}}]
  {Peisong Wen} received the B.S. degree in intelligent science and technology from Nankai University in 2020 and the Ph.D. degree in computer science from University of Chinese Academy of Sciences in 2025. He is currently a Postdoc fellow with University of Chinese Academy of Sciences. His research interest is machine learning and computer vision. He has authored or coauthored 20 academic papers in top-tier international journals and conferences (including T-PAMI, IJCV, NeurIPS, ICML, CVPR, ICCV, etc). 
\end{IEEEbiography}

\begin{IEEEbiography}
	[{\includegraphics[width=1in,height=1.25in,clip,keepaspectratio]{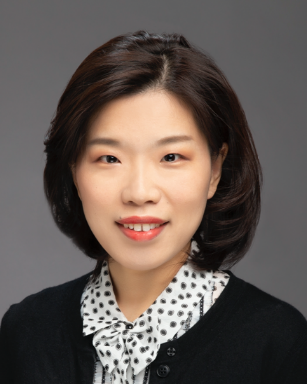}}]{Qianqian Xu} received the B.S. degree in computer science from China University of Mining and Technology in 2007 and the Ph.D. degree in computer science from University of Chinese Academy of Sciences in 2013. She is currently a Professor with the Institute of Computing Technology, Chinese Academy of Sciences, Beijing, China. Her research interests include statistical machine learning, with applications in multimedia and computer vision. She has authored or coauthored 100+ academic papers in prestigious international journals and conferences (including T-PAMI, IJCV, T-IP, NeurIPS, ICML, CVPR, AAAI, etc). Moreover, she serves as an associate editor of IEEE Transactions on Circuits and Systems for Video Technology, IEEE Transactions on Multimedia, and ACM Transactions on Multimedia Computing, Communications, and Applications.
\end{IEEEbiography}

\begin{IEEEbiography}
    [{\includegraphics[width=1in,height=1.25in,clip,keepaspectratio]{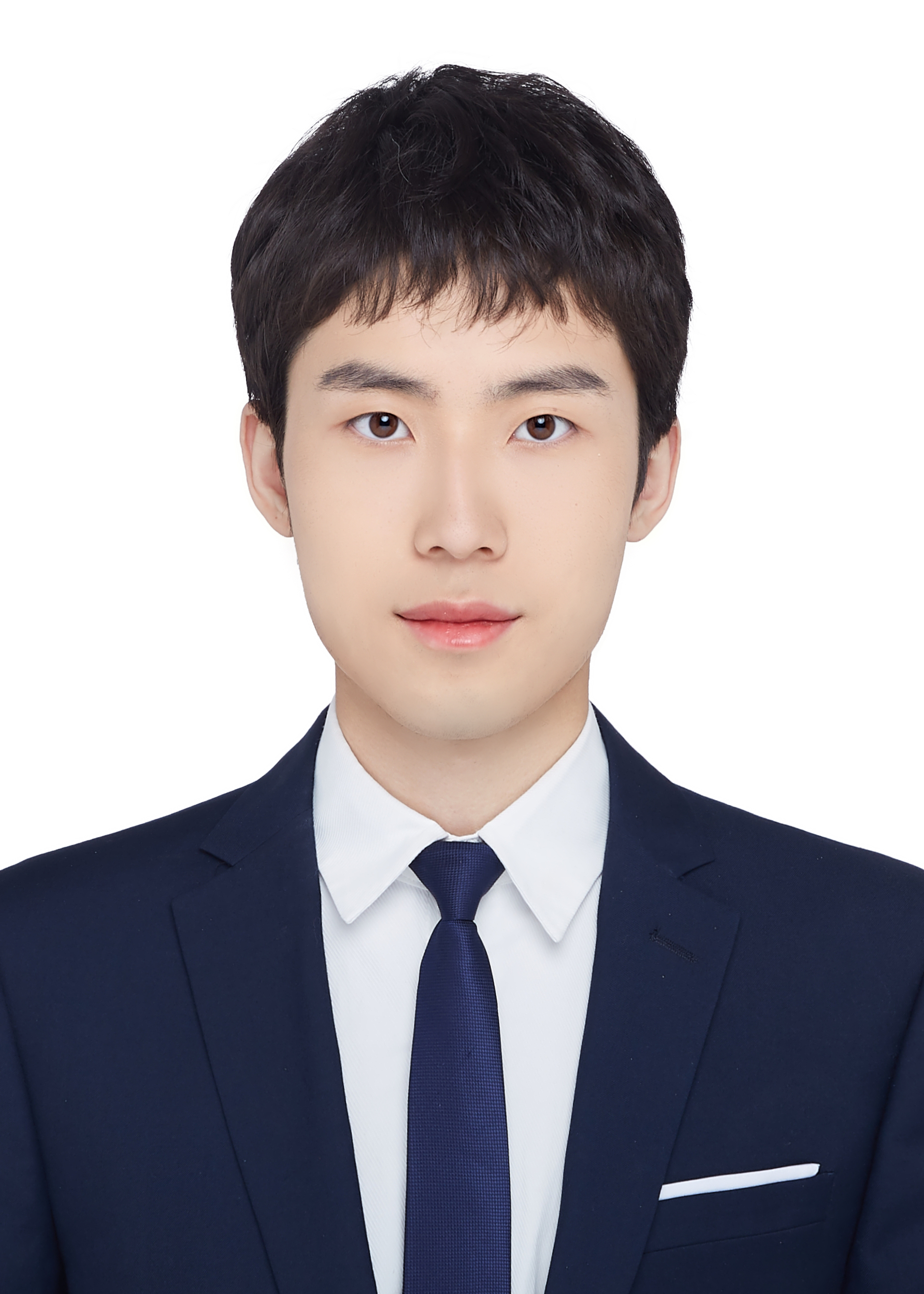}}]
  {Siran Dai} received the B.S. degree in robotic engineering from Beijing University of Technology in 2022. He is currently a Ph.D. student at University of Chinese Academy of Sciences. His research interest includes machine learning, computer vision, and information theory.
\end{IEEEbiography}

\vspace{1cm}
\begin{IEEEbiography}
	[{\includegraphics[width=1in,height=1.25in,clip,keepaspectratio]{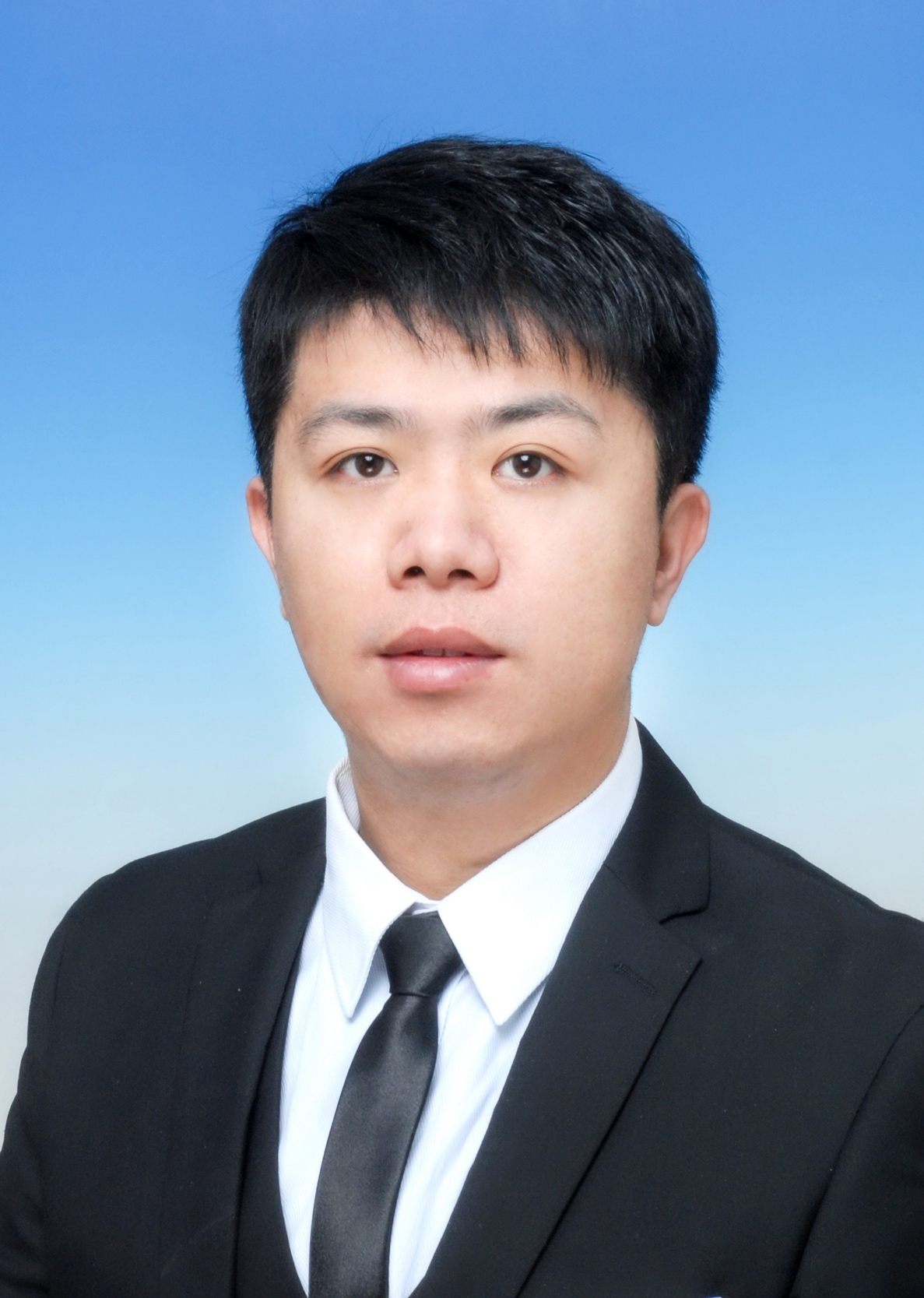}}]{\textbf{Runmin Cong}} received the Ph.D. degree in information and communication engineering from Tianjin University, Tianjin, China, in June 2019.
	He is currently a Professor with the School of Control Science and Engineering, Shandong University (SDU), Jinan, China. His research interests include computer vision, multimedia understanding, content enhancement, machine learning, etc. He has published more than 90 papers in prestigious international journals and conferences, including 2 ESI hot papers (Top 0.1\%), 15 ESI highly cited papers (Top 1\%). In addition, 30 China patents have been authorized.
	Dr. Cong was a recipient of the Young Elite Scientist Sponsorship Program by the China Association for Science and Technology, the Hong Kong Scholars Program, the IEEE ICME Best Student Paper Runner-Up Award, the ACM SIGWEB China Rising Star Award. He serves as an Associate Editor of the IEEE Transactions on Neural Networks and Learning Systems, Neurocomputing, IEEE Journal of Oceanic Engineering, the Youth Editorial Board Member of the CAAI Transactions on Intelligence Technology and the Area Chair/SPC/PC Member of NeurIPS, CVPR, ICML, ICLR, ICCV, ECCV, SIGGRAPH, ACM MM, AAAI, IJCAI, ICME, ICPR, ICMR, etc.
\end{IEEEbiography}

\begin{IEEEbiography}
	[{\includegraphics[width=1in,height=1.25in,clip,keepaspectratio]{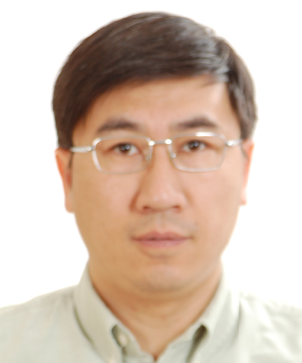}}]{Qingming Huang} is a chair professor in the University of Chinese Academy of Sciences and an adjunct research professor in the Institute of Computing Technology, Chinese Academy of Sciences. He graduated with a Bachelor degree in Computer Science in 1988 and Ph.D. degree in Computer Engineering in 1994, both from Harbin Institute of Technology, China. His research areas include multimedia computing, image processing, computer vision and pattern recognition. He has authored or coauthored more than 400 academic papers in prestigious international journals and top-level international conferences. He was the associate editor of IEEE Trans. on CSVT and Acta Automatica Sinica, and the reviewer of various international journals including IEEE Trans. on PAMI, IEEE Trans. on Image Processing, IEEE Trans. on Multimedia, etc. He is a Fellow of IEEE and has served as general chair, program chair, area chair and TPC member for various conferences, including ACM Multimedia, CVPR, ICCV, ICML, AAAI, ICMR, PCM, BigMM, PSIVT, etc.
\end{IEEEbiography}





\clearpage
\onecolumn
\appendices

\section*{\textcolor{blue}{\Large{Contents}}}
\startcontents[appendices]
\printcontents[appendices]{l}{1}{\setcounter{tocdepth}{3}}
\newpage

\section{Detailed Proofs}
\label{app:proofs}
\ClsCorrect*
\begin{proof}
  Without loss of generality, let $y(\xi) = 1$. According to the definition of $G_f$, $\xi$ is correctly classified if
  \begin{equation*}
    \|\zi - \bm{\mu}_k\| \geq \|\zi - \bm{\mu}_1\|, ~~\forall 2 \leq k \leq K.
  \end{equation*}
  Denote $\mathbb{E}_{j\in \mathcal{T}_2^{(i)}, l\in \mathcal{D}^{(k)}}[\bm{z}_j^\top \bm{z}_l] = e_k$. Recall that $\bm{\mu}_k = \mathbb{E}_{j\in \mathcal{D}^{(k)}}[\zj]$. According to the Young's inequality, for any $0 \leq \beta \leq 1$, the left side could be upper bounded:
  \begin{equation}
  \label{eq:acc_temp_1}
    \begin{aligned}
      \|\zi - \bm{\mu}_k\|^2 
      &\geq (1 - \beta)\left\|\zi - \mathbb{E}_{j\in \mathcal{T}_2^{(i)}}[\zj]\right\|^2 
        + (1 - 1 / \beta) \left\|\mathbb{E}_{j\in \mathcal{T}_2^{(i)}}[\zj] - \mathbb{E}_{l\in \mathcal{D}^{(k)}}[\bm{z}_l]\right\|^2 \\
      &= (1 - \beta)\left\|\zi - \mathbb{E}_{j\in \mathcal{T}_2^{(i)}}[\zj]\right\|^2
      - 2(1 - 1 / \beta) \mathbb{E}_{j\in \mathcal{T}_2^{(i)}, l\in \mathcal{D}^{(k)}}[\bm{z}_j^\top \bm{z}_l] \\
      &~~~~~~~~+ (1 - 1 / \beta)\left\|\mathbb{E}_{j\in \mathcal{T}_2^{(i)}}[\zj]\right\|^2 + (1 - 1 / \beta)\left\|\mathbb{E}_{j\in \mathcal{D}^{(k)}}[\zj]\right\|^2 \\
      &\geq (1 - \beta)\left\|\zi - \mathbb{E}_{j\in \mathcal{T}_2^{(i)}}[\zj]\right\|^2
      - 2(1 - 1 / \beta) e_k
      + 2(1 - 1 / \beta)r^2
    \end{aligned}
  \end{equation}
  Similarly, for any $0 \leq \alpha \leq 1$, we have
  \begin{equation}
  \label{eq:acc_temp_2}
    \begin{aligned}
      \|\zi - \bm{\mu}_1\|^2 
      &\leq (1 + \alpha)\left\|\zi - \mathbb{E}_{j\in \mathcal{T}_2^{(i)}}[\zj]\right\|^2 
        + (1 + 1 / \alpha) \left\|\mathbb{E}_{j\in \mathcal{T}_2^{(i)}}[\zj] - \mathbb{E}_{l\in \mathcal{D}^{(1)}}[\bm{z}_l]\right\|^2 \\
      &= (1 + \alpha)\left\|\zi - \mathbb{E}_{j\in \mathcal{T}_2^{(i)}}[\zj]\right\|^2 - 2(1 + 1 / \alpha) \mathbb{E}_{j\in \mathcal{T}_2^{(i)}, l\in \mathcal{D}^{(1)}}[\bm{z}_j^\top \bm{z}_l] \\
      &~~~~~~~~+ (1 + 1 / \alpha)\left\|\mathbb{E}_{j\in \mathcal{T}_2^{(i)}}[\zj]\right\|^2 + (1 + 1 / \alpha)\left\|\mathbb{E}_{j\in \mathcal{D}^{(1)}}[\zj]\right\|^2 \\
      &\leq (1 + \alpha)\left\|\zi - \mathbb{E}_{j\in \mathcal{T}_2^{(i)}}[\zj]\right\|^2 - 2(1 + 1 / \alpha) e_1 + 2(1 + 1 / \alpha)r^2
    \end{aligned}
  \end{equation}
  Combining \Eqref{eq:acc_temp_1} and \Eqref{eq:acc_temp_2}, we get
  \begin{equation}
    \begin{aligned}
      \|\zi - \bm{\mu}_k\|^2 - \|\zi - \bm{\mu}_1\|^2
      &\geq -(\alpha + \beta) \left\|\zi - \mathbb{E}_{j\in \mathcal{T}_2^{(i)}}[\zj]\right\|^2
      - 2(1 / \alpha + 1 / \beta)r^2 \\
      &~~~~~~~~+ 2(1 + 1 / \alpha) e_1 - 2(1 - 1 / \beta) e_k \\
      &= 2(e_1 - e_k) - \left(\alpha\left\|\zi - \mathbb{E}_{j\in \mathcal{T}_2^{(i)}}[\zj]\right\|^2 + 2 / \alpha \cdot (r^2 - e_1) \right) \\
      &~~~~~~~~- \left(\beta\left\|\zi - \mathbb{E}_{j\in \mathcal{T}_2^{(i)}}[\zj]\right\|^2 + 2 / \beta \cdot (r^2 - e_k) \right).
    \end{aligned}
  \end{equation}
  By setting $\alpha = \frac{\sqrt{2(r^2 - e_1)}}{\big\|\zi - \mathbb{E}_{j\in \mathcal{T}_2^{(i)}}[\zj]\big\|}$, $\beta = \frac{\sqrt{2(r^2 - e_k)}}{\big\|\zi - \mathbb{E}_{j\in \mathcal{T}_2^{(i)}}[\zj]\big\|}$, we get
  \begin{equation}
    \begin{aligned}
      \|\zi - \bm{\mu}_k\|^2 - \|\zi - \bm{\mu}_1\|^2
      &\geq 2(e_1 - e_k) - 2\sqrt{2}\left(\sqrt{r^2 - e_1} + \sqrt{r^2 - e_k}\right) \cdot \big\|\zi - \mathbb{E}_{j\in \mathcal{T}_2^{(i)}}[\zj]\big\| \\
      &\geq 2\left(4r - \sqrt{2}\left(\sqrt{r^2 - e_1} + \sqrt{r^2 - e_k}\right)\right) \cdot \big\|\zi - \mathbb{E}_{j\in \mathcal{T}_2^{(i)}}[\zj]\big\| \\
      &\geq 2\left(4r - 2\sqrt{2(r^2 - e_k)}\right) \cdot \big\|\zi - \mathbb{E}_{j\in \mathcal{T}_2^{(i)}}[\zj]\big\|.
    \end{aligned}
  \end{equation}
  The proof is completed by substituting $e_k \geq -r^2$ into the above inequality.
\end{proof}

\SumDelta*
\begin{proof}
    Consider the Lagrangian function of \OP{1}:
    \begin{equation*}
      \mathcal{L}_{Lag} = \mathcal{L}_{SSL} + \frac{1}{2N^2} \sum_{i.j=1}^N \lambda_{ij} (\vij \|\zi - \zj\|^2 - 1),
    \end{equation*}
    where $\vij = 1 / (\phi_f^2\|\xi - \xj\|^2 + \epsilon)$, $\epsilon$ is a small constant to avoid numerical problem, $\lambda_{ij} \geq 0$ are the multiplier. Let $\bm{A}$ be the adjacency matrix such that its element $a_{ij} = \whatij + \lambda_{ij} \vij$, $\bm{L} = \bm{D} - \bm{A}$ be the Laplacian matrix, where $\bm{D} = diag(\sum_{j}a_{1j},\cdots,\sum_{j}a_{nj})$ is the degree matrix, \ie, 
    \begin{equation}
    \bm{L}_{ij} = \left\{
        \begin{array}{cc}
        \begin{aligned}
            & \sum_{k\neq i} a_{ik}, & j = i, \\
            & -a_{ij}, & j \neq i.
        \end{aligned} 
        \end{array}
      \right.
    \end{equation}

    According to the spectral analysis theory for graphs, $\mathcal{L}_{Lag}$ can be reformulated as:
    \begin{equation*}
      \begin{aligned}
        \mathcal{L}_{Lag}(\bm{Z}, \lambda) = \frac{1}{N^2} Tr(\bm{Z}\bm{L}\bm{Z}^\top) + \alpha \cdot Tr\left(\Big(\bm{C}_f - \frac{r^2}{d}\mathbf{I}\Big)\Big(\bm{C}_f - \frac{r^2}{d}\mathbf{I}\Big)^\top\right) - \frac{1}{2N^2} \sum_{i,j=1}^N \lambda_{ij} \vij
      \end{aligned}
    \end{equation*}
    Consider the Karush-Kuhn-Tucker (KKT) conditions:
    \begin{equation}
    \label{eq:kkt}
      \begin{aligned}
        \frac{\partial \mathcal{L}_{Lag}}{\partial \bm{Z}} &= \frac{2}{N^2} \bm{Z}\bm{L} + \frac{4\alpha}{N} \bm{C}_f \bm{Z} - \frac{4\alpha r^2}{dN} \bm{Z} = 0 \Rightarrow \bm{Z}\bm{L} + 2 \alpha N (\bm{C}_f - \frac{r^2}{d}\mathbf{I})\bm{Z} = 0, \\
        \frac{\partial \mathcal{L}_{Lag}}{\partial \lambda_{ij}} &\propto \vij \|\zi - \zj\|^2 - 1 = 0 \text{ or } \lambda_{ij} = 0.
      \end{aligned}
    \end{equation}
    Here the first equation is due to the following facts:
    \begin{equation*}
      \begin{aligned}
        Tr\left(\Big(\bm{C}_f - \frac{r^2}{d}\mathbf{I}\Big)\Big(\bm{C}_f - \frac{r^2}{d}\mathbf{I}\Big)^\top\right) = Tr(\bm{C}_f^\top \bm{C}_f) - \frac{2r^2}{d}\cdot Tr(\bm{C}_f) + \frac{r^4}{d}
      \end{aligned}
    \end{equation*}
    
    \begin{equation*}
      \begin{aligned}
        Tr(\bm{C}_f^\top \bm{C}_f) 
        = \frac{1}{N^2} \sum_{i=1}^N (\zi^\top \zi)^2 + \frac{1}{N^2}\sum_{i=1}^N\sum_{j\neq i}(\zi^\top \zj)^2 \Rightarrow 
        \frac{\partial Tr(\bm{C}_f^\top \bm{C}_f)}{\partial \zi}
        = \frac{4}{N^2} \sum_{i=1}^N \zi \zi^\top \zi + \frac{4}{N^2}\sum_{j\neq i}\zj \zi^\top \zj
        = \frac{4}{N} \bm{C}_f \zi 
      \end{aligned}
    \end{equation*}

  Next, \Eqref{eq:kkt} can be reformulated as:
  \begin{equation}
    - \bm{Z}\bm{L} = 2 \alpha N (\bm{C}_f - \frac{r^2}{d}\mathbf{I})\bm{Z}.
  \end{equation}
  Consider the $i$-th column of the left side:
  \begin{equation}
    -\sum_{j=1}^N \bm{L}_{i,j} \zj = \sum_{j\neq i} a_{ij} \zj - \sum_{k\neq i} a_{ik} \zi.
  \end{equation}
  Similarly, for the right side, we have
  \begin{equation}
    2\alpha N (\bm{C}_f - \frac{r^2}{d}\mathbf{I})\cdot \zi = 2\alpha \sum_{j\neq i} \zj \zj^\top \zi + 2\alpha \cdot \zi\zi^\top \zi - \frac{2\alpha N r^2}{d} \zi
  \end{equation}
  Combining the above results we get:
  \begin{equation}
    0 = \sum_{j\neq i} \zj \zj^\top \zi - \frac{1}{2\alpha}\sum_{j\neq i} a_{ij} \zj + \frac{1}{2\alpha} \sum_{j\neq i} a_{ij} \zi + \zi\zi^\top \zi - \frac{N r^2}{d} \zi.
  \end{equation}
  Then, by left multiplying $\zi$ we get:
  \begin{equation}
  \label{eq:sim_equ}
    \begin{aligned}
      0 &= \sum_{j\neq i} (\zi^\top \zj)^2 - \frac{1}{2\alpha}\sum_{j\neq i} a_{ij} \cdot \zi^\top\zj + \frac{r^2}{2\alpha} \sum_{j\neq i} a_{ij} + r^4 - \frac{Nr^4}{d} \\
      &= \sum_{j\neq i} (\zi^\top \zj)^2 - \frac{1}{2\alpha}\sum_{j\neq i} \whatij \cdot \zi^\top\zj + \frac{r^2}{2\alpha} \sum_{j\neq i} \whatij + r^4 - \frac{Nr^4}{d} - \frac{1}{2\alpha}\sum_{j\neq i} \lambda_{ij} \vij \cdot \zi^\top\zj + \frac{r^2}{2\alpha} \sum_{j\neq i} \lambda_{ij} \vij \\
      &= \sum_{j\neq i} (\zi^\top \zj)^2 - \frac{1}{2\alpha}\sum_{j\neq i} \whatij \cdot \zi^\top\zj + \frac{r^2}{2\alpha} \sum_{j\neq i} \whatij + r^4 - \frac{Nr^4}{d} + \frac{1}{2\alpha}\sum_{j\neq i} \lambda_{ij} \vij \cdot (r^2 - \zi^\top\zj) \\
      &= \sum_{j\neq i} \left(\sij - \frac{\whatij}{4\alpha}\right)^2 + \sum_{j\neq i} \left(2r^2 \cdot \frac{\whatij}{4\alpha} - (\frac{\whatij}{4\alpha})^2\right) + r^4 - \frac{Nr^4}{d} + \frac{1}{2\alpha}\sum_{j\neq i} \lambda_{ij} \vij \cdot (r^2 - \zi^\top\zj) \\
      &= \sum_{j\neq i} \left(\sij - \frac{\whatij}{4\alpha}\right)^2 - \sum_{j\neq i} \left(\frac{\whatij}{4\alpha} - r^2\right)^2 + Nr^4 (1- 1/d) + \frac{1}{2\alpha}\sum_{j\neq i} \lambda_{ij} \vij \cdot (r^2 - \zi^\top\zj).
    \end{aligned}
  \end{equation}

  According to \Eqref{eq:kkt}, we have $\sij = r^2 - 1 / \vij$ if $\lambda_{ij} > 0$, or equivalently 
  \begin{equation}
  \label{eq:temp3}
    \lambda_{ij} \vij \cdot (r^2 - \zi^\top\zj) = \lambda_{ij} \vij \cdot (1 / \vij) = \lambda_{ij}.
  \end{equation}
  By substituting \Eqref{eq:temp3} into \Eqref{eq:sim_equ}, we get
  \begin{equation}
    \sum_{j\neq i} \left(\sij - \frac{\whatij}{4\alpha}\right)^2 - \sum_{j\neq i} \left(\frac{\whatij}{4\alpha} - r^2\right)^2 = Nr^4 (1- 1/d) + \frac{1}{2\alpha}\sum_{j\neq i} \lambda_{ij} \geq Nr^4 (1- 1/d).
  \end{equation}
  Notice that $\whatij = 1$ iff $j \in \mathcal{T}_2^{(i)}$, hence we have
  \begin{equation}
    \begin{aligned}
      \sum_{j\in \mathcal{T}_1^{(i)}} \left(\sij - \frac{\whatij}{4\alpha}\right)^2 - \sum_{j\in \mathcal{T}_1^{(i)}} \left(\frac{\whatij}{4\alpha} - r^2\right)^2 &= \sum_{j\in \mathcal{T}_1^{(i)}}\big(r^4 - (\sij)^2\big) = \Delta^{(i)}_1, \\
      \sum_{j\in \mathcal{T}_2^{(i)}} \left(\sij - \frac{\whatij}{4\alpha}\right)^2 - \sum_{j\in \mathcal{T}_2^{(i)}} \left(\frac{\whatij}{4\alpha} - r^2\right)^2 &= \sum_{j\in \mathcal{T}_2^{(i)}}\Big(\big(r^2 - \frac{1}{4\alpha}\big)^2 - \big(\sij - \frac{1}{4\alpha}\big)^2\Big) = \Delta^{(i)}_2, \\
      \sum_{j\in \mathcal{D}^{(k)}} \left(\sij - \frac{\whatij}{4\alpha}\right)^2 - \sum_{j\in \mathcal{D}^{(k)}} \left(\frac{\whatij}{4\alpha} - r^2\right)^2 &= \sum_{j\in \mathcal{D}^{(k)}}\big(r^4 - (\sij)^2\big) = \Delta^{(i,k)}.
    \end{aligned}
  \end{equation}
  The proof is completed since $\bigcup_{k\in [K]\setminus \{y(\xi)\}} \mathcal{D}^{(k)} \cup \mathcal{T}_1^{(i)} \cup \mathcal{T}_2^{(i)} = [N]$ and $\left(\zi^\top \zi - \frac{\whatij}{4\alpha}\right)^2 = \left(\frac{\whatij}{4\alpha} - r^2\right)^2$.
\end{proof}

\begin{lem}
\label{lem:prop_in_class}
  Let \Asmpref{asm:in-of-class-dist} holds, then for any $i \in [N]$, a feasible solution $\bm{Z}$ satisfies $\Prob_{j \in \mathcal{T}_2^{(i)}, l\in \mathcal{T}_1^{(j)}}\Big[\bm{z}_j^\top \bm{z}_l \leq d_T\Big] \leq q_T$.
\end{lem}
\begin{proof}
According to \Asmpref{asm:in-of-class-dist}, at least with probability $1 - q_T$ we have
\begin{equation*}
  \|\sjl\|^2 = 2r^2 - 2\sjl \leq \phi_f^2 \|\xj - \bm{x}_l\|^2 < 2(r^2 - d_T) ~~ \Rightarrow~~ \sjl > d_T.
\end{equation*}
\end{proof}

\begin{lem}
\label{lem:var_bnd}
  Let $\{X_i \in S\}_{i=1}^n$ be a sample set from $S \subseteq [a, b]$. Given a constant $c \in (a, b)$, denote $S_1 = \{X_i \in S | X_i \leq c\}$, $S_2 = \{X_i \in S | X_i > c\}$. If $\Prob_{X \in S}[X \in S_1] \leq p$, then we have
  \begin{equation*}
    \mathbb{V}_{X\in S} [X] \leq p\cdot \frac{(c - a)^2}{4} + \frac{(b - c)^2}{4} + p \cdot (b-a)^2
  \end{equation*}
  where $\mathbb{V}$ refers to the variance.

    Let \Asmpref{asm:in-of-class-dist} holds, then for any $i \in [N]$, a feasible solution $\bm{Z}$ satisfies
    \begin{equation*}
      \mathbb{V}_{j\in \mathcal{T}_2^{(i)}, l \in \mathcal{T}_1^{(j)}}[\sjl] \leq (1 - q_T)(r^2 - d_T)^2 / 4 - q_T (r^2 + d_T)^2 / 4 - q_Tr^4,
    \end{equation*}
  \end{lem}
  \begin{proof}
    Denote $q = \Prob_{X \in S}[X \in S_1]$, $m_1 = \Expt_{X \in S_1}[X]$, $m_2 = \Expt_{X \in S_2}[X]$, then we have $\Expt_{X \in S}[X] = q m_1 + (1 - q) m_2$, and
    \begin{equation}
    \label{eq:var_bnd}
      \begin{aligned}
        \mathbb{V}_{X\in S} [X] 
        &= \Expt_{X \in S}[(X - qm_1 - (1-q)m_2)^2] \\
        &= q \Expt_{X \in S_1}[(X - m_1 + (1-q)(m_1 - m_2))^2] + (1-q) \Expt_{X \in S_2}[(X - m_2 + q (m_2 - m_1))^2] \\
        &= q \Expt_{X \in S_1}[(X - m_1)^2] + q(1-q)^2 (m_1 - m_2)^2 \\
        &~~~~+ (1-q) \Expt_{X \in S_2}[(X - m_2)^2] + (1-q)q^2 (m_2 - m_1)^2 \\
        &= q \mathbb{V}_{X \in S_1}[X] + (1 - q) \mathbb{V}_{X \in S_2}[X] + (1-q)q(m_1 - m_2)^2 \\
        &\leq q \mathbb{V}_{X \in S_1}[X] + \mathbb{V}_{X \in S_2}[X] + q(m_1 - m_2)^2 \\
        &\leq p \mathbb{V}_{X \in S_1}[X] + \mathbb{V}_{X \in S_2}[X] + p(m_1 - m_2)^2 \\
      \end{aligned}
    \end{equation}
    Notice that $S_1 \subseteq [a, c]$, $S_2 \subseteq [c, b]$, thus we have $\mathbb{V}_{X \in S_1}[X] \leq \frac{(c - a)^2}{4}$, $\mathbb{V}_{X \in S_2}[X] \leq \frac{(b - c)^2}{4}$. By combining the above results, we have 
    \begin{equation}
      \mathbb{V}_{X\in S} [X] \leq p\cdot \frac{(c - a)^2}{4} + \frac{(b - c)^2}{4} + p \cdot (b-a)^2.
    \end{equation}
  \end{proof}

\begin{lem}
  \label{lem:bounds_of_expt}
  Let \Asmpref{asm:in-of-class-dist} holds. For all $i \in [N]$, $0 \leq \delta_T \leq r^2$, denote $\sigma_i(\delta_{T}) = \Prob_{j,l \in \mathcal{T}_2^{(i)}}\left[\sjl \leq \delta_{T} \right]$. If $\bm{Z}$ achieves the optimum of \OP{1}, $\frac{1}{8r^2} \leq \alpha \leq \frac{1}{4r^2}$, $d_T \geq \frac{r^2}{3}$, then for any constant $0 < \delta_D \leq r^2$, $k \in [N] \setminus \{y(\xi)\}$, the following conclusions hold:
  \begin{itemize}
    \item[\textbf{(a)}] $ \mathbb{E}_{j\in \mathcal{T}_2^{(i)}, l \in \mathcal{T}_1^{(j)}}[\sjl] 
      \geq 3d_T / 2 - (1/2 + \sqrt{6q_T}) r^2$;
    \item[\textbf{(b)}] $\mathbb{E}_{j,l\in \mathcal{T}_2^{(i)}}[\sjl]
      \geq \delta_T - \sqrt{5\sigma_i(\delta_T) (r^2+\delta_T)r^2}$;
    \item[\textbf{(c)}] $\mathbb{E}_{j\in \mathcal{T}_2^{(i)}, l \in \mathcal{D}^{(k)}}[\sjl]
    \leq \left(\sqrt{\frac{Nr^4}{d} + 30|\mathcal{T}_2^{(i)}|r^4 - |\mathcal{D}^{(k)}| d_T^2}
    + \sqrt{|\mathcal{T}_1^{(i)}| q_T}d_T + \sqrt{|\mathcal{T}_2^{(i)}|\sigma_i(\delta_T)}\delta_T\right) / \sqrt{|\mathcal{D}^{(k)}|}$
  \end{itemize}
\end{lem}
\begin{proof}
  To begin with, we consider the bounds of $\Delta^{(i)}_1$, $\Delta^{(i)}_2$, and $\Delta^{(i,k)}$. According to \Lemref{lem:prop_in_class}, we have
  \begin{equation}
  \label{eq:upb_delta_1_temp_1}
    \begin{aligned}
      \frac{1}{\big|\mathcal{T}_2^{(i)}\big|\big|\mathcal{T}_1^{(i)}\big|}\sum_{j\in \mathcal{T}_2^{(i)}}\Delta^{(j)}_1 
      &= \frac{1}{\big|\mathcal{T}_2^{(i)}\big|\big|\mathcal{T}_1^{(i)}\big|} \sum_{j\in \mathcal{T}_2^{(i)}, l\in \mathcal{T}_1^{(j)}}\left(r^4 - (\sjl)^2\right) \\
      &\leq \Prob_{j\in\mathcal{T}_2^{(i)}, l\in\mathcal{T}_1^{(j)}}[\sjl > d_T]\cdot \left(r^4 - d_T^2\right) 
      + \Prob_{j\in\mathcal{T}_2^{(i)}, l\in\mathcal{T}_1^{(j)}}[\sjl \leq d_T]\cdot r^4 \\
      &= r^4 - \Prob_{j\in\mathcal{T}_2^{(i)}, l\in\mathcal{T}_1^{(j)}}[\sjl > d_T]\cdot d_T^2 \\
      &\leq r^4 - d_T^2 + q_T \cdot d_T^2.
    \end{aligned}
  \end{equation}
  Notice that
  \begin{equation}
  \label{eq:eq_delta_1_temp_1}
    \begin{aligned}
      \frac{1}{\big|\mathcal{T}_2^{(i)}\big|\big|\mathcal{T}_1^{(i)}\big|}\sum_{j\in \mathcal{T}_2^{(i)}}\Delta^{(j)}_1 
      =& r^4 - \mathbb{E}_{j\in \mathcal{T}_2^{(i)}, l \in \mathcal{T}_1^{(j)}}[(\sjl)^2] \\
      =& r^4 - \left(\mathbb{E}_{j\in \mathcal{T}_2^{(i)}, l \in \mathcal{T}_1^{(j)}}[\sjl]\right)^2 
        - \mathbb{V}_{j\in \mathcal{T}_2^{(i)}, l \in \mathcal{T}_1^{(j)}}[\sjl]
    \end{aligned}
  \end{equation}
  By substituting \Eqref{eq:eq_delta_1_temp_1} into \Eqref{eq:upb_delta_1_temp_1} we get \textbf{(a)}:
  \begin{equation}
    \begin{aligned}
      \mathbb{E}_{j\in \mathcal{T}_2^{(i)}, l \in \mathcal{T}_1^{(j)}}[\sjl] 
      \geq& \sqrt{(1 - q_T)d_T^2 - \mathbb{V}_{j\in \mathcal{T}_2^{(i)}, l \in \mathcal{T}_1^{(j)}}[\sjl]} \\
      \overset{(i)}{\geq}& \sqrt{(1 - q_T)d_T^2 - q_T (r^2+d_T)^2 / 4 - (r^2 - d_T)^2 / 4 - 4q_T \cdot r^2} \\
      =& \sqrt{d_T^2 - \frac{(r^2 - d_T)^2}{4} - 6q_T r^4} \\
      \overset{(ii)}{\geq}& d_T - (r^2 - d_T) / 2 - \sqrt{6q_T}r^2 \\
      =& 3d_T / 2 - (1/2 + \sqrt{6q_T}) r^2,
    \end{aligned}
  \end{equation}
  where $(i)$ is due to \Lemref{lem:var_bnd} and $(ii)$ requires $d_T \geq r^2 / 3$.
  
  Similarly, for $\Delta^{(j)}_2$ we have
  \begin{equation}
  \label{eq:temp_delta_2_2}
    \begin{aligned}
      \frac{1}{\big|\mathcal{T}_2^{(i)}\big|^2}\sum_{j\in \mathcal{T}_2^{(i)}}\Delta^{(j)}_2 
      &= \frac{1}{\big|\mathcal{T}_2^{(i)}\big|^2}\sum_{j, l\in \mathcal{T}_2^{(i)}}\left(\left(r^2 - \frac{1}{4\alpha}\right)^2 - \left(\sjl - \frac{1}{4\alpha}\right)^2\right) \\
      &\geq \left(r^2 - \frac{1}{4\alpha}\right)^2 
        - \Prob_{j, l \in \mathcal{T}_2^{(i)}}[\sjl > \delta_T] \left(\delta_T - \frac{1}{4\alpha}\right)^2
        - \Prob_{j, l \in \mathcal{T}_2^{(i)}}[\sjl \leq \delta_T] \left(-r^2 - \frac{1}{4\alpha}\right)^2 \\
      &= \left(r^2 - \frac{1}{4\alpha}\right)^2 
        - (1 - \sigma_i(\delta_T)) \left(\delta_T - \frac{1}{4\alpha}\right)^2
        - \sigma_i(\delta_T) \left(r^2 + \frac{1}{4\alpha}\right)^2 \\
      &= \left(r^2 - \frac{1}{4\alpha}\right)^2 
      - \left(\delta_T - \frac{1}{4\alpha}\right)^2
      + \sigma_i(\delta_T) \left(\delta_T^2 - r^4 - \frac{r^2 + \delta_T}{2\alpha}\right) \\
      &= \left(r^2 - \frac{1}{4\alpha}\right)^2 
      - \left(\delta_T - \frac{1}{4\alpha}\right)^2
      - \sigma_i(\delta_T) (r^2+\delta_T)(r^2+\frac{1}{2\alpha}), \\
    \end{aligned},
  \end{equation}
  and
  \begin{equation}
    \begin{aligned}
      \frac{1}{\big|\mathcal{T}_2^{(i)}\big|^2}\sum_{j\in \mathcal{T}_2^{(j)}}\Delta^{(i)}_2 
      =& (r^2 - \frac{1}{4\alpha})^2 - \left(\mathbb{E}_{j,l\in \mathcal{T}_2^{(i)}}[(\sjl - \frac{1}{4\alpha})^2]\right) \\
      =& (r^2 - \frac{1}{4\alpha})^2 
        - \left(\mathbb{E}_{j,l\in \mathcal{T}_2^{(i)}}[\sjl] - \frac{1}{4\alpha}\right)^2 
        - \mathbb{V}_{j,l\in \mathcal{T}_2^{(i)}}[\sjl],
    \end{aligned}
  \end{equation}
  leading to
  \begin{equation}
    \begin{aligned}
      \left(\mathbb{E}_{j,l\in \mathcal{T}_2^{(i)}}[\sjl] - \frac{1}{4\alpha}\right)^2
      \leq& \left(\delta_T - \frac{1}{4\alpha}\right)^2
        + \sigma_i(\delta_T) (r^2+\delta_T)(r^2+\frac{1}{2\alpha}) \\
      \leq& \left(\delta_T - \frac{1}{4\alpha}\right)^2
        + 3\sigma_i(\delta_T) (r^2+\delta_T)r^2. \\
      \end{aligned}
  \end{equation}
  Since $\mathbb{E}_{j,l\in \mathcal{T}_2^{(i)}}[\sjl] \leq r^2 \leq \frac{1}{4\alpha}$, we get \textbf{(b)}:
  \begin{equation}
    \begin{aligned}
      \mathbb{E}_{j,l\in \mathcal{T}_2^{(i)}}[\sjl]
      \geq& \frac{1}{4\alpha} - \sqrt{\left(\delta_T - \frac{1}{4\alpha}\right)^2 + 3\sigma_i(\delta_T) (r^2+\delta_T)r^2} \\
      \geq& \frac{1}{4\alpha} - \left(\frac{1}{4\alpha} - \delta_T\right) - \sqrt{3\sigma_i(\delta_T) (r^2+\delta_T)r^2} \\
      =& \delta_T - \sqrt{3\sigma_i(\delta_T) (r^2+\delta_T)r^2}.
    \end{aligned}
  \end{equation} 

  Next, we consider the upper bound of $\sum_{j\in \mathcal{T}_2^{(i)}}(\Delta_1^{(j)} + \Delta_2^{(j)})$. First, we have
  \begin{equation}
    \begin{aligned}
      \frac{1}{\big|\mathcal{T}_2^{(i)}\big|^2}\sum_{j\in \mathcal{T}_2^{(i)}}\Delta^{(j)}_2 
      &= \frac{1}{\big|\mathcal{T}_2^{(i)}\big|^2}\sum_{j, l\in \mathcal{T}_2^{(i)}}\left(\left(r^2 - \frac{1}{4\alpha}\right)^2 - \left(\sjl - \frac{1}{4\alpha}\right)^2\right) \\
      &\leq \Prob_{j, l \in \mathcal{T}_2^{(i)}}[\sjl \leq \delta_T] \cdot \left(\left(r^2 - \frac{1}{4\alpha}\right)^2 - \left(\delta_T - \frac{1}{4\alpha}\right)^2\right) \\
      &= \sigma_i(\delta_T) \cdot \left(r^4 - \delta_T^2 + \frac{\delta_T - r^2}{2\alpha}\right). \\
    \end{aligned}
  \end{equation}
  Notice that
  \begin{equation}
    \begin{aligned}
      r^4 - \delta_T^2 - \frac{r^2 - \delta_T}{2\alpha} 
      =& -2\left(\delta_T - \frac{1}{2\alpha}\right)^2 + \left(r^2 - \frac{1}{4\alpha}\right)^2 + \frac{3\delta_T}{2\alpha} - \frac{7}{16\alpha^2} + \delta_T^2 \\ 
      \leq& 2\left(r^2 - \frac{1}{4\alpha}\right)^2 - \frac{3\delta_T}{2\alpha} + \frac{7}{16\alpha^2} + \delta_T^2 \\
      \leq& 2\left(r^2 - \frac{1}{4\alpha}\right)^2 + 28r^2 + \delta_T^2,
    \end{aligned}
  \end{equation}
  where the last inequality is due to $\delta_T \geq 0$ and $\alpha \leq 1 / (4r^2)$. If we further set $\alpha \geq 1 / (8r^2)$, we have
  \begin{equation}
    2\left(r^2 - \frac{1}{4\alpha}\right)^2 \leq 2r^4 \Rightarrow r^4 - \delta_T^2 - \frac{r^2 - \delta_T}{2\alpha} \leq 30r^4 + \delta_T^2.
  \end{equation}
  Combining the above results we get
  \begin{equation}
  \label{eq:temp_delta_2_1}
    \frac{1}{\big|\mathcal{T}_2^{(i)}\big|^2}\sum_{j\in \mathcal{T}_2^{(i)}}\Delta^{(j)}_2 
    \leq \sigma_i(\delta_T) \cdot (30r^4 + \delta_T^2) \leq r^4 + \sigma_i(\delta_T) \cdot (29r^4 + \delta_T^2).
  \end{equation}

  Adding \Eqref{eq:upb_delta_1_temp_1} and \Eqref{eq:temp_delta_2_1} we get
  \begin{equation}
  \label{eq:temp_delta_1_add_delta_2}
    \begin{aligned}
      \frac{1}{|\mathcal{T}_2^{(i)}|} \sum_{j\in \mathcal{T}_2^{(i)}}\left(\Delta^{(j)}_1 + \Delta^{(j)}_2\right)
      &\leq |\mathcal{T}_1^{(i)}| \left(r^4 - d_T^2 + q_T \cdot d_T^2\right) + |\mathcal{T}_2^{(i)}| r^4 + |\mathcal{T}_2^{(i)}|\sigma_i(\delta_T) \cdot (29r^4+\delta_T^2) \\
      &= \big(|\mathcal{T}_1^{(i)}| + |\mathcal{T}_2^{(i)}|\big)r^4 - |\mathcal{T}_1^{(i)}| (1 - q_T) d_T^2 + |\mathcal{T}_2^{(i)}|\sigma_i(\delta_T) \cdot (29r^4 + \delta_T^2)
    \end{aligned}
  \end{equation}
  According to \Lemref{lem:sum_delta} and the fact that $\Delta^{(j,k)} \leq |\mathcal{D}^{(k)}| r^4$, we have
  \begin{equation}
  \label{eq:temp_delta_k_2}
  \begin{aligned}
      \frac{1}{|\mathcal{T}_2^{(i)}|}\sum_{j\in \mathcal{T}_2^{(i)}}\Delta^{(j,k)}
      &\geq (1 - \frac{1}{d}) N r^4 
      - \frac{1}{|\mathcal{T}_2^{(i)}|} \sum_{j\in \mathcal{T}_2^{(i)}}\left(\sum_{t \in [K] \setminus \{k, y(\xi)\}} \Delta^{(j,t)} - \Delta^{(j)}_1 - \Delta^{(j)}_2\right) \\
      &\geq (1 - \frac{1}{d})N r^4 
      - \sum_{t \in [K] \setminus \{k, y(\xi)\}} |\mathcal{D}^{(t)}| r^4 
      - \frac{1}{|\mathcal{T}_2^{(i)}|} \sum_{j\in \mathcal{T}_2^{(i)}}\left(\Delta^{(j)}_1 + \Delta^{(j)}_2\right) \\
      &= \left(|\mathcal{D}^{(y(\xi))} \cup \mathcal{D}^{(k)}| - \frac{N}{d}\right) r^4 
      - \frac{1}{|\mathcal{T}_2^{(i)}|} \sum_{j\in \mathcal{T}_2^{(i)}}\left(\Delta^{(j)}_1 + \Delta^{(j)}_2\right).
    \end{aligned}
  \end{equation}
  Notice that 
  \begin{equation}
  \label{eq:temp_delta_k_3}
    \begin{aligned}
      \frac{1}{\big|\mathcal{T}_2^{(i)}\big|\big|\mathcal{D}^{(k)}\big|}\sum_{j\in \mathcal{T}_2^{(i)}}\Delta^{(j,k)} 
      &= r^4 - \mathbb{E}_{j\in \mathcal{T}_2^{(i)}, l \in \mathcal{D}^{(k)}}[(\sjl)^2] \\
      &= r^4 - \left(\mathbb{E}_{j\in \mathcal{T}_2^{(i)}, l \in \mathcal{D}^{(k)}}[\sjl]\right)^2 - \mathbb{V}_{j\in \mathcal{T}_2^{(i)}, l \in \mathcal{D}^{(k)}}[\sjl] \\
      &\leq r^4 - \left(\mathbb{E}_{j\in \mathcal{T}_2^{(i)}, l \in \mathcal{D}^{(k)}}[\sjl]\right)^2
    \end{aligned}
  \end{equation}
  By substituting $\mathcal{D}^{(y(\xi))} = \mathcal{T}_1^{(i)} \cup \mathcal{T}_2^{(i)}$, \Eqref{eq:temp_delta_k_3}, and \Eqref{eq:temp_delta_1_add_delta_2} into \Eqref{eq:temp_delta_k_2}, we get
  \begin{equation}
  \label{eq:upb_pd_temp1}
    \begin{aligned}
      |\mathcal{D}^{(k)}| \left(\mathbb{E}_{j\in \mathcal{T}_2^{(i)}, l \in \mathcal{D}^{(k)}}[\sjl]\right)^2 
      \leq& \left(\frac{N}{d} - |\mathcal{T}_1^{(i)}| - |\mathcal{T}_2^{(i)}|\right)r^4
        +\frac{1}{|\mathcal{T}_2^{(i)}|} \sum_{j\in \mathcal{T}_2^{(i)}}\left(\Delta^{(j)}_1 + \Delta^{(j)}_2\right) \\
      \leq& \left(\frac{N}{d} - |\mathcal{T}_1^{(i)}| - |\mathcal{T}_2^{(i)}|\right)r^4 
      + \big(|\mathcal{T}_1^{(i)}| + |\mathcal{T}_2^{(i)}|\big)r^4 \\
      &- |\mathcal{T}_1^{(i)}| (1 - q_T) d_T^2 
      + |\mathcal{T}_2^{(i)}|\sigma_i(\delta_T) \cdot (29r^4 + \delta_T^2) \\
      =& \frac{Nr^4}{d} + 29|\mathcal{T}_2^{(i)}|r^4 - |\mathcal{T}_1^{(i)}| d_T^2 
      + |\mathcal{T}_1^{(i)}| q_T d_T^2 
      + |\mathcal{T}_2^{(i)}|\sigma_i(\delta_T) \cdot \delta_T^2 \\
      \leq & \frac{Nr^4}{d} + 30|\mathcal{T}_2^{(i)}|r^4 - |\mathcal{D}^{(k)}| d_T^2 
      + |\mathcal{T}_1^{(i)}| q_T d_T^2 
      + |\mathcal{T}_2^{(i)}|\sigma_i(\delta_T) \cdot \delta_T^2 \\
    \end{aligned}
  \end{equation} 
  or equivalently
  \begin{equation}
  \label{eq:upb_pd_temp2}
    \begin{aligned}
      \sqrt{|\mathcal{D}^{(k)}|} \cdot \mathbb{E}_{j\in \mathcal{T}_2^{(i)}, l \in \mathcal{D}^{(k)}}[\sjl]
      \leq& \sqrt{\frac{Nr^4}{d} + 30|\mathcal{T}_2^{(i)}|r^4 - |\mathcal{D}^{(k)}| d_T^2
      + |\mathcal{T}_1^{(i)}| q_T d_T^2 
      + |\mathcal{T}_2^{(i)}|\sigma_i(\delta_T) \cdot \delta_T^2} \\
      \leq& \sqrt{\frac{Nr^4}{d} + 30|\mathcal{T}_2^{(i)}|r^4 - |\mathcal{D}^{(k)}| d_T^2}
      + \sqrt{|\mathcal{T}_1^{(i)}| q_T}d_T + \sqrt{|\mathcal{T}_2^{(i)}|\sigma_i(\delta_T)}\delta_T. \\
    \end{aligned}
  \end{equation}

\end{proof}

\begin{thm}[Formal version of \Thmref{thm:main}]
\label{thm:main_formal}
  Let \Asmpref{asm:in-of-class-dist} hold and assume that all categories have $n = \bar{N} / K$ samples. Consider a proper feature dimension such that $K / 2 \leq d \leq K$. Denote $p_T = \sup_{i\in [N]}\Prob_{j,l \in \mathcal{T}_2^{(i)}}\left[\sjl \leq \delta_{T} \right]$. By minimizing $\mathcal{L}_{SSL}$ and setting $\delta_T \geq d_T \geq (3 + 2\sqrt{K/d + 30/n - 1/13})r^2 / \sqrt{13}$, for all $q_T, p_T$ s.t. $C_1 > C_2 \sqrt{p_T} + C_3 \sqrt{q_T}$, the error rate on the downstream task is upper bounded by:
  \begin{equation*}
    \mathcal{E}(f) \leq 32r^2 \cdot \frac{r^2 - \delta_T + r \sqrt{5p_T\left(\delta_T + r^2\right)}}{(C_1 - C_2 \sqrt{p_T} - C_3 \sqrt{q_T})^2},
  \end{equation*}
  where
  \begin{equation*}
    \begin{aligned}
      C_1 &= \frac{2\delta_T + (n-1)(3d_T - r^2)}{2n} - \sqrt{\frac{Kr^4}{d} - d_T^2 + \frac{30r^4}{n}} \geq 0, \\
      C_2 &= \frac{r}{n} \sqrt{5\left(\delta_T + r^2\right)} + \sqrt{\frac{1}{n}}\delta_T, \\
      C_3 &= \frac{n-1}{n}\sqrt{6}r^2 + \sqrt{\frac{n-1}{n}}  d_T \geq 0.
    \end{aligned}
  \end{equation*}
\end{thm}  

\begin{proof}
  According to \Lemref{lem:bounds_of_expt}, we have
  \begin{equation}
  \label{eq:expt_left}
    \begin{aligned}
      &\mathbb{E}_{j\in \mathcal{T}_2^{(i)}, l\in \mathcal{D}^{(y(\xi))}}[\bm{z}_j^\top \bm{z}_l] 
        - \mathbb{E}_{j\in \mathcal{T}_2^{(i)}, l \in \mathcal{D}^{(k)}}[\sjl] \\
      =& \frac{1}{n} \mathbb{E}_{j, l \in \mathcal{T}_2^{(i)}}[\bm{z}_j^\top \bm{z}_l] 
      + \frac{n - 1}{n}\mathbb{E}_{j \in \mathcal{T}_2^{(i)}, l\in \mathcal{T}_1^{(j)}}[\bm{z}_j^\top \bm{z}_l]
      - \mathbb{E}_{j\in \mathcal{T}_2^{(i)}, l \in \mathcal{D}^{(k)}}[\sjl] \\
      \geq& \frac{\delta_T - r \sqrt{5\sigma_i(\delta_T) \left(\delta_T + r^2\right)}}{n} 
        + \frac{3(n-1)d_T / 2 - (n-1)(1/2 + \sqrt{6q_T}) r^2}{n} \\
        &- \sqrt{Kr^4 / d - d_T^2 + 30r^4/n}
        - \sqrt{\frac{n-1}{n} q_T}  d_T
        - \sqrt{\frac{1}{n}\sigma_i(\delta_T)}\delta_T \\
      \geq& \frac{2\delta_T + (n-1)(3d_T - r^2)}{2n} 
        - \frac{r}{n} \sqrt{5\sigma_i(\delta_T) \left(\delta_T + r^2\right)}
        - \frac{n-1}{n}\sqrt{6q_T}r^2 \\
        &- \sqrt{Kr^4 / d - d_T^2 + 30r^4/n}
        - \sqrt{\frac{n-1}{n} q_T}  d_T
        - \sqrt{\frac{1}{n}\sigma_i(\delta_T)}\delta_T \\
      =& C_1 - C_2 \sqrt{\sigma_i(\delta_T)} - C_3 \sqrt{q_T} \\
      \geq& C_1 - C_2 \sqrt{p_T} - C_3 \sqrt{q_T}
    \end{aligned}
  \end{equation}
  where the last inequality holds when $C_3 \geq 0$, which is obvious if $n \geq 4$. Besides, when $\delta_T \geq d_T \geq (3 + 2\sqrt{13K/d + 390/n - 1})r^2 / 13$ we get $3d_T - r^2 - 2\sqrt{Kr^4 / d - d_T^2 + 390r^4/n} \geq 0$, leading to
  \begin{equation}
    C_1 = \frac{2(\delta_T - d_T) + r^2 - d_T}{2n} + \frac{3d_T - r^2}{2} - \sqrt{Kr^4 / d - d_T^2 + 30r^4/n} \geq 0.
  \end{equation}
  For any positive constant $0 < a < r^2$ we have
  \begin{equation}
    \begin{aligned}
      \mathbb{E}_{j,l\in \mathcal{T}_2^{(i)}}[\sjl]
      \leq& \Prob_{j \in \mathcal{T}_2^{(i)}}\left[\mathbb{E}_{l\in \mathcal{T}_2^{(j)}}[\sjl] \leq r^2 - a^2\right] (r^2 - a^2) \\ 
        &+ \left(1 - \Prob_{j \in \mathcal{T}_2^{(i)}}\left[\mathbb{E}_{l\in \mathcal{T}_2^{(j)}}[\sjl] \leq r^2 - a^2\right]\right)r^2 \\
      =& r^2 - a^2 \Prob_{j \in \mathcal{T}_2^{(i)}}\left[\mathbb{E}_{l\in \mathcal{T}_2^{(j)}}[\sjl] \leq r^2 - a^2\right] \\
      \leq& r^2 - a^2 \Prob_{j \in \mathcal{T}_2^{(i)}}\left[\big\|\zj - \mathbb{E}_{l\in \mathcal{T}_2^{(j)}}[\zl]\big\| \geq \sqrt{2}a\right].
    \end{aligned}
  \end{equation}
  On the other side, according to \textbf{(b)} in \Lemref{lem:bounds_of_expt} we have
  \begin{equation}
    \mathbb{E}_{j,l\in \mathcal{T}_2^{(i)}}[\sjl] \geq \delta_T - r \sqrt{5\sigma_i(\delta_T) \left(\delta_T + r^2\right)} \geq \delta_T - r\sqrt{5p_T \left(\delta_T + r^2\right)}.
  \end{equation}
  The above two inequalities lead to
  $\Prob_{j \in \mathcal{T}_2^{(i)}}\left[\big\|\zj - \mathbb{E}_{l\in \mathcal{T}_2^{(j)}}[\zl]\big\| \geq \sqrt{2}a\right] \leq \left(r^2 - \delta_T + r \sqrt{5p_T\left(\delta_T + r^2\right)}\right) / a^2$. 
  Notice that there exists a subset $\mathcal{B} \subset [N]$ such that $\{\mathcal{T}_2^{(i)} | i \in \mathcal{B}\}$ forms a partition of $[N]$, thus the above conclusion can be easily extended to all $j \in [N]$, \ie, 
  \begin{equation}
  \label{eq:expt_right}
    \Prob_{j \in [N]}\left[\big\|\zj - \mathbb{E}_{l\in \mathcal{T}_2^{(j)}}[\zl]\big\| \geq \sqrt{2}a\right] \leq \left(r^2 - \delta_T + r \sqrt{5p_T\left(\delta_T + r^2\right)}\right) / a^2
  \end{equation}

Consider \Lemref{lem:condition_cls_correct}, \Eqref{eq:expt_left} and \Eqref{eq:expt_right}, by setting $\sqrt{2}a = \frac{C_1 - C_2 \sqrt{p_T} - C_3 \sqrt{q_T}}{4r}$ we get
\begin{equation}
  \begin{aligned}
    \mathcal{E}(f) =& \Prob_{i \in [N]}[G_f(\xi) \neq y(\xi)] \\
    \leq& \Prob_{i \in [N]}\left[4r\big\|\zi - \mathbb{E}_{j\in \mathcal{T}_2^{(i)}}[\zj]\big\| \geq \mathbb{E}_{j\in \mathcal{T}_2^{(i)}, l\in \mathcal{D}^{(y(\xi))}}[\bm{z}_j^\top \bm{z}_l] 
    - \mathbb{E}_{j\in \mathcal{T}_2^{(i)}, l \in \mathcal{D}^{(k)}}[\sjl] \right] \\
    \leq& 32r^2 \cdot \frac{r^2 - \delta_T + r \sqrt{5p_T\left(\delta_T + r^2\right)}}{(C_1 - C_2 \sqrt{p_T} - C_3 \sqrt{q_T})^2}.
  \end{aligned}
\end{equation}
\end{proof}

\CotapUpBnd*
\begin{proof}  
\begin{equation}
  \begin{aligned}
    \ell_{CTAP}(\bm{u}, \bm{v}) 
    =& \frac{1}{N_{pair}}\sum_{k=1}^{N_{pair}} \gamma(\qhat_k) \cdot \ell_{AP}(\bm{u}, \bm{v};\qhat_k) \\
    =& \frac{1}{N_{pair}}\sum_{k=1}^{N_{pair}} \frac{\gamma(\qhat_k)}{\sum_{i=1}^{N_{pair}}\one[\qhat_i \geq \qhat_k]} \sum_{\qhat_i \geq \qhat_k} g\left(\frac{\sum_{\qhat_j < \qhat_k} \one[\phat_i \leq \phat_j]}{\sum_{\qhat_j \geq \qhat_k} \one[\phat_i \leq \phat_j]}\right) \\
    =& \frac{1}{N_{pair}}\sum_{k=1}^{N_{pair}} \frac{\gamma(\qhat_k)}{\sum_{i=1}^{N_{pair}}\one[\qhat_i \geq \qhat_k]} \sum_{\qhat_i \geq \qhat_k} g\left(
      \frac{\sum_{j=1}^{N_{pair}} \one[\phat_i \leq \phat_j] \one[\qhat_j < \qhat_k]}
      {\sum_{j=1}^{N_{pair}} \one[\phat_i \leq \phat_j]\one[\qhat_j \geq \qhat_k]}
    \right) \\
    \leq& \frac{1}{N_{pair}}\sum_{k=1}^{N_{pair}} \frac{\gamma(\qhat_k)}{\sum_{i=1}^{N_{pair}}\one[\qhat_i \geq \qhat_k]} \sum_{\qhat_i \geq \qhat_k} g\left(
      \frac{\sum_{j=1}^{N_{pair}} \one[\phat_i \leq \phat_j] \one[\qhat_j < \qhat_i]}
      {\sum_{j=1}^{N_{pair}} \one[\phat_i \leq \phat_j]\one[\qhat_j \geq \qhat_i]}
    \right) \\
    \leq& \frac{1}{N_{pair}}\sum_{i=1}^{N_{pair}} g\left(
      \frac{\sum_{j=1}^{N_{pair}} \one[\phat_i \leq \phat_j] \one[\qhat_j < \qhat_i]}
      {\sum_{j=1}^{N_{pair}} \one[\phat_i \leq \phat_j]\one[\qhat_j \geq \qhat_i]}
    \right) \cdot \sum_{\qhat_k \leq \qhat_i} \frac{\gamma(\qhat_k)}{\sum_{j=1}^{N_{pair}}\one[\qhat_j \geq \qhat_k]} \\
    =& \frac{1}{N_{pair}}\sum_{i=1}^{N_{pair}} \tilde{\gamma}_i \cdot g\left(
      \psi_i \sum_{j=1}^{N_{pair}} \one[\phat_i \leq \phat_j] \one[\qhat_j < \qhat_i]
    \right), \\
  \end{aligned}
\end{equation}
where $\psi_i = 1 / \sum_{j=1}^{N_{pair}} \one[\phat_i \leq \phat_j]\one[\qhat_j \geq \qhat_i]$, $\tilde{\gamma}_i = \sum_{\qhat_k \leq \qhat_i} \frac{\gamma(\qhat_k)}{\sum_{j=1}^{N_{pair}}\one[\qhat_j \geq \qhat_k]}$.
\end{proof}


\section{Implementation Details \& Additional Results}
\label{app:imp_detail}

\subsection{Effects of Non-strict Spatial Alignment \& Augmentations}
\label{app:cifar}
\textbf{Overall Setting.} To explore the impact of position randomness on model performance, we conducted three experiments: \textbf{1)} For each view, we independently applied the random crop and augmentations. \textbf{2)} For both views, we resized the original image to $32\times 32$ and then applied additional augmentations. \textbf{3)} For both views, we randomly cropped the same area of the original image, resized it to $32\times 32$, and then applied other augmentations respectively.  For \textbf{2)} and \textbf{3)}, both views share the same position.

\textbf{Basic Setup.} Following \cite{huang2022towards}, we trained models for 800 epochs on both CIFAR-10 and CIFAR-100 datasets using a batch size of $512$. We utilize SGD as the optimizer, enhanced with a CosineAnnealing scheduler. The momentum was set at $0.9$, and the weight decay was fixed at $5 \times 10^{-4}$. The initial learning rates for SimCLR \cite{chen2020simple}, MoCo \cite{he2020momentum}, SimSiam \cite{simsiam}, and Barlow Twins \cite{zbontar2021barlow} were $0.06$, $0.08$, $0.05$, and $0.11$, respectively. For evaluation, we employed a $k$-NN classifier with $K=200$ and a temperature setting of $\tau=0.1$, presenting the average results from three separate trials.

\textbf{Augmentations.}
In terms of image augmentations, we set the scale for random cropping to $(0.08, 1)$. Gaussian blur was applied with a $50\%$ chance, using a mean kernel size of $3.2$. The random grayscale conversion was applied with a $20\%$ probability. For color jittering, we applied it $80\%$ of the time with an intensity of $0.5$. Lastly, we used random horizontal flipping with a $50\%$ chance.

\textbf{Additional results.} As shown in \Tbref{tab:effect_aug_full}, performance under the setting \textbf{2)} and \textbf{3)} consistently decreases, which supports our hypothesis that non-strict spatial alignment is crucial for semantic concatenation. 

\begin{table}[h]
  \caption{Downstream accuracy under different types of augmentations: (a) random cropping; (b) random Gaussian blur; (c) color dropping; (d) color distortion; (e) random horizontal flipping. Here \Circpipe~ means applying the same random cropping to generate two views for an image.}
  \label{tab:effect_aug_full}
  \centering
  \begin{tabular}{ccccc|cccc|cccc}
  \toprule
  \multicolumn{5}{c|}{Transformations} & \multicolumn{4}{c|}{CIFAR-10} &\multicolumn{4}{c}{CIFAR-100} \\
  (a) & (b) & (c) & (d) & (e) & SimCLR  & Barlow Twins  & MoCo  & SimSiam & SimCLR  & Barlow Twins  & MoCo  & SimSiam \\ 
  \midrule
  \cmark  & \cmark  & \cmark  & \cmark  & \cmark  & \cellcolor[rgb]{0.557 0.812 0.788}\textbf{\underline{89.8}} & \cellcolor[rgb]{0.557 0.812 0.788}\textbf{\underline{86.9}} & \cellcolor[rgb]{0.557 0.812 0.788}\textbf{\underline{90.1}} & \cellcolor[rgb]{0.557 0.812 0.788}\textbf{\underline{90.6}} & \cellcolor[rgb]{1.000 0.745 0.478}\textbf{\underline{57.7}} & \cellcolor[rgb]{1.000 0.745 0.478}\textbf{\underline{58.0}} & \cellcolor[rgb]{1.000 0.745 0.478}\textbf{\underline{64.2}} & \cellcolor[rgb]{1.000 0.745 0.478}\textbf{\underline{63.5}} ~\\
  \cmark  & \cmark  & \cmark  & \cmark  &  ~  & \cellcolor[rgb]{0.569 0.817 0.794}88.5 & \cellcolor[rgb]{0.569 0.817 0.794}85.4 & \cellcolor[rgb]{0.561 0.813 0.790}89.7 & \cellcolor[rgb]{0.566 0.816 0.792}89.3 & \cellcolor[rgb]{1.000 0.759 0.508}55.4 & \cellcolor[rgb]{1.000 0.760 0.509}55.2 & \cellcolor[rgb]{1.000 0.754 0.497}62.5 & \cellcolor[rgb]{1.000 0.759 0.507}60.3 ~\\
  \cmark  & \cmark  & \cmark  &  ~  &  ~  & \cellcolor[rgb]{0.614 0.836 0.815}83.5 & \cellcolor[rgb]{0.595 0.828 0.806}82.0 & \cellcolor[rgb]{0.587 0.825 0.803}86.8 & \cellcolor[rgb]{0.592 0.827 0.805}85.4 & \cellcolor[rgb]{1.000 0.824 0.639}45.1 & \cellcolor[rgb]{1.000 0.785 0.560}50.4 & \cellcolor[rgb]{1.000 0.784 0.558}57.0 & \cellcolor[rgb]{1.000 0.799 0.588}51.4 ~\\
  \cmark  & \cmark  &  ~  &  ~  &  ~  & \cellcolor[rgb]{0.798 0.914 0.903}63.2 & \cellcolor[rgb]{0.705 0.875 0.859}67.8 & \cellcolor[rgb]{0.695 0.870 0.854}75.1 & \cellcolor[rgb]{0.742 0.890 0.877}63.3 & \cellcolor[rgb]{1.000 0.931 0.858}28.0 & \cellcolor[rgb]{1.000 0.871 0.736}34.1 & \cellcolor[rgb]{1.000 0.875 0.744}40.2 & \cellcolor[rgb]{1.000 0.910 0.816}26.3 ~\\
  \cmark  &  ~  &  ~  &  ~  &  ~  & \cellcolor[rgb]{0.802 0.916 0.905}62.7 & \cellcolor[rgb]{0.705 0.875 0.859}67.8 & \cellcolor[rgb]{0.696 0.871 0.855}74.9 & \cellcolor[rgb]{0.754 0.895 0.882}61.5 & \cellcolor[rgb]{1.000 0.931 0.859}27.9 & \cellcolor[rgb]{1.000 0.872 0.737}34.0 & \cellcolor[rgb]{1.000 0.878 0.750}39.6 & \cellcolor[rgb]{1.000 0.912 0.820}25.9 ~\\
  & \cmark  & \cmark  & \cmark  & \cmark  & \cellcolor[rgb]{0.887 0.952 0.946}53.3 & \cellcolor[rgb]{0.917 0.965 0.960}40.6 & \cellcolor[rgb]{0.872 0.946 0.939}55.8 & \cellcolor[rgb]{0.955 0.981 0.979}31.8 & \cellcolor[rgb]{1.000 0.973 0.945}21.2 & \cellcolor[rgb]{1.000 0.955 0.908}18.1 & \cellcolor[rgb]{1.000 0.934 0.865}29.2 & \cellcolor[rgb]{1.000 0.974 0.948}11.8 ~\\
  ~  & \cmark  & \cmark  & \cmark  &  ~  & \cellcolor[rgb]{0.889 0.953 0.947}53.1 & \cellcolor[rgb]{0.925 0.968 0.964}39.6 & \cellcolor[rgb]{0.892 0.954 0.948}53.6 & \cellcolor[rgb]{0.921 0.967 0.962}36.8 & \cellcolor[rgb]{1.000 0.984 0.966}19.5 & \cellcolor[rgb]{1.000 0.957 0.913}17.7 & \cellcolor[rgb]{1.000 0.949 0.896}26.4 & \cellcolor[rgb]{1.000 0.990 0.980}8.2 ~\\
  ~  & \cmark  & \cmark  &  ~  &  ~  & \cellcolor[rgb]{1.000 1.000 1.000}40.8 & \cellcolor[rgb]{0.983 0.993 0.992}32.1 & \cellcolor[rgb]{1.000 1.000 1.000}41.9 & \cellcolor[rgb]{1.000 1.000 1.000}25.2 & \cellcolor[rgb]{1.000 1.000 1.000}16.9 & \cellcolor[rgb]{1.000 1.000 1.000}9.6 & \cellcolor[rgb]{1.000 1.000 1.000}17.0 & \cellcolor[rgb]{1.000 1.000 1.000}6.0 ~\\
  ~  & \cmark  &  ~  &  ~  &  ~  & \cellcolor[rgb]{0.960 0.983 0.981}45.2 & \cellcolor[rgb]{1.000 1.000 1.000}29.9 & \cellcolor[rgb]{0.942 0.975 0.972}48.2 & \cellcolor[rgb]{0.975 0.989 0.988}28.9 & \cellcolor[rgb]{1.000 0.959 0.917}23.4 & \cellcolor[rgb]{1.000 0.998 0.996}10.0 & \cellcolor[rgb]{1.000 0.973 0.945}22.0 & \cellcolor[rgb]{1.000 0.995 0.990}7.2 ~\\
  \Circpipe& \cmark  & \cmark  & \cmark  & \cmark  & \cellcolor[rgb]{0.856 0.939 0.931}56.7 & \cellcolor[rgb]{0.911 0.962 0.958}41.3 & \cellcolor[rgb]{0.894 0.955 0.949}53.4 & \cellcolor[rgb]{0.956 0.981 0.979}31.7 & \cellcolor[rgb]{1.000 0.941 0.878}26.4 & \cellcolor[rgb]{1.000 0.963 0.924}16.7 & \cellcolor[rgb]{1.000 0.946 0.889}27.0 & \cellcolor[rgb]{1.000 0.988 0.975}8.8 ~\\
  \Circpipe& \cmark  & \cmark  & \cmark  &  ~  & \cellcolor[rgb]{0.827 0.927 0.918}59.9 & \cellcolor[rgb]{0.932 0.971 0.968}38.6 & \cellcolor[rgb]{0.891 0.954 0.948}53.7 & \cellcolor[rgb]{0.960 0.983 0.981}31.1 & \cellcolor[rgb]{1.000 0.930 0.857}28.1 & \cellcolor[rgb]{1.000 1.000 1.000}6.3 & \cellcolor[rgb]{1.000 0.949 0.896}26.4 & \cellcolor[rgb]{1.000 1.000 1.000}5.4 ~\\
  \Circpipe& \cmark  & \cmark  &  ~  &  ~  & \cellcolor[rgb]{1.000 1.000 1.000}40.1 & \cellcolor[rgb]{0.981 0.992 0.991}32.3 & \cellcolor[rgb]{1.000 1.000 1.000}41.1 & \cellcolor[rgb]{1.000 1.000 1.000}24.8 & \cellcolor[rgb]{1.000 1.000 1.000}16.6 & \cellcolor[rgb]{1.000 1.000 1.000}9.5 & \cellcolor[rgb]{1.000 1.000 1.000}16.5 & \cellcolor[rgb]{1.000 0.992 0.983}7.9 ~\\
  \Circpipe& \cmark  &  ~  &  ~  &  ~  & \cellcolor[rgb]{0.993 0.997 0.997}41.6 & \cellcolor[rgb]{0.996 0.998 0.998}30.4 & \cellcolor[rgb]{0.939 0.974 0.971}48.5 & \cellcolor[rgb]{0.985 0.994 0.993}27.4 & \cellcolor[rgb]{1.000 0.980 0.960}20.0 & \cellcolor[rgb]{1.000 1.000 1.000}8.5 & \cellcolor[rgb]{1.000 0.965 0.928}23.5 & \cellcolor[rgb]{1.000 0.998 0.997}6.4 ~\\
  \bottomrule
  \end{tabular}%
\end{table}

\subsection{Transfer Learning on Semantic Segmentation}

\textbf{Datasets.}
We evaluate the learned dense presentations on three semantic segmentation benchmarks with various domains and scales, including \textbf{COCOStuff27}\footnote{\url{https://github.com/nightrome/cocostuff}} \cite{caesar2018coco}, \textbf{PASCAL VOC 2012}\footnote{\url{http://host.robots.ox.ac.uk/pascal/VOC}} \cite{everingham2010pascal}, \textbf{ADE20k}\footnote{\url{https://groups.csail.mit.edu/vision/datasets/ADE20K}} \cite{zhou2017scene}, and \textbf{Cityscapes}\footnote{\url{https://www.cityscapes-dataset.com}} \cite{cordts2016cityscapes}. We follow the official setting to split each dataset into a training set and a validation set. 

\noindent\textbf{Network architecture.}
Following \cite{dino}, we drop the projector and use the frozen patch representations extracted from ViT-S/16. In the linear transfer learning setting, we utilize a $1\times 1$ convolutional layer to map the patch representations into probability vectors. As for FCN, we stack three upsampling blocks, where each block consists of two $1\times 1$ convolutional layers, one $3\times 3$ convolutional layer, and a bilinear upsampling layer. The activation function is implemented as ReLU.

\noindent\textbf{Optimization strategy.}
Since the Cityscapes dataset involves numerous small objects, images are randomly cropped into $896\times 896$ at the phase of training and resized into $1024\times 1024$ at the testing phase. For the other three datasets, images are randomly cropped into $336\times 336$ at the phase of training and resized into $448\times 448$ at the testing phase.
The additional layers are optimized with the Adam \cite{loshchilovdecoupled} optimizer. The initial learning rate is searched in the range $\{10^{-4}, 10^{-3}, 3\times 10^{-3}, 1\times {-2}\}$. For FCN, the learning rate decays to $10^{-4}$ with a cosine decay schedule. By default, we set the batch size to $64$ and $128$ for Cityscapes and other datasets, respectively. Linear models are trained with $100k$ samples (or equivalently $1560$ iterations). FCN models are trained with $400k$ samples (or equivalently $6250$ iterations). It takes at most $3$ hours to train an FCN model on eight NVIDIA RTX 4090 GPUs.

\noindent\textbf{Additional Results.}
The pixel accuracy results are shown in \Tbref{tab:semantic_seg_acc}. It can be seen that the pixel accuracy is basically consistent with mIoU. However, limited by the imbalanced distribution of categories, the pixel accuracy metric is insensitive to comparable methods, especially for FCN models which are more discriminative.

\begin{table*}
  \caption{Transfer learning results (pixel accuracy) on semantic segmentation.}
  \centering
  \begin{tabular}{ll|cccccccc}
      \toprule
      \multirow{2}{*}{Methods} & \multirow{2}{*}{Dataset} & \multicolumn{2}{c}{\textbf{COCOStuff-27}} & \multicolumn{2}{c}{\textbf{PASCAL VOC}}  & \multicolumn{2}{c}{\textbf{ADE20k}} & \multicolumn{2}{c}{\textbf{Cityscapes}} \\
      & & Linear & FCN & Linear & FCN & Linear & FCN & Linear & FCN \\
      \midrule
      DINO \cite{dino} & IN-1k & \cellcolor[rgb]{1.000 1.000 1.000}71.7 & \cellcolor[rgb]{0.968 0.987 0.985}79.2 & \cellcolor[rgb]{1.000 1.000 1.000}88.4 & \cellcolor[rgb]{1.000 0.958 0.913}92.5 & \cellcolor[rgb]{1.000 1.000 1.000}66.5 & \cellcolor[rgb]{0.999 0.965 0.961}73.2 & \cellcolor[rgb]{0.871 0.918 0.954}88.5 & \cellcolor[rgb]{0.548 0.714 0.837}92.2 ~\\
      iBOT \cite{zhou2021ibot} & IN-1k & \cellcolor[rgb]{0.883 0.950 0.944}76.5 & \cellcolor[rgb]{0.731 0.886 0.871}80.7 & \cellcolor[rgb]{1.000 0.857 0.707}92.3 & \cellcolor[rgb]{1.000 0.847 0.687}93.8 & \cellcolor[rgb]{0.996 0.885 0.871}70.8 & \cellcolor[rgb]{0.986 0.654 0.611}75.0 & \cellcolor[rgb]{0.794 0.870 0.926}88.8 & \cellcolor[rgb]{0.585 0.738 0.851}92.1 ~\\
      Leopart \cite{leopart} & IN-1k + CC & \cellcolor[rgb]{0.848 0.936 0.928}76.8 & \cellcolor[rgb]{0.842 0.933 0.924}80.0 & \cellcolor[rgb]{1.000 0.913 0.822}91.4 & \cellcolor[rgb]{1.000 0.915 0.826}93.0 & \cellcolor[rgb]{1.000 1.000 1.000}69.4 & \cellcolor[rgb]{0.997 0.931 0.922}73.4 & \cellcolor[rgb]{0.665 0.788 0.879}\underline{89.3} & \cellcolor[rgb]{0.661 0.786 0.878}91.9 ~\\
      Mugs \cite{zhou2022mugs} & IN-1k & \cellcolor[rgb]{0.778 0.906 0.894}\underline{77.4} & \cellcolor[rgb]{0.652 0.852 0.834}\underline{81.2} & \cellcolor[rgb]{1.000 0.801 0.593}\underline{93.2} & \cellcolor[rgb]{1.000 0.788 0.565}\underline{94.5} & \cellcolor[rgb]{0.988 0.684 0.645}\underline{72.2} & \cellcolor[rgb]{0.982 0.550 0.494}\underline{75.6} & \cellcolor[rgb]{0.690 0.804 0.889}89.2 & \cellcolor[rgb]{0.510 0.690 0.824}\textbf{\underline{92.3}} ~\\
      SERE \cite{li2023sere} & IN-1k & \cellcolor[rgb]{0.802 0.916 0.905}77.2 & \cellcolor[rgb]{0.778 0.906 0.894}80.4 & \cellcolor[rgb]{1.000 0.820 0.631}92.9 & \cellcolor[rgb]{1.000 0.847 0.687}93.8 & \cellcolor[rgb]{0.990 0.742 0.710}71.8 & \cellcolor[rgb]{0.988 0.688 0.649}74.8 & \cellcolor[rgb]{0.690 0.804 0.889}89.2 & \cellcolor[rgb]{0.585 0.738 0.851}92.1 ~\\
      CrOC \cite{stegmuller2023croc} & IN-1k & \cellcolor[rgb]{0.953 0.980 0.978}75.9 & \cellcolor[rgb]{0.937 0.973 0.970}79.4 & \cellcolor[rgb]{1.000 0.857 0.707}92.3 & \cellcolor[rgb]{1.000 0.983 0.965}92.2 & \cellcolor[rgb]{0.999 0.986 0.984}70.1 & \cellcolor[rgb]{0.996 0.896 0.883}73.6 & \cellcolor[rgb]{0.819 0.886 0.935}88.7 & \cellcolor[rgb]{0.661 0.786 0.878}91.9 ~\\
      \textbf{DINO + Ours} & IN-1k + CC & \cellcolor[rgb]{0.604 0.832 0.811}78.9 & \cellcolor[rgb]{0.620 0.839 0.818}81.4 & \cellcolor[rgb]{1.000 0.776 0.542}93.6 & \cellcolor[rgb]{1.000 0.788 0.565}94.5 & \cellcolor[rgb]{0.981 0.512 0.451}73.4 & \cellcolor[rgb]{0.982 0.550 0.494}75.6 & \cellcolor[rgb]{0.510 0.690 0.824}\textbf{89.9} & \cellcolor[rgb]{0.510 0.690 0.824}92.3 ~\\
      \textbf{iBOT + Ours} & IN-1k + CC & \cellcolor[rgb]{0.557 0.812 0.788}\textbf{79.3} & \cellcolor[rgb]{0.557 0.812 0.788}\textbf{81.8} & \cellcolor[rgb]{1.000 0.751 0.491}94.0 & \cellcolor[rgb]{1.000 0.771 0.531}94.7 & \cellcolor[rgb]{0.982 0.527 0.468}73.3 & \cellcolor[rgb]{0.980 0.498 0.435}\textbf{75.9} & \cellcolor[rgb]{0.613 0.755 0.861}89.5 & \cellcolor[rgb]{0.623 0.762 0.864}92.0 ~\\
      \textbf{Mugs + Ours} & IN-1k + CC & \cellcolor[rgb]{0.592 0.827 0.805}79.0 & \cellcolor[rgb]{0.573 0.818 0.796}81.7 & \cellcolor[rgb]{1.000 0.745 0.478}\textbf{94.1} & \cellcolor[rgb]{1.000 0.745 0.478}\textbf{95.0} & \cellcolor[rgb]{0.980 0.498 0.435}\textbf{73.5} & \cellcolor[rgb]{0.980 0.498 0.435}75.9 & \cellcolor[rgb]{0.510 0.690 0.824}89.9 & \cellcolor[rgb]{0.623 0.762 0.864}92.0 ~\\
      \bottomrule
  \end{tabular}
  \label{tab:semantic_seg_acc}
  \end{table*}

\begin{table*}[t]
  \caption{Results on object detection and instance segmentation.}
  \setlength\tabcolsep{3.6pt}
  \centering
    \begin{tabular}{ll|cccccc|cccccc}
        \toprule
        \multirow{2}{*}{Methods} & \multirow{2}{*}{Arch.} & \multicolumn{6}{c|}{\textbf{Object Detection}} & \multicolumn{6}{c}{\textbf{Instance Segmentation}} \\
        & & $\mathrm{AP^b}$ & $\mathrm{AP^b_{50}}$ & $\mathrm{AP^b_{75}}$ & $\mathrm{AP^b_{small}}$ & $\mathrm{AP^b_{medium}}$ & $\mathrm{AP^b_{large}}$ & $\mathrm{AP^m}$ & $\mathrm{AP^m_{50}}$ & $\mathrm{AP^m_{75}}$ & $\mathrm{AP^m_{small}}$ & $\mathrm{AP^m_{medium}}$ & $\mathrm{AP^m_{large}}$ \\
        \midrule
        Sup. \cite{liu2021swin} & Swin-T & \cellcolor[rgb]{0.982 0.992 0.991}48.1 & \cellcolor[rgb]{0.986 0.994 0.993}67.1 & \cellcolor[rgb]{0.815 0.922 0.912}52.5 &  --  &  --  &  --  & \cellcolor[rgb]{0.994 0.854 0.835}41.7 & \cellcolor[rgb]{0.997 0.926 0.916}64.4 & \cellcolor[rgb]{0.992 0.807 0.783}45.0 &  --  &  --  &  --  ~\\
        MoBY \cite{xie2021self} & Swin-T & \cellcolor[rgb]{0.982 0.992 0.991}48.1 & \cellcolor[rgb]{0.986 0.994 0.993}67.1 & \cellcolor[rgb]{0.865 0.942 0.935}52.1 &  --  &  --  &  --  & \cellcolor[rgb]{0.996 0.895 0.882}41.5 & \cellcolor[rgb]{1.000 1.000 1.000}64.0 & \cellcolor[rgb]{0.995 0.865 0.848}44.7 &  --  &  --  &  --  ~\\
        Sup. \cite{liu2021swin} & ViT-S/16 & \cellcolor[rgb]{1.000 1.000 1.000}46.2 & \cellcolor[rgb]{1.000 1.000 1.000}65.9 & \cellcolor[rgb]{1.000 1.000 1.000}49.6 &  --  &  --  &  --  & \cellcolor[rgb]{1.000 1.000 1.000}40.1 & \cellcolor[rgb]{1.000 1.000 1.000}62.9 & \cellcolor[rgb]{1.000 1.000 1.000}42.8 &  --  &  --  &  --  ~\\
        iBOT \cite{zhou2021ibot} & ViT-S/16 & \cellcolor[rgb]{0.742 0.890 0.876}49.4 & \cellcolor[rgb]{0.757 0.897 0.884}68.7 & \cellcolor[rgb]{0.717 0.880 0.865}53.3 &  --  &  --  &  --  & \cellcolor[rgb]{0.987 0.665 0.624}42.6 & \cellcolor[rgb]{0.988 0.703 0.665}65.6 & \cellcolor[rgb]{0.986 0.652 0.609}45.8 &  --  &  --  &  --  ~\\
        Mugs \cite{zhou2022mugs} & ViT-S/16 & \cellcolor[rgb]{0.668 0.859 0.841}\underline{49.8} & \cellcolor[rgb]{0.728 0.885 0.870}\underline{68.9} & \cellcolor[rgb]{0.692 0.869 0.853}\underline{53.5} & \cellcolor[rgb]{1.000 0.967 0.932}\underline{31.3} & \cellcolor[rgb]{1.000 0.891 0.776}\underline{52.9} & \cellcolor[rgb]{1.000 0.777 0.544}\underline{66.1} & \cellcolor[rgb]{0.984 0.582 0.529}\underline{43.0} & \cellcolor[rgb]{0.988 0.684 0.644}\underline{65.7} & \cellcolor[rgb]{0.986 0.633 0.587}\underline{45.9} & \cellcolor[rgb]{0.808 0.879 0.931}\underline{24.9} & \cellcolor[rgb]{1.000 1.000 1.000}\underline{46.0} & \cellcolor[rgb]{0.510 0.690 0.824}\underline{59.8} ~\\
        \textbf{DINO + Ours} & ViT-S/16 & \cellcolor[rgb]{0.742 0.890 0.876}49.4 & \cellcolor[rgb]{0.771 0.903 0.891}68.6 & \cellcolor[rgb]{0.729 0.885 0.871}53.2 & \cellcolor[rgb]{1.000 0.900 0.796}31.9 & \cellcolor[rgb]{1.000 0.891 0.776}52.9 & \cellcolor[rgb]{1.000 0.904 0.804}64.9 & \cellcolor[rgb]{0.987 0.665 0.624}42.6 & \cellcolor[rgb]{0.988 0.703 0.665}65.6 & \cellcolor[rgb]{0.985 0.614 0.566}46.0 & \cellcolor[rgb]{0.680 0.798 0.885}25.5 & \cellcolor[rgb]{1.000 1.000 1.000}46.0 & \cellcolor[rgb]{0.864 0.914 0.951}58.5 ~\\
        \textbf{iBOT + Ours} & ViT-S/16 & \cellcolor[rgb]{0.668 0.859 0.841}49.8 & \cellcolor[rgb]{0.714 0.879 0.863}69.0 & \cellcolor[rgb]{0.643 0.848 0.829}53.9 & \cellcolor[rgb]{1.000 0.922 0.841}31.7 & \cellcolor[rgb]{1.000 0.854 0.702}53.2 & \cellcolor[rgb]{1.000 0.777 0.544}66.1 & \cellcolor[rgb]{0.984 0.603 0.553}42.9 & \cellcolor[rgb]{0.985 0.610 0.561}66.1 & \cellcolor[rgb]{0.983 0.575 0.522}46.2 & \cellcolor[rgb]{0.744 0.838 0.908}25.2 & \cellcolor[rgb]{0.777 0.859 0.920}46.5 & \cellcolor[rgb]{0.673 0.793 0.882}59.2 ~\\
        \textbf{Mugs + Ours} & ViT-S/16 & \cellcolor[rgb]{0.557 0.812 0.788}\textbf{50.4} & \cellcolor[rgb]{0.557 0.812 0.788}\textbf{70.1} & \cellcolor[rgb]{0.557 0.812 0.788}\textbf{54.6} & \cellcolor[rgb]{1.000 0.745 0.478}\textbf{33.3} & \cellcolor[rgb]{1.000 0.745 0.478}\textbf{54.1} & \cellcolor[rgb]{1.000 0.745 0.478}\textbf{66.4} & \cellcolor[rgb]{0.980 0.498 0.435}\textbf{43.4} & \cellcolor[rgb]{0.980 0.498 0.435}\textbf{66.7} & \cellcolor[rgb]{0.980 0.498 0.435}\textbf{46.6} & \cellcolor[rgb]{0.510 0.690 0.824}\textbf{26.3} & \cellcolor[rgb]{0.510 0.690 0.824}\textbf{47.1} & \cellcolor[rgb]{0.510 0.690 0.824}\textbf{59.8} ~\\
        \bottomrule
    \end{tabular}
  \label{tab:obj_det_full}
\end{table*}

  \begin{figure*}[t]
    \centering
    \vspace{-2mm}
    \subfigure[Image classification on ImageNet-1k]{
      \includegraphics[scale=0.27]{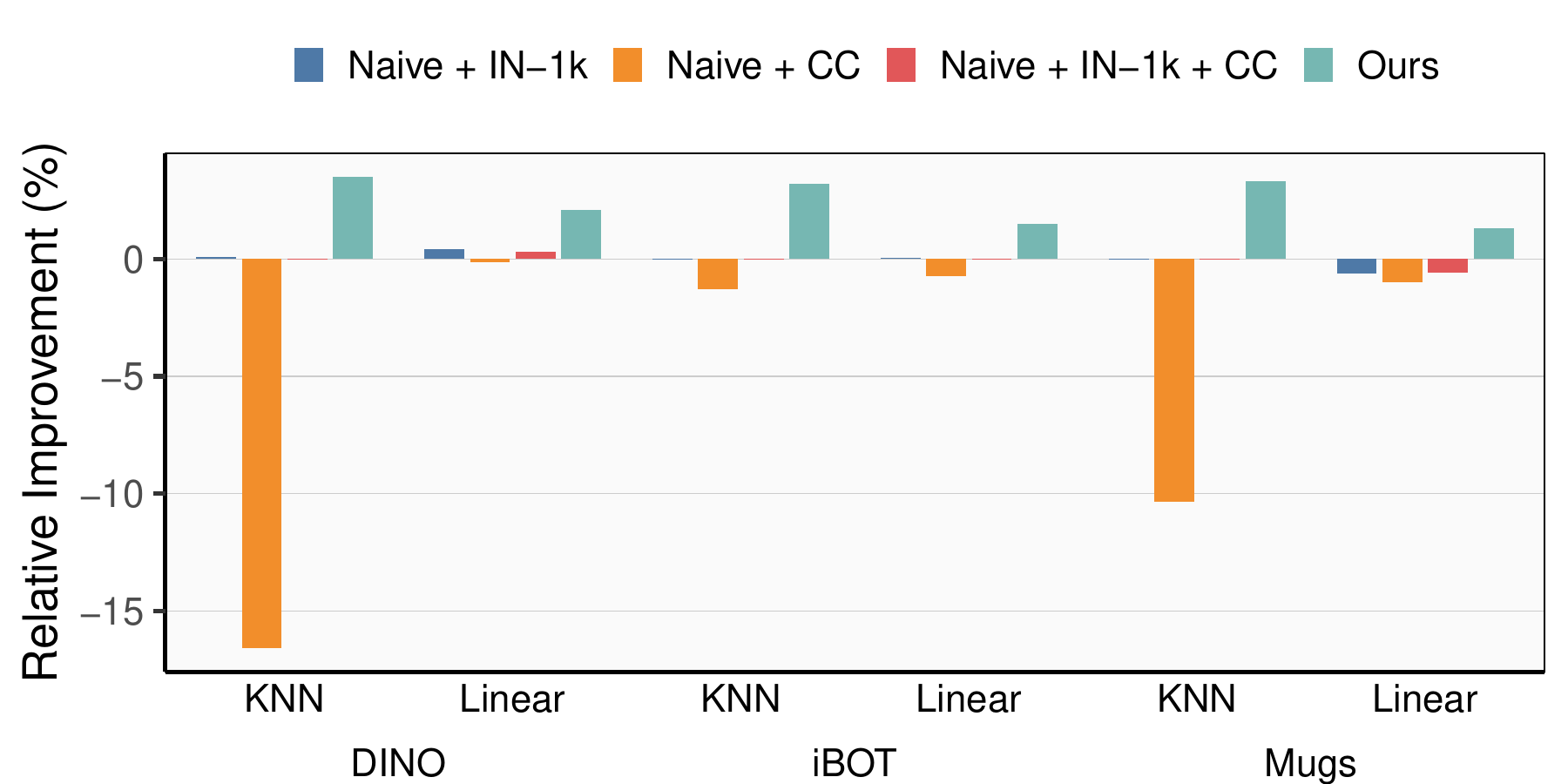}
    }
    \subfigure[Semantic segmentation on COCOStuff-27]{
      \includegraphics[scale=0.27]{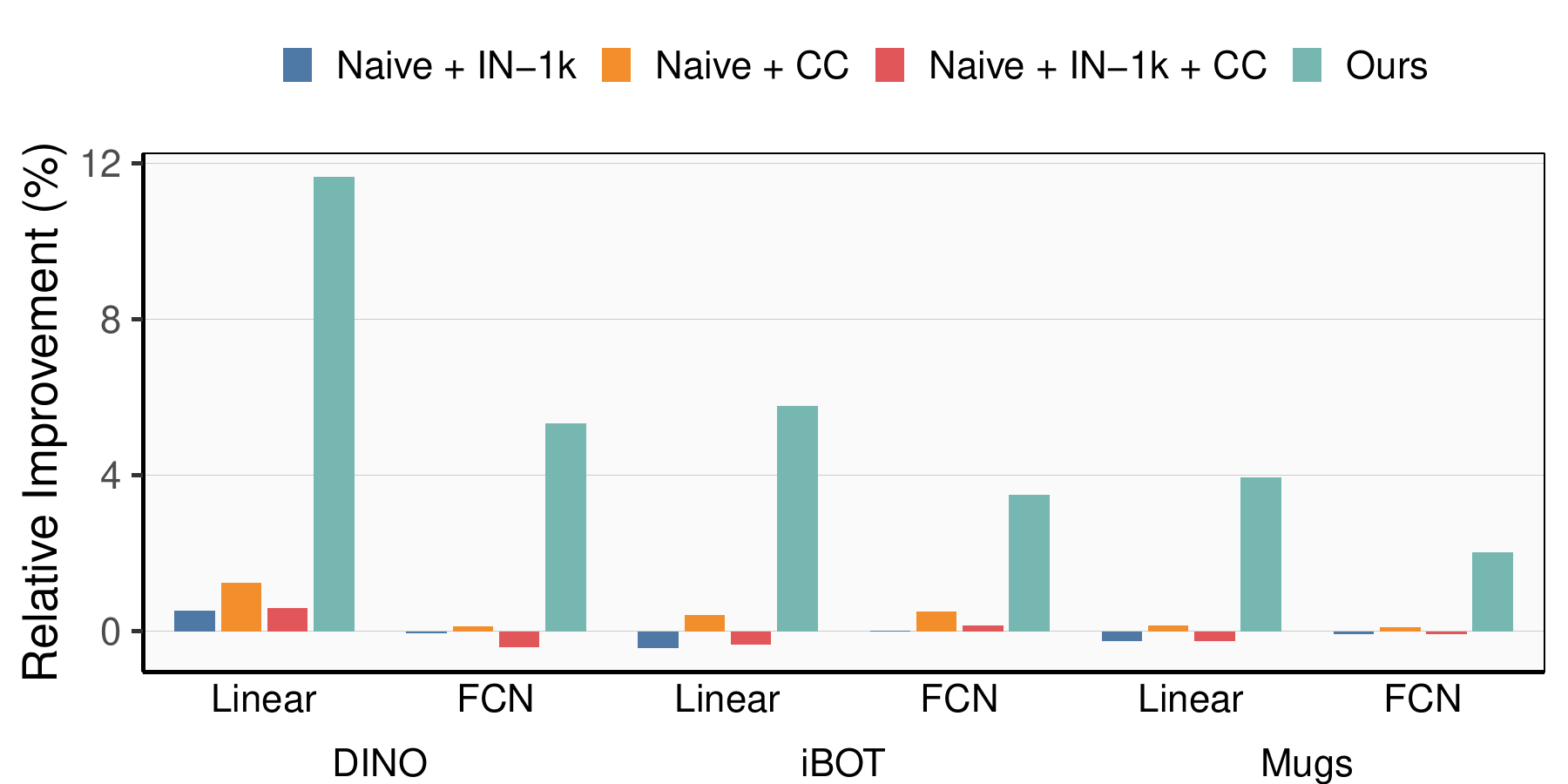}
    }
    \subfigure[Semantic segmentation on PASCAL VOC]{
      \includegraphics[scale=0.27]{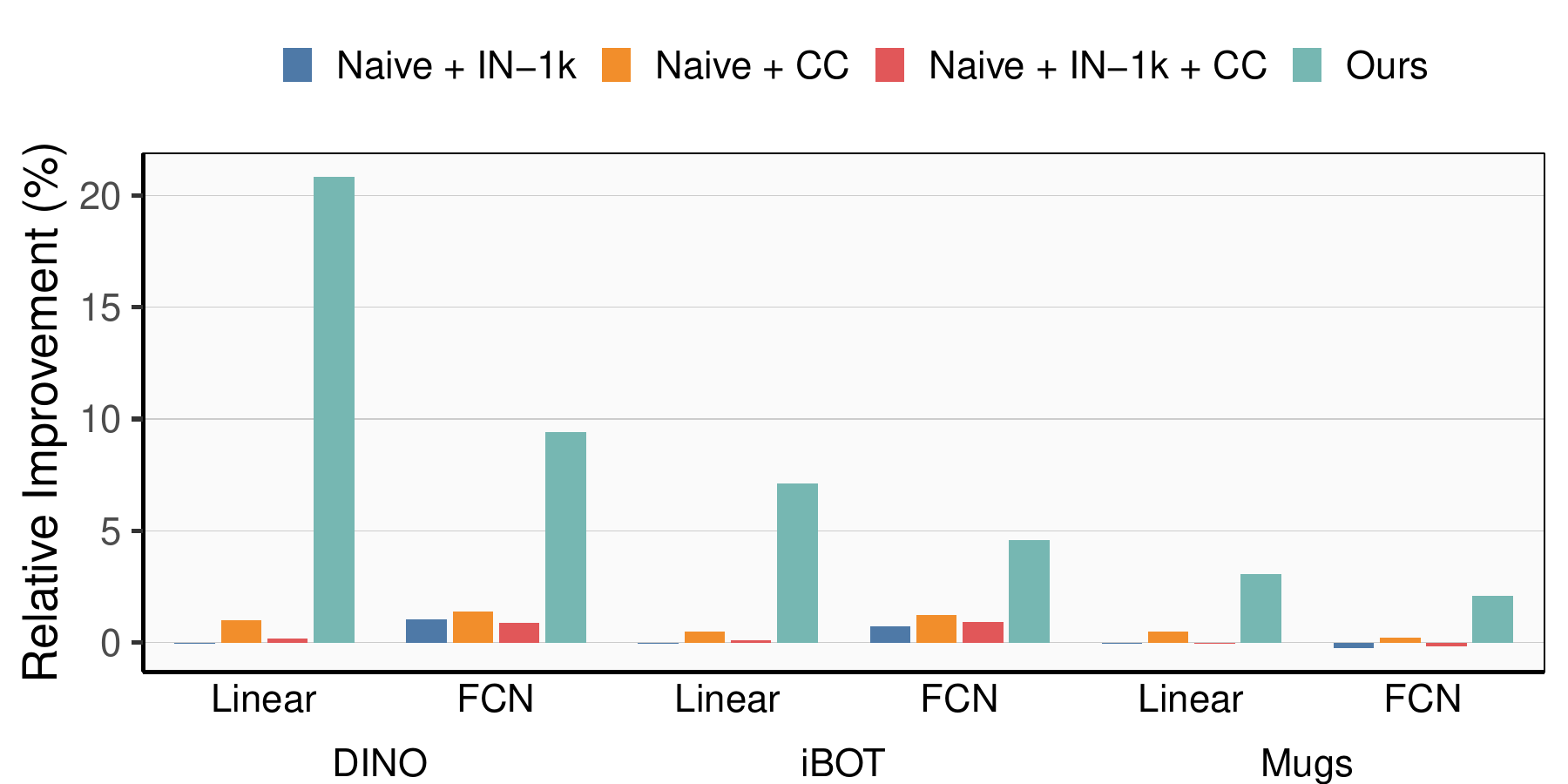}
    }
    \subfigure[Semantic segmentation on ADE20k]{
      \includegraphics[scale=0.27]{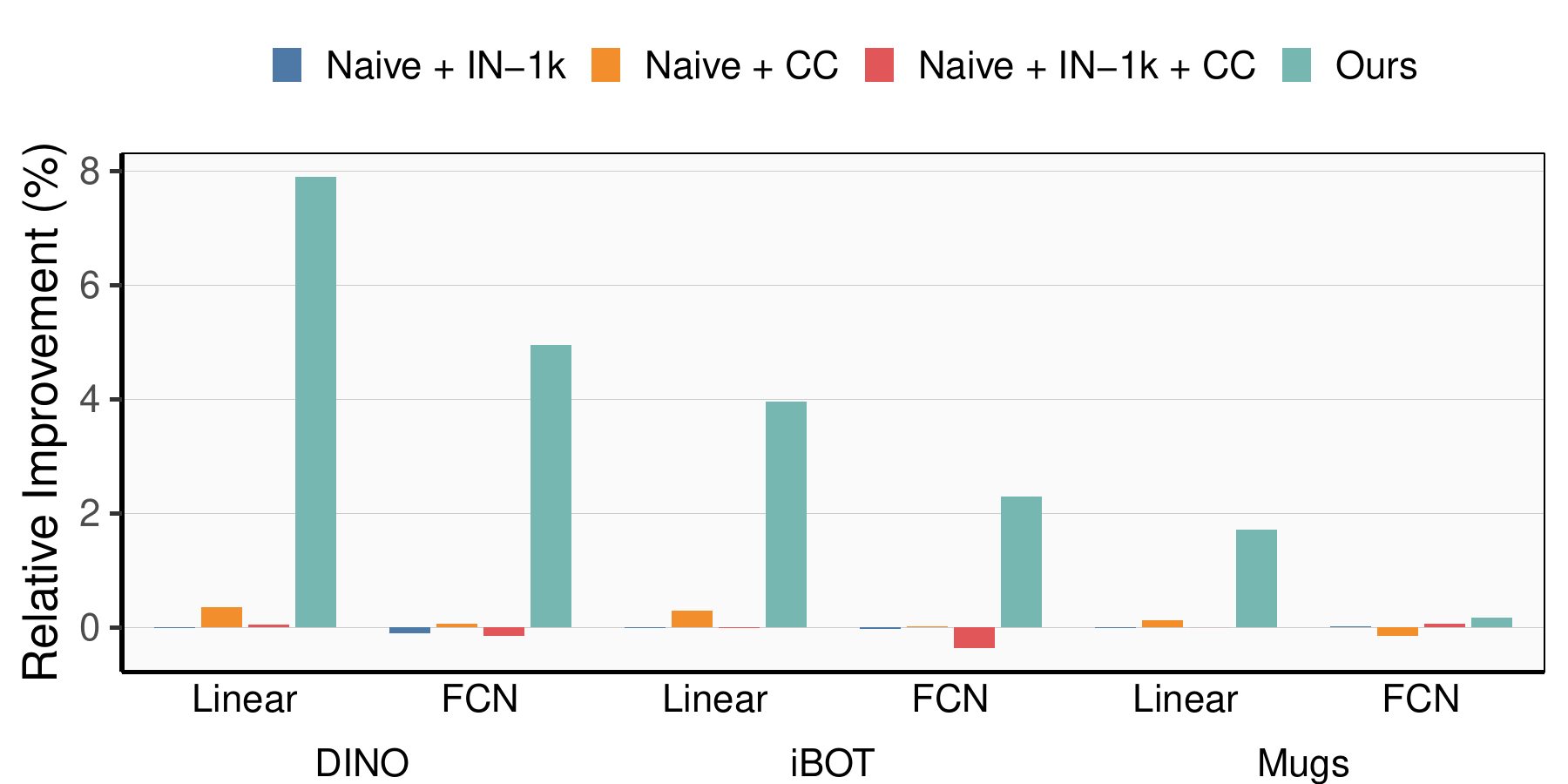}
    }
    \subfigure[Semantic segmentation on Cityscapes]{
      \includegraphics[scale=0.27]{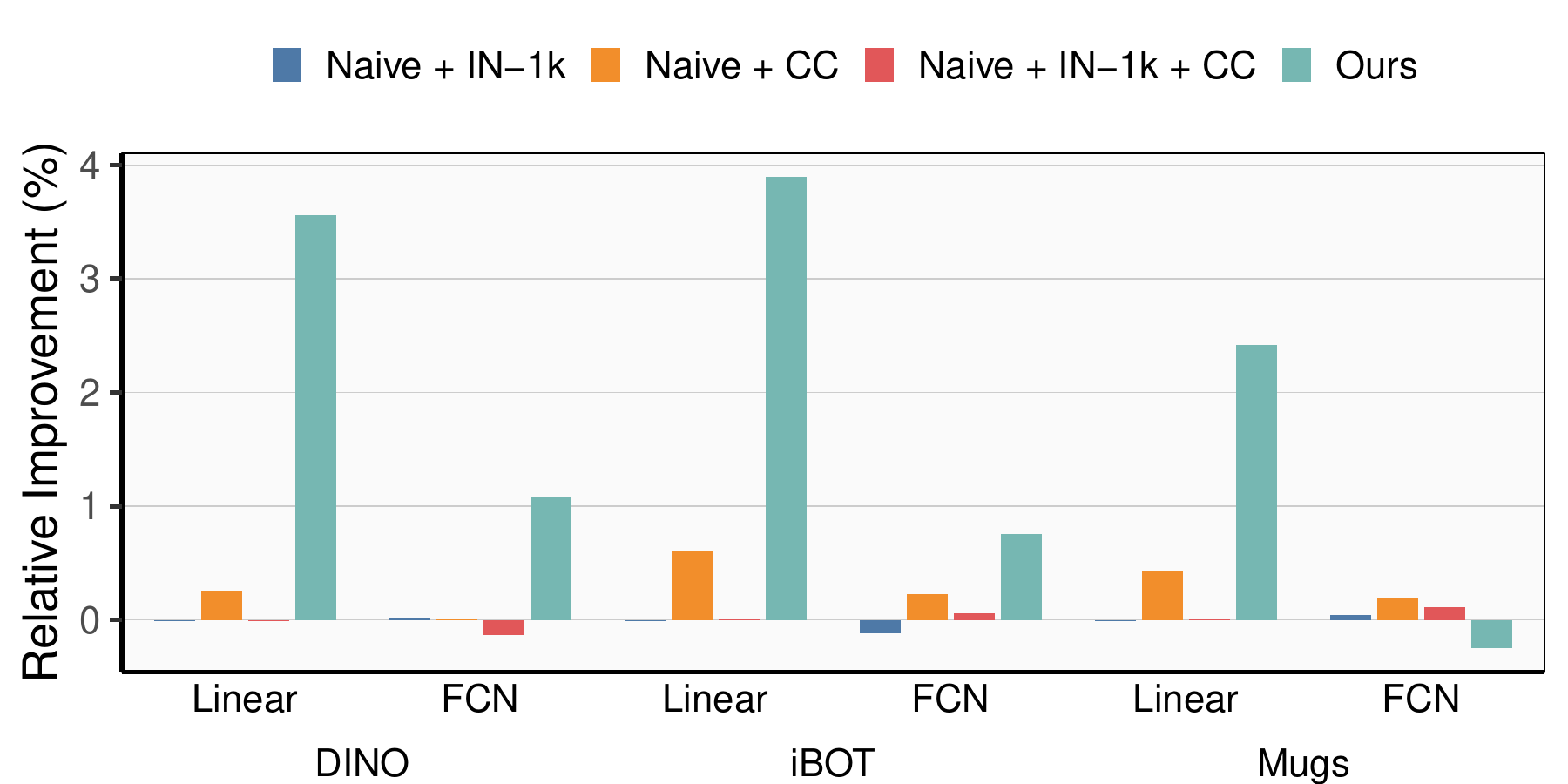}
    }
    \subfigure[Video object segmentation on DAVIS-2017]{
      \includegraphics[scale=0.27]{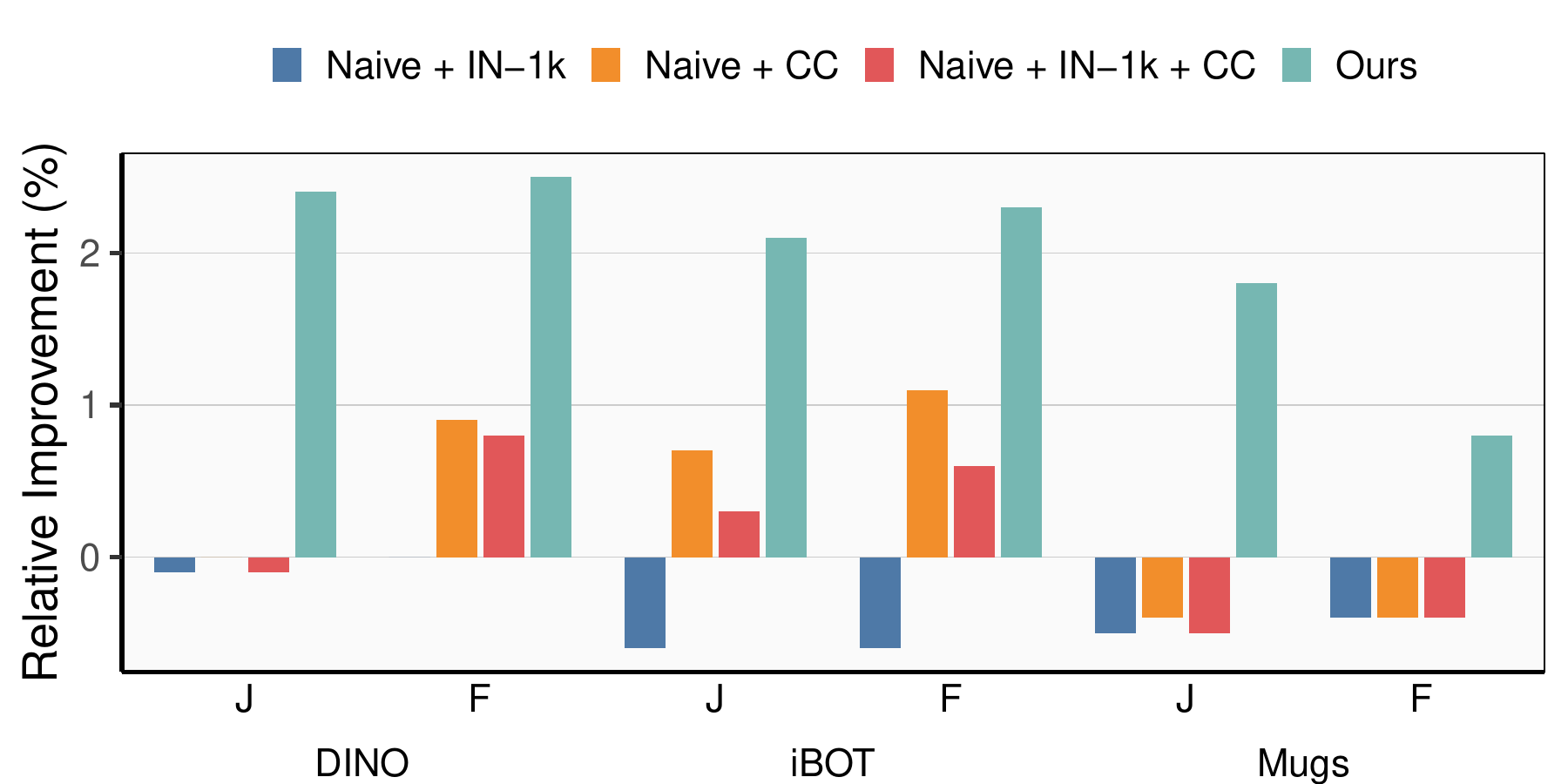}
    }
    \vspace{-2mm}
    \caption{Relative improvement of additional training epochs and the training dataset over different baselines.}
    \label{fig:longer_training_full}
  \end{figure*}
  
\subsection{Transfer Learning on Object Detection \& Instance Segmentation}
\noindent\textbf{Implementation.} Following \cite{zhou2021ibot,zhou2022mugs}, we plug the pretrained model into the Cascade Mask
R-CNN \cite{cai2019cascade}. Following \cite{li2022exploring}, we apply a simple feature pyramid without FPN. All parameters are fine-tuned for $12$ epochs with an AdamW optimizer. The learning rate is initialized as $0.0002$ and decays by $0.1$ at the $9$-th and $11$-th epoch.

\noindent\textbf{Additional Results.} Results in terms of other metrics are provided in \Tbref{tab:obj_det_full}, from which can be seen that our proposed method surpasses all competitors consistently. Compared with the state-of-the-art method Mugs, ours achieves comparable performance for large objects, while for small objects ours shows a significant advantage. Consequently, our method brings improvements by generating fine-grained representations.

\subsection{Zero-shot Transfer Learning on Video Object Segmentation}
\noindent \textbf{Implementation.} We leverage the learned dense representations to solve the semi-supervised VOS, \ie, given the target mask in the first frame, we need to track and segment the object in the rest frames. Following \cite{dino,zhou2021ibot}, masks of the first frame and the previous frames are propagated by finding the nearest neighbor between consecutive frames. Experiments are conducted on the DAVIS-2017 \cite{pont20172017} dataset, which constants $30$ videos and $59$ objects.

\subsection{Additional Results of Extended Training and Datasets}
The evaluation results on other testing settings are provided in \Fgref{fig:longer_training_full}. The trend is consistent with \Fgref{fig:longer_training}. It can be seen that the relative improvement is more obvious in $k$-NN classification and linear-based semantic segmentation, demonstrating that our proposed method generate more semantic-aware representations for both image-level token and patch tokens.

\section{Visualization Results}
\label{app:vis}
To qualitatively demonstrate the effects of our proposed method, we provided more visualization results here. All visualizations are conducted with a ViT-S/16 model pretrained with our proposed framework. The input images are resized into $336\times 336$. 

\textbf{Cross-attention maps of the OAF module.} In \Fgref{fig:vis_attn_full}, we visualize the cross-attention maps of different heads between object prototypes and the input image. Compared with self-attention maps shown in \Fgref{fig:vis_attn}, the proposed OAF is more effective in capturing the shared patterns between different objects.

\textbf{Clustering of patch representations.} The patch representations are clustered and visualized via t-SNE in \Fgref{fig:vis_tsne}, which presents a clear clustered structure. Object parts of the same category are well-aligned in the embedding space, even if their appearances could be various. 

\textbf{Sparse correspondences between images/views.} A natural concern is whether the semantic concentration techniques impair the distinguishability between patches from the same category. To answer this question, we utilize the correspondence visualization techniques proposed by \cite{amir2021deep} to perform key point matching. As shown in \Fgref{fig:vis_corr_pos} and \Fgref{fig:vis_corr}, representations of the object parts are robust to background scenes, spatial shift, and color distortion, showing satisfactory discriminability of fine-grained representations.

\begin{figure*}[h]
  \centering
  \includegraphics[scale=0.55]{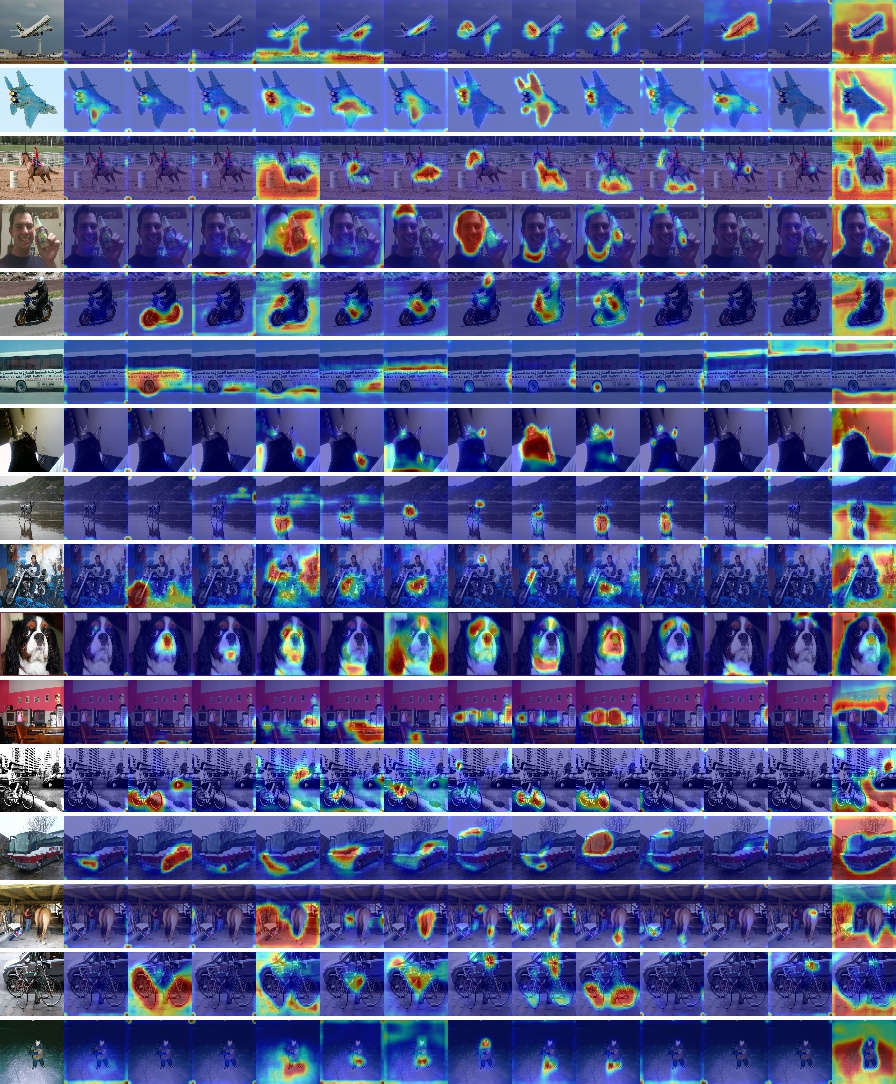}
  \caption{Cross-attention maps of the OAF module. Each column corresponds to one of the heads of an object prototype. The response in the attention maps is ignorable if no similar object appears in the image.}
  \label{fig:vis_attn_full}
\end{figure*}

\begin{figure*}[h]
  \centering
  \includegraphics[scale=0.52]{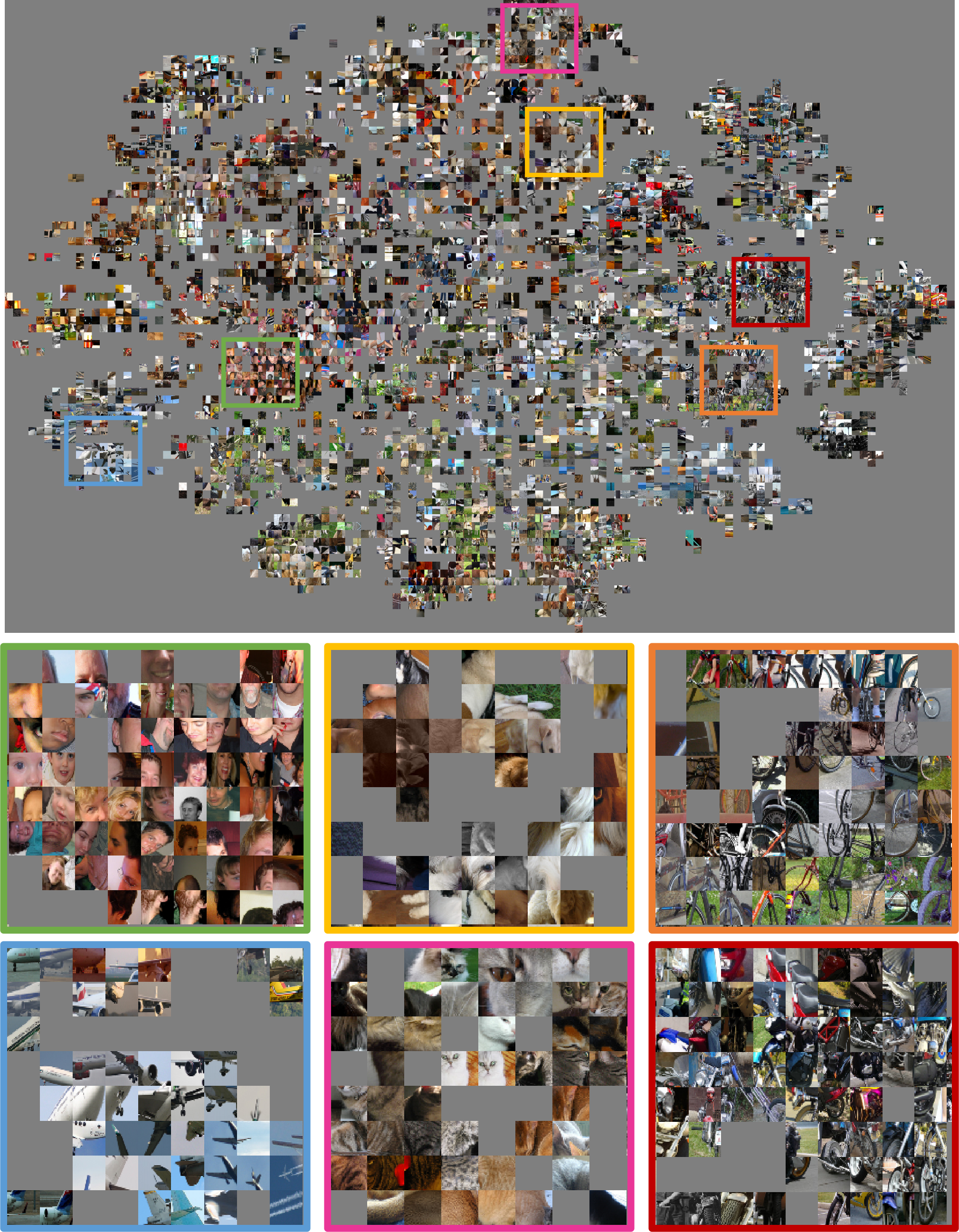}
  \caption{Clustering of patch representations in the PASCAL VOC dataset. The dimension of representations is reduced to 2 via t-SNE. The patches are randomly sampled from the validation set. Since it is difficult to discern the semantic information from the $16\times 16$ patches, the height and width of the patches shown in the figure are expanded 5 times for the sake of observation.}
  \label{fig:vis_tsne}
\end{figure*}

\begin{figure*}[h]
  \centering
  \includegraphics[scale=0.92]{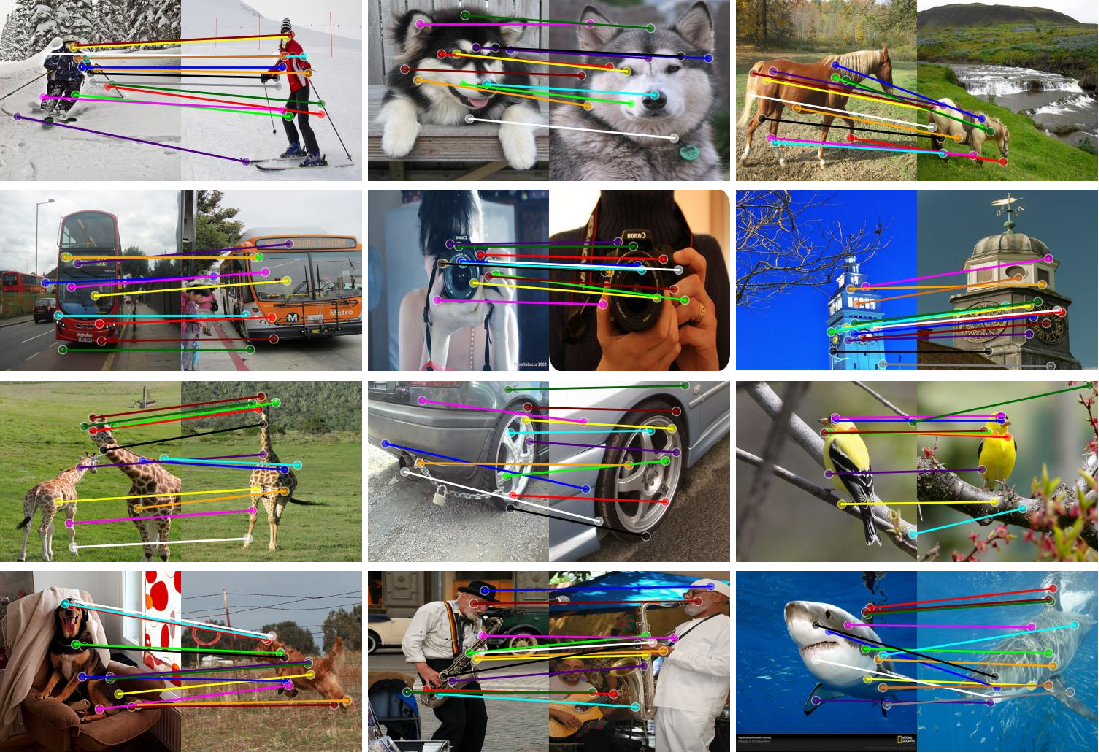}
  \caption{Visualization of sparse correspondences between images from the same category.}
  \label{fig:vis_corr_pos}
\end{figure*}

\begin{figure*}[h]
  \centering
  \includegraphics[scale=0.92]{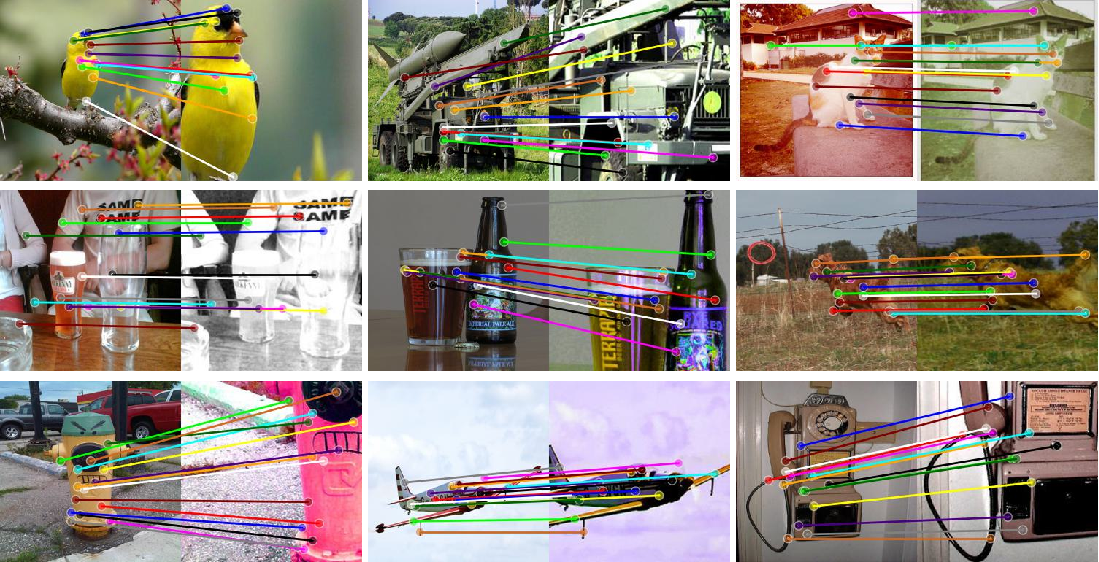}
  \caption{Visualization of sparse correspondences between different views.}
  \label{fig:vis_corr}
\end{figure*}

\end{document}